\documentclass[oneside,11pt]{book}
\usepackage[Lenny]{fncychap}
\ChNameVar{\Large}

\newcommand{\BlackBox}{\rule{1.5ex}{1.5ex}}  %
\newenvironment{proof}{\par\noindent{\bf Proof\ }}{\hfill\BlackBox\\[2mm]}
\newtheorem{theorem}{Theorem}
\newtheorem{proposition}{Proposition}
\newtheorem{remark}{Remark}
\newtheorem{definition}{Definition}

\usepackage[margin=1.5in]{geometry}

\usepackage{natbib}
\usepackage{xspace}

\newcommand{\ie}{\emph{i.e.}\xspace}
\newcommand{\eg}{\emph{e.g.}\xspace}

\usepackage{amsmath,amsfonts,bm,mathtools,amssymb}

\def\1{\bm{1}}

\def\gA{{\mathcal{A}}}
\def\gB{{\mathcal{B}}}
\def\gC{{\mathcal{C}}}
\def\gD{{\mathcal{D}}}
\def\gE{{\mathcal{E}}}
\def\gF{{\mathcal{F}}}

\def\gK{{\mathcal{K}}}
\def\gL{{\mathcal{L}}}
\def\gM{{\mathcal{M}}}
\def\gN{{\mathcal{N}}}

\def\gP{{\mathcal{P}}}
\def\gQ{{\mathcal{Q}}}
\def\gR{{\mathcal{R}}}
\def\gS{{\mathcal{S}}}
\def\gT{{\mathcal{T}}}
\def\gU{{\mathcal{U}}}

\def\gX{{\mathcal{X}}}
\def\gY{{\mathcal{Y}}}
\def\gZ{{\mathcal{Z}}}

\DeclareMathOperator*{\E}{\mathbb{E}}
\DeclareMathOperator{\HH}{\mathbb{H}}
\DeclareMathOperator{\Var}{\rm{Var}}

\newcommand{\R}{\mathbb{R}}
\newcommand{\D}{\mathrm{D}}

\DeclareMathOperator*{\argmax}{arg\,max}
\DeclareMathOperator*{\argmin}{arg\,min}
\DeclareMathOperator*{\arginf}{arg\,inf}
\DeclareMathOperator*{\argsup}{arg\,sup}
\DeclareMathOperator*{\minimize}{minimize}
\DeclareMathOperator*{\maximize}{maximize}
\DeclareMathOperator*{\subjectto}{subject\;to}
\DeclareMathOperator*{\st}{s.t.}

\DeclarePairedDelimiterX{\infdivx}[2]{(}{)}{%
  #1\;\delimsize|\delimsize|\;#2%
}
\newcommand{\kl}[2]{\ensuremath{{\rm D}_{\rm KL}\infdivx{#1}{#2}}\xspace}

\DeclareMathOperator{\sign}{sign}

\DeclareMathOperator{\ELBO}{ELBO}

\newcommand{\defeq}{\vcentcolon=}
\newcommand{\eqdef}{=\vcentcolon}

\usepackage[utf8]{inputenc} %
\usepackage[T1]{fontenc}    %
\DeclareUnicodeCharacter{1EF3}{\`y}
\DeclareUnicodeCharacter{0301}{\'{e}}

\usepackage{url}            %
\usepackage{booktabs}       %
\usepackage{amsfonts}       %
\usepackage{nicefrac}       %
\usepackage{microtype}      %
\usepackage{enumitem}
\usepackage{slantsc}
\usepackage{caption}

\usepackage{xcolor}
\definecolor{linkcolor}{RGB}{74, 102, 146}
\definecolor{lightpurple}{RGB}{168, 141, 201}

\usepackage[
colorlinks=true,allcolors=linkcolor,pageanchor=true,
plainpages=false,pdfpagelabels,bookmarks,bookmarksnumbered,
backref=page,
pdfauthor={Brandon Amos},
pdftitle={Tutorial on amortized optimization},
]{hyperref}

\renewcommand*{\backref}[1]{}
\renewcommand*{\backrefalt}[4]{%
  \ifcase #1 \or (Cited on page~#2.)
  \else (Cited on pages~#2.)
  \fi%
}

\usepackage[nameinlink]{cleveref}
\Crefname{equation}{Eq.}{Eqs.}

\newcommand\pro{\item[$+$]}
\newcommand\con{\item[$-$]}

\usepackage{caption}
\usepackage{subcaption}
\usepackage{wrapfig}
\usepackage{tikz}
\usetikzlibrary{arrows.meta, arrows, backgrounds, bayesnet, calc, matrix}

\usepackage{listings}
\definecolor{code_green}{rgb}{0,0.6,0}
\definecolor{code_gray}{rgb}{0.5,0.5,0.5}
\definecolor{code_purple}{rgb}{.5, .21, .68}

\newcommand{\cblock}[3]{
  \hspace{-1.5mm}
  \begin{tikzpicture}
    [
    node/.style={square, minimum size=10mm, thick, line width=0pt},
    ]
    \node[fill={rgb,255:red,#1;green,#2;blue,#3}] () [] {};
  \end{tikzpicture}%
}

\begin{document}

\begin{titlepage}
\thispagestyle{empty}
\begin{center}
\textbf{\Large Tutorial on amortized optimization} \\
{\large Learning to optimize over continuous spaces} \\~\\
Brandon Amos, \emph{Meta AI}
\end{center}

\vspace{0.9cm}
\noindent\textbf{Abstract.} \\
Optimization is a ubiquitous modeling tool and is often
deployed in settings which repeatedly solve similar
instances of the same problem.
Amortized optimization methods use learning to predict the solutions to
problems in these settings, exploiting the shared structure
between similar problem instances.
These methods have been crucial in variational inference
and reinforcement learning and are capable of solving
optimization problems many orders of magnitude faster
than traditional optimization methods that do not use amortization.
This tutorial presents an introduction to the amortized optimization
foundations behind these advancements and overviews
their applications in variational inference, sparse coding,
gradient-based meta-learning, control, reinforcement learning,
convex optimization, optimal transport, and deep equilibrium networks.
The source code for this tutorial is available at
{\footnotesize\url{https://github.com/facebookresearch/amortized-optimization-tutorial}}.
\end{titlepage}

\setcounter{tocdepth}{1}
\tableofcontents

\chapter{Introduction}
\begin{figure}[t]
\centering
\includegraphics[width=2in]{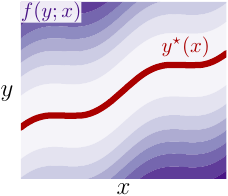}
\caption{Illustration of the parametric optimization problem
  in \cref{eq:opt}.
  Each context $x$ parameterizes an
  optimization problem that the objective $f(y; x)$ depends on.
  The contours show the values of the objectives where
  darker colors indicate higher values.
  The objective is then minimized over $y$ and the resulting
  solution $y^\star(x)$ is shown in red.
  In other words, each vertical slice is an optimization problem
  and this visualization shows a continuum of optimization problems.
}
\label{fig:opt}
\end{figure}
This tutorial studies the use of machine learning
to improve repeated solves of parametric optimization
problems of the form
\begin{equation}
  y^\star(x) \in \argmin_y f(y; x),
  \label{eq:opt}
\end{equation}
where the \emph{non-convex} objective
$f: \gY\times \gX\rightarrow \R$
takes a \emph{context} or \emph{parameterization}
$x\in\gX$ which can be continuous or discrete,
and the \emph{continuous, unconstrained domain} of
the problem is $y\in\gY=\R^n$.
\Cref{eq:opt} implicitly defines a \emph{solution}
$y^\star(x)\in\gY$.
In most of the applications considered later in
\cref{sec:apps}, $y^\star(x)$ is unique and smooth,
\ie, the solution continuously changes in a
connected way as the context changes, as illustrated
in \cref{fig:opt}.

Parametric optimization problems such as \cref{eq:opt}
have been studied for decades
\citep{bank1982non,fiacco1990sensitivity,shapiro2003sensitivity,klatte2006nonsmooth,bonnans2013perturbation,still2018lectures,fiacco2020mathematical}
with a focus on sensitivity analysis.
The general formulation in \cref{eq:opt} captures many
tasks arising in physics, engineering, mathematics, control,
inverse modeling, and machine learning.
For example, when controlling a continuous robotic system,
$\gX$ is the space of \emph{observations} or \emph{states},
\eg, angular positions and velocities describing
the configuration of the system,
the domain $\gY\defeq \gU$ is the \emph{control space},
\eg, torques to apply to each actuated joint,
and $f(u; x)\defeq -Q(u, x)$ is the \emph{control cost}
or the negated \emph{Q-value} of the state-action tuple $(x,u)$,
\eg, to reach a goal location or to maximize the velocity.
For every encountered state $x$, the system is controlled
by solving an optimization problem in the form of \cref{eq:opt}.
While $\gY=\R^n$ is over a deterministic real-valued space
in \cref{eq:opt}, the formulation can also capture
stochastic optimization problems as discussed in
\cref{sec:extensions:sto}. For example,
\Cref{sec:apps:avi} optimizes over the (real-valued)
parameters of a variational distribution and
\cref{sec:apps:ctrl} optimizes over the (real-valued)
parameters of a stochastic policy for control and
reinforcement learning.

\begin{figure}[t]
  \centering
  \includegraphics[width=\textwidth]{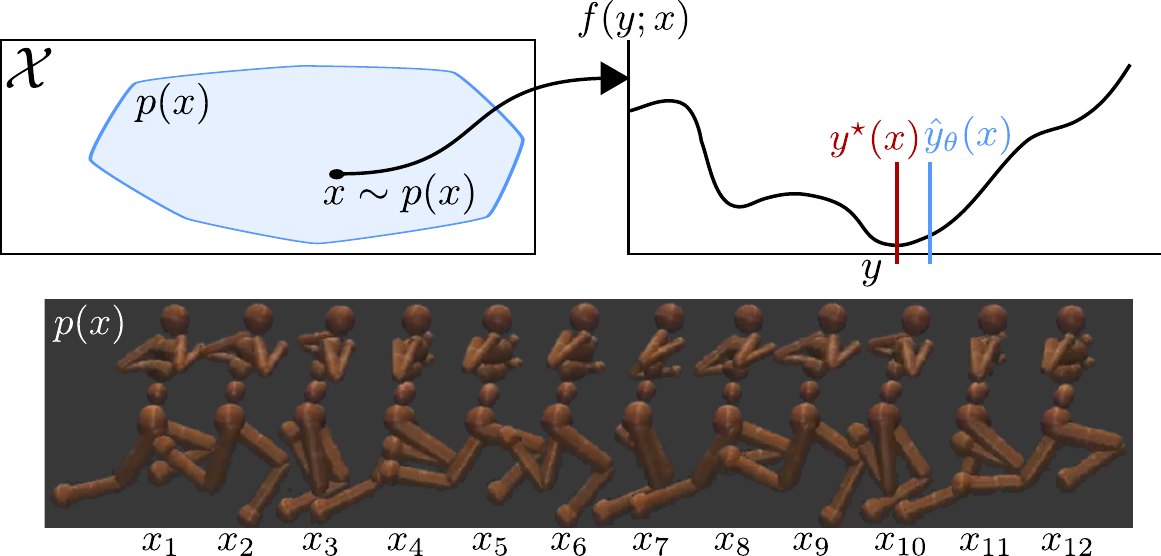}
  \caption{An amortized optimization method learns
    a model $\hat y_\theta$ to predict the minimum
    of an \emph{objective} $f(y;x)$ to a parameterized
    optimization problem, as in \cref{eq:opt},
    which depends on a \emph{context} $x$.
    For example, in control,
    the context space $\gX$ is the state space of the system,
    \eg angular positions and velocities describing
    the configuration of the system,
    the domain $\gY\defeq\gU$ is the control space,
    \eg torques to apply to each actuated joint,
    the cost (or negated value) of a state-action
    pair is $f(u; x)\defeq -Q(x,u)$, and the state distribution is $p(x)$.
    For an encountered state $x$,
    many reinforcement learning policies $\pi_\theta(x)\defeq\hat y_\theta(x)$
    amortize the solution to the underlying control problem
    with true solution $y^\star(x)$.
    This humanoid policy was obtained with the model-based
    stochastic value gradient in \citet{amos2021model}.
  }
  \label{fig:overview}
\end{figure}

Optimization problems such as \cref{eq:opt} quickly become a
computational bottleneck in systems they are a part of.
These problems often do not have a closed-form
analytic solution and are instead solved with
approximate numerical methods which iteratively
search for the solution.
This computational problem has led to many specialized
solvers that leverage domain-specific insights to
deliver fast solves.
Specialized algorithms are
especially prevalent in convex optimization methods for
linear programming, quadratic programming, cone programming,
and control and use theoretical insights of the problem
structure to bring empirical gains of computational
improvements and improved convergence
\citep{boyd2004convex,nocedal2006numerical,bertsekas2015convex,bubeck2015convex,nesterov2018lectures}.

Mostly separate from optimization research and algorithmic advancements,
the machine learning community has focused on developing
generic function approximation methods for estimating non-trivial
high-dimensional mappings from data
\citep{murphy2012machine,goodfellow2016deep,deisenroth2020mathematics}.
While machine learning models are often used to reconstruct mappings
from data, \eg for supervised classification or regression where
the targets are given by human annotations.
Many computational advancements on the software and hardware
have been developed in recent years to make the prediction time fast:
the forward pass of a neural network generating a prediction
can execute in milliseconds on a graphics processing unit.

\textbf{Overview.}
This tutorial studies the use of machine learning models to
rapidly predict the solutions to the optimization problem in
\cref{eq:opt}, which is referred to as
\emph{amortized optimization} or \emph{learning to optimize}.
Amortized optimization methods are capable of significantly
improving the computational time of
classical algorithms \emph{on a focused subset of problems}.
This is because the model is able to learn about the
solution mapping from $x$ to $y^\star(x)$ that classical
optimization methods usually do not assume access to.
My goal in writing this is to explore a unified perspective
of modeling approaches of amortized optimization in
\cref{sec:foundations} to help draw connections
between the applications in \cref{sec:apps},
\eg between amortized variational inference, meta-learning,
and policy learning for control and reinforcement learning,
sparse coding, convex optimization, optimal transport,
and deep equilibrium networks.
These topics have historically been studied in isolation
without connections between their amortization components.
\Cref{sec:implementation} presents a computational tour
through source code for variational inference, policy learning,
and a spherical optimization problem and
\cref{sec:discussion} concludes with a discussion of
challenges, limitations, open problems, and related work.

\textbf{How much does amortization help?}
Amortized optimization has been revolutionary to many fields,
especially including variational inference and reinforcement
learning.
\Cref{fig:vae-performance} shows that the amortization component
of a variational autoencoder trained on MNIST is \textbf{25000}
times faster (0.4ms vs.~8 seconds!) than solving a batch of
1024 optimization problems from scratch to obtain a
solution of the same quality.
These optimization problems are solved in every training iteration
and can become a significant bottleneck if they are
inefficiently solved.
If the model is being trained for millions of iterations,
then the difference between solving the optimization problem
in 0.4ms vs.~8 seconds makes the difference between the
entire training process finishing in a few hours or a month.

\textbf{A historic note: amortization in control and statistical inference.}
Amortized optimization has arisen in many fields as a result
to practical optimization problems being non-convex and not
having easily computed, or closed-form solutions.
Continuous control problems with linear dynamics and quadratic
cost are convex and often easily solved with the linear
quadratic regulator (LQR) and many non-convex extensions and
iterative applications of LQR have been successful over
the decades, but becomes increasingly infeasible on
non-trivial systems and in reinforcement learning settings
where the policy often needs to be rapidly executed.
For this reason, the reinforcement learning community almost
exclusively amortizes control optimization problems with
a learned policy \citep{sutton2018reinforcement}.
Related to this throughline in control and reinforcement learning,
many statistical optimization problems have closed
form solutions for known distributions such as Gaussians.
For example, the original Kalman filter is defined with Gaussians
and the updates take a closed form. The extended Kalman filter
generalizes the distributions to non-Gaussians, but the updates
are in general no longer available analytically and need to be
computationally estimated.
\citet{marino2018general} shows how amortization helps improve
this computationally challenging step.
Both of these control and statistical settings start with a
simple setting with analytic solutions to optimization problems,
generalize to more challenging optimization problems
that need to be computationally estimated, and then
add back some computational tractability with amortized optimization.

\chapter{Amortized optimization foundations}
\label{sec:foundations}

\begin{figure}[ht!]
  \centering
  \resizebox{\textwidth}{!}{
  \begin{tikzpicture}[every node/.style={align=left,anchor=west}]
    \node[align=right,text width=2.4in] at (0.,0.) (model) {Amortization model $\hat y_\theta$ (\cref{sec:model})};
    \node[right=.4in of model,yshift=2.5mm,text width=3.5in] (full) {Fully-amortized (\cref{sec:model:full}): no objective access};
    \node[below=-1mm of full,text width=3.5in] (semi) {Semi-amortized (\cref{sec:model:semi}): accesses objective};
    \node[below=.4in of model,align=right,text width=2.4in] (loss) {Amortization loss $\gL$ (\cref{sec:learning})};
    \node[right=.4in of loss,yshift=2.5mm,text width=3.5in] (regression) {Regression (\cref{sec:learning:reg}): $\E_{p(x)} \|\hat y_\theta(x)-y^\star(x)\|_2^2$};
    \node[below=-1mm of regression,text width=3.5in] (objective) {Objective (\cref{sec:learning:grad}): $\E_{p(x)} f(\hat y_\theta(x))$};
    \draw[->] (model.east) -- (full.west);
    \draw[->] (model.east) -- (semi.west);
    \draw[->] (loss.east) -- (regression.west);
    \draw[->] (loss.east) -- (objective.west);
  \end{tikzpicture}}
  \label{fig:foundations}
  \caption{Overview of amortized optimization modeling and loss choices.}
\end{figure}
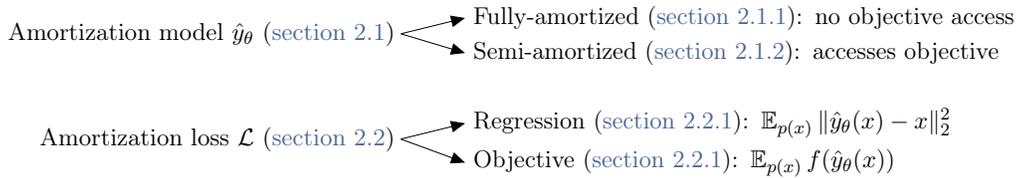

The machine learning, statistics, and optimization
communities are exploring methods of \emph{learning
to optimize} to obtain fast solvers for \cref{eq:opt}.
I will refer to these methods as \emph{amortized optimization}
as they \emph{amortize} the cost of solving the
optimization problems across many contexts to approximate
the solution mapping $y^\star$.
Amortized optimization is promising because in many applications,
there are significant correlations and structure between the
solutions which show up in $y^\star$ that a model can learn.
This tutorial follows \citet{shu2017amortized} for defining the core
foundation of amortized optimization.

\begin{definition}
  An \emph{amortized optimization method} to solve \cref{eq:opt}
  can be represented by
  $\gA\defeq (f, \gY, \gX, p(x), \hat y_\theta, \gL)$,
  where
  $f: \gY\times\gX \rightarrow \R$ is the
  unconstrained \emph{objective} to optimize,
  $\gY$ is the \emph{domain},
  $\gX$ is the \emph{context space},
  $p(x)$ is the \emph{probability distribution over contexts} to optimize,
  $\hat y_\theta: \gX\rightarrow \gY$ is the \emph{amortization model}
  parameterized by $\theta$
  which is learned by optimizing a \emph{loss}
  defined on all the components
  $\gL(f, \gY, \gX, p(x), \hat y_\theta)$.
  \label{def:amor}
\end{definition}

The objective $f$ and domain $\gY$ arise from
the problem setting along with the context space
$\gX$ and distribution over it $p(x)$, and
the remaining definitions of the model
$\hat y_\theta$ and loss $\gL$ are application-specific
design decisions that \cref{sec:model,sec:learning}
opens up.
These sections present the modeling and loss foundations
for the core problem in \cref{def:amor} agnostic
of specific downstream applications that will use them.
The key choices highlighted in \cref{fig:foundations}
are how much information
1) the model $\hat y_\theta$ has about the
objective $f$ (fully- vs.~semi-amortized), and
2) the loss has about the true solution $y^\star$
(regression- vs.~objective-based).
\Cref{fig:overview} instantiates these components
for amortizing the control of a robotic system.
The model $\hat y_\theta$ solves the solution mapping
$y^\star$ simultaneously for all contexts.
The methods here usually assume the solution mapping
$y^\star$ to be almost-everywhere smooth
and well-behaved.
The best modeling approach is an open research topic
as there are many tradeoffs,
and many specialized insights from the application domain
can significantly improve the performance.
The generalization capacity along with the model's
convergence guarantees are challenging topics
which \cref{sec:generalization} covers in more detail.

\textbf{Origins of the term ``amortization'' for optimization.}
The word ``amortization'' generally means to spread out costs and
thus ``amortized optimization'' usually means to spread out
computational costs of the optimization process.
The term originated in the variational inference community
for inference optimization
\citep{kingma2013auto,rezende2014stochastic,stuhlmuller2013learning,gershman2014amortized,webb2017faithful,ravi2018amortized,cremer2018inference,wu2020meta},
and is used more generally in
\citet{xue2020amortized,sercu2021neural,xiao2021amortized}.
\citet[p.~28]{marino2021learned} give further background on the
origins and uses of amortization.
Concurrent to these developments, other communities have
independently developed amortization methods without referring to them
by the same terminology and analysis, such as in reinforcement learning,
policy optimization, and sparse coding ---
\cref{sec:apps} connects all of these under \cref{def:amor}.

\textbf{Conventions and notation.}
The context space $\gX$ represents the
sample space of a probability space that
the distribution $p(x)$ is defined on,
assuming it is Borel if not otherwise specified.
For a function $f: \R^n\rightarrow\R$ in standard Euclidean space,
$\nabla_x f(\bar x)\in \R^n$ denotes the \emph{gradient} at a point $\bar x$
and $\nabla^2_x f(\bar x)\in\R^{n\times n}$ denotes the \emph{Hessian}.
For $f: \R^n\rightarrow\R^m$, $\D_x f(\bar x)\in\R^{m\times n}$
represents the \emph{Jacobian} at $\bar x$ with
entries $[\D_x f(\bar x)]_{ij}\defeq \frac{\partial f_i}{\partial x_j}(\bar x)$.
I abbreviate the loss to $\gL(\hat y_\theta)$
when the other components can be inferred from the
surrounding text and prefer the term ``context'' for $x$ instead of
``parameterization'' to make the distinction between the
$x$-parameterized optimization problem
and the $\theta$-parameterized model clear.
I use ``;'' as separation in $f(y; x)$ to emphasize the separation
between the domain variables $y$ that \cref{eq:opt}
optimizes over from the context ones $x$ that remain fixed.
A model's parameters $\theta$ are usually subscripts
as $h_\theta(x)$ but I will equivalently write
$h(x; \theta)$ sometimes.

\section{Defining the model $\hat y_\theta(x)$}
\label{sec:model}
The model $\hat y_\theta(x): \gX\times\Theta\rightarrow \gY$
predicts a solution to \cref{eq:opt}.
In many applications, the best model design is an active
area of research that is searching for models that are
expressive and more computationally efficient than the
algorithms classically used to solve the optimization problem.
\Cref{sec:model:full} starts simple with \emph{fully-amortized} models
that approximate the entire solution
to the optimization problem with a single black-box model.
Then \cref{sec:model:semi} shows how to open up the model
to include more information about the optimization problem
that can leverage domain knowledge with \emph{semi-amortized} models.

\subsection{Fully-amortized models}
\label{sec:model:full}
\begin{definition}
  A \emph{fully-amortized} model $\hat y_\theta: \gX\rightarrow\gY$
  maps the context to the solution of \cref{eq:opt}
  and does \emph{not} access the objective $f$.
\end{definition}

I use the prefix ``fully'' to emphasize
that the entire computation of the solution to the
optimization problem is absorbed into
a black-box model that does \emph{not} access the objective $f$.
The prefix ``fully'' can be omitted when the context is clear
because most amortization is fully amortized.
These are standard in amortized variational inference
(\cref{sec:apps:avi}) and policy
learning (\cref{sec:apps:ctrl}), that typically use
feedforward neural networks to map from
the context space $\gX$ to the solution of the
optimization problem living in $\gY$.
Fully-amortized models are remarkable because they
are often successfully able to predict the solution
to the optimization problem in \cref{eq:opt}
\emph{without} ever accessing the objective of
the optimization problem after being trained.

Fully-amortized models are the most useful for attaining approximate
solutions that are computationally efficient.
They tend to work the best when the
solution mappings $y^\star(x)$ are predictable,
the domain $\gY$ is relatively small,
usually hundreds or thousands of dimensions,
and the context distribution isn't too large.
When fully-amortized models don't work well,
semi-amortized models help open up the black box
and use information about the objective.

\subsection{Semi-amortized models}
\label{sec:model:semi}
\begin{definition}
  A \emph{semi-amortized} model $\hat y_\theta: \gX\rightarrow\gY$
  maps the context to the solution of the optimization problem
  and accesses the objective $f$ of \cref{eq:opt},
  typically iteratively.
\end{definition}

\citet{kim2018semi,marino2018iterative}
proposed \emph{semi-amortized} models for variational inference
that add back domain knowledge of the optimization problem
to the model $\hat y_\theta$ that the fully-amortized
models do not use.
These are brilliant ways of integrating the optimization-based
domain knowledge into the learning process.
The model can now internally integrate
solvers to improve the prediction.
Semi-amortized methods are typically iterative and update
iterates in the domain $\gY$ or in an \emph{auxiliary}
or \emph{latent space} $\gZ$.
I refer to the space the semi-amortization iterates over
as the \emph{amortization space} and denote iterate $t$
in these spaces, respectively, as $\hat y^{t}_\theta$ and $z^t_\theta$.
While the iterates and final prediction $\hat y_\theta$
can now query the objective $f$ and gradient $\nabla_y f$,
I notationally leave this dependence implicit for
brevity and only reference these queries in the relevant definitions.

\subsubsection{Semi-amortized models over the domain $\gY$}
\label{sec:semi-domain}
\begin{center}
\begin{tikzpicture}
  \matrix (m) [
      matrix of math nodes,row sep=2em,column sep=1em,
      minimum width=2em,nodes={anchor=center}
  ] {
    \hat y^{0}_\theta & \hat y_\theta^{1} &
    \ldots & \hat y_\theta^{K}\eqdef \hat y_\theta(x) \\
  };
  \path[-stealth]
    (m-1-1) edge node {} (m-1-2)
    (m-1-2) edge node {} (m-1-3)
    (m-1-3) edge node {} (m-1-4);
\end{tikzpicture}
\end{center}
\vspace{-3mm}

One of the most common semi-amortized model is to
parameterize and integrate an optimization procedure
used to solve \cref{eq:opt} into the model $\hat y_\theta$,
such as gradient descent \citep{andrychowicz2016learning,finn2017model,kim2018semi}.
This optimization procedure is an internal part of
the amortization model $\hat y_\theta$,
often referred to as the \emph{inner-level} optimization
problem in the bi-level setting that arises for learning.

\textbf{Examples.}
This section instantiates a canonical semi-amortized model based
gradient descent that learns the initialization as in
model-agnostic meta-learning (MAML) by \citet{finn2017model},
structured prediction energy networks (SPENs) by \citet{belanger2017end},
and semi-amortized variational auto-encoders (SAVAEs) by \citet{kim2018semi}.
The initial iterate $\hat y^{0}_\theta(x)\defeq \theta$
is parameterized by $\theta\in\gX$ for all contexts.
Iteratively updating $\hat y^{t}_\theta$ for $K$ gradient steps
with a \emph{learning rate} or \emph{step size} $\alpha\in\R_+$
on the objective $f(y;x)$ gives
\begin{equation}
  \hat y^{t}_\theta \defeq \hat y^{t-1}_\theta - \alpha \nabla_y f(\hat y^{t-1}_\theta; x) \qquad t\in\{1\ldots,K\},
  \label{eq:gd}
\end{equation}
where model's output is defined as
$\hat y_\theta\defeq \hat y^{K}$.

Semi-amortized models over the domain can go significantly beyond
gradient-based models and in theory, any algorithm to solve
the original optimization problem in \cref{eq:opt}
can be integrated into the model.
\Cref{sec:unrolled} further discusses the learning of
semi-amortized models by unrolling that are instantiated later:
\begin{itemize}
\item \Cref{sec:apps:lista} discusses how
\citet{gregor2010learning} integrate ISTA iterates
\citep{daubechies2004iterative,beck2009fast}
into a semi-amortized model.
\item \Cref{sec:apps:neural-fp} discusses models that integrate
fixed-point computations into semi-amortized models.
\citet{venkataraman2021neural} amortize convex cone programs by
differentiating through the splitting cone solver \citep{o2016conic}
and \citet{bai2022neural} amortize
deep equilibrium models \citep{bai2019deep,bai2020multiscale}.
\item \Cref{sec:apps:qprl} discusses RLQP by \citet{ichnowski2021accelerating}
  that uses the OSQP solver \citep{stellato2018osqp} inside
  of a semi-amortized model.
\end{itemize}

\subsubsection{Semi-amortized models over a latent space $\gZ$}
\label{sec:semi-latent}
\begin{center}
\begin{tikzpicture}
  \matrix (m) [
      matrix of math nodes,row sep=2em,column sep=1em,
      minimum width=2em,nodes={anchor=center}
  ] {
    \hat z^{0}_\theta & \hat z_\theta^{1} &
    \ldots & \hat z_\theta^{K} & \hat y_\theta(x) \\};
  \path[-stealth]
    (m-1-1) edge node {} (m-1-2)
    (m-1-2) edge node {} (m-1-3)
    (m-1-3) edge node {} (m-1-4)
    (m-1-4) edge node {} (m-1-5);
\end{tikzpicture}
\end{center}
\vspace{-3mm}

In addition to only updating iterates over the domain $\gY$,
a natural generalization is to introduce a latent space $\gZ$
that is iteratively optimized over \emph{inside} of
the amortization model.
This is usually done to give the semi-amortized model
more capacity to learn about the structure of the optimization
problems that are being solved.
The latent space can also be interpreted as a representation
of the optimal solution space.
This is useful for learning an optimizer that only searches
over the \emph{optimal} region of the solution space rather
than the entire solution space.

\textbf{Examples.}
The iterative gradient updates in \cref{eq:gd}
can be replaced with a learned update function as in
\citet{ravi2016optimization,li2016learning,andrychowicz2016learning,li2017learning}.
These model the past sequence of iterates and learn how
to best-predict the next iterate, pushing them towards optimality.
This can be done with a recurrent cell $g$ such
as an LSTM \citep{hochreiter1997long} or GRU \citep{cho2014learning}
and leads to updates of the form
\begin{equation}
  z^t_\theta, \hat y^{t}_\theta \defeq g_\theta(z^{t-1}_\theta, x^{t-1}_\theta, \nabla_y f(\hat y^{t-1}_\theta; x)) \qquad t\in\{1\ldots,K\}
  \label{eq:rec}
\end{equation}
where each call to the recurrent cell $g$
takes a hidden state $z$ along with an iterate and
the derivative of the objective.
This endows $g$ with the capacity to learn significant
updates leveraging the problem structure that a
traditional optimization method would not be able
to make.
In theory, traditional update rules can also be
fallen back on as the gradient step in \cref{eq:gd}
is captured by removing the hidden state $z$ and
setting
\begin{equation}
  g(x, \nabla_y f(y; x))\defeq x-\alpha\nabla_y f(y; x).
  \label{eq:rec-grad-step}
\end{equation}

Latent semi-amortized models are a budding topic and can
excitingly learn many other latent representations
that go beyond iterative gradient updates in the original
space.
\citet{luo2018neural,amos2019dcem}
learn a \emph{latent domain} connected to the
original domain where the latent domain captures
hidden structures and redundancies present in
the original high-dimensional domain $\gY$.
\citet{luo2018neural} consider gradient updates
in the latent domain and \citet{amos2019dcem}
show that the cross-entropy method \citep{de2005tutorial}
can be made differentiable and learned as an alternative
to gradient updates.
\citet{amos2017input} unrolls and differentiates through
the bundle method \citep{smola2007bundle} in a convex setting
as an alternative to gradient steps.
The latent optimization could also be done over a learned parameter
space as in POPLIN \citep{wang2019exploring}, which \emph{lifts}
the domain of the optimization problem \cref{eq:opt}
from $\gY$ to the parameter space of a fully-amortized neural network.
This leverages the insight that the parameter space of
over-parameterized neural networks can induce easier
non-convex optimization problems than in the original
space, which is also studied in \citet{hoyer2019neural}.

\subsubsection{Comparing semi-amortized models with warm-starting}
Semi-amortized models are conceptually similar to learning a fully-amortized model
to warm-start an existing optimization procedure that fine-tunes the solution.
The crucial difference is that semi-amortized learning often end-to-end learns
through the final prediction while warm-starting and fine-tuning only learns
the initial prediction and does not integrate the knowledge of the fine-tuning
procedure into the learning procedure.
Choosing between these is an active research topic and while this
tutorial will
mostly focus on semi-amortized models, learning a fully-amortized
warm-starting model brings promising results to some fields too,
such as \citet{zhang2019safe,baker2019learning,chen2022large}.
In variational inference, \citet[Table 2]{kim2018semi} compare semi-amortized
models (SA-VAE) to warm-starting and fine-tuning (VAE+SVI) and demonstrate
that the end-to-end learning signal is helpful.
In other words, amortization finds an initialization that is
helpful for gradient-based optimization.
\citet{arbel2021amortized} further study fully-amortized warm-started
solvers that arise in bi-level optimization problems for
hyper-parameter optimization and use the theoretical framework from
singularly perturbed systems \citep{habets2010stabilite}
to analyze properties of the approximate solutions.

\subsubsection{On second-order derivatives of the objective}
\label{sec:second-derivatives}

Training a semi-amortized model is usually more computationally
challenging than training a fully-amortized model.
This section looks at how second-order derivatives of the
objective may come up when unrolling and create a
computational bottleneck when learning a semi-amortized model.
The next derivation follows \citet[\S5]{nichol2018first}
and \citet{weng2018metalearning} and shows the model derivatives
that arise when composing a semi-amortized model with a loss.

\textbf{Starting with a single-step model.}
This section instantiates a single-step model
similar to \cref{eq:gd} that parameterizes the initial
iterate $\hat y^0_\theta(x)\defeq\theta$ and takes one gradient step:
\begin{equation}
  \hat y_\theta(x)\defeq \hat y^0_\theta(x) - \alpha \nabla_y f(\hat y^0_\theta(x); x)
  \label{eq:single-step}
\end{equation}
Interpreting $\hat y_\theta(x)$ as a model is non-standard in contrast
to other parametric models because it makes the
optimization step \emph{internally part of the model}.
Gradient-based optimization of losses with respect to the model's parameters,
such as \cref{eq:reg-loss,eq:grad-loss} requires the Jacobian of $\hat y_\theta(x)$
w.r.t.~the parameters, \ie $\D_\theta[\hat y_\theta(x)]$
(or Jacobian-vector products with it).
Because $\hat y_\theta(x)$ is an optimization step, the derivative
of the model requires differentiating through the optimization step,
which for \cref{eq:single-step} is
\begin{equation}
  \D_\theta[\hat y_\theta(x)] = I-\alpha \nabla_y^2 f(y_\theta^0(x); x)
  \label{eq:second-derivatives-single-step}
\end{equation}
and requires the Hessian of the objective.
In \citet{finn2017model}, $\nabla_y^2 f$ is the Hessian of the
model's parameters on the training loss (!) and is
compute- and memory-expensive to instantiate for large models.
In practice, the Hessian in \cref{eq:second-derivatives-single-step}
is often never explicitly instantiated as optimizing the
loss only requires Hessian-vector products.
The Hessian-vector product can be computed exactly or
estimated without fully instantiating the Hessian, similar to
how computing the derivative of a neural network with backprop
does not instantiate the intermediate Jacobians and only computes
the Jacobian-vector product.
More information about efficiently computing Hessian-vector products
is available in \citet{pearlmutter1994fast,domke2012generic}.
Jax's \href{https://github.com/google/jax/blob/27360b9/docs/notebooks/autodiff_cookbook.ipynb}{autodiff cookbook}
\citep{bradbury2020jax}
further describes efficient Hessian-vector products.
Before discussing alternatives,
the next portion derives similar results for a $K$-step model.

\textbf{Multi-step models.}
\Cref{eq:single-step} can be extended to the $K$-step setting with
\begin{equation}
\hat y_\theta^K(x)\defeq \hat y^{K-1}_\theta(x) - \alpha\nabla_y f(\hat y^{K-1}_\theta(x); x),
\end{equation}
where the base $\hat y_\theta^0(x)\defeq\theta$ as before.
Similar to \cref{eq:second-derivatives-single-step},
the derivative of a single step is
\begin{equation}
  \D_\theta[\hat y_\theta^K(x)] = \D_\theta[\hat y_\theta^{K-1}(x)]\left(I-\alpha\nabla_y^2 f(y_\theta^{K-1}(x); x)\right),
  \label{eq:second-derivatives-recurrent}
\end{equation}
and composing the derivatives down to $\hat y_\theta^0$ yields the product structure
\begin{equation}
  \D_\theta[\hat y_\theta^K(x)] = \prod_{k=0}^{K-1} \left( I-\alpha\nabla_y^2 f(y_\theta^k(x); x) \right),
  \label{eq:second-derivatives-k-steps}
\end{equation}
where $\D_\theta[\hat y_\theta^0(x)]=I$ at the base case.
Computing \cref{eq:second-derivatives-k-steps} is
now $K$ times more challenging as it requires the Hessian
$\nabla_y^2 f$ at \emph{every} iteration of the model.
While using Hessian-vector products can alleviate some
computational burden of this term, it often still requires
significantly more operations than most other derivatives.

\textbf{Computationally cheaper alternatives.}
The first-order MAML baseline in \citet{finn2017model} suggests to
simply not use the second-order terms $\nabla_y^2 f$ here,
approximating the model derivative as the identity,
\ie $\D_\theta[\hat y_\theta^K(x)]\approx I$,
and relying on only information from the outer loss
to update the parameters.
They use the intuition from \citet{goodfellow2014explaining}
that neural networks are locally linear and therefore these
second-order terms of $f$ are not too important.
They show that this approximation works well in some cases,
such as MiniImagenet \citep{ravi2016optimization}.
The MAML$++$ extension by \citet{antoniou2018train} proposes to
use first-order MAML during the early phases of training, but
to later add back this second-order information.
\citet{nichol2018first} further analyze first-order approximations
to MAML and propose another approximation called Reptile that
also doesn't use this second-order information.
These higher-order terms also come up when unrolling in the
different bi-level optimization setting for hyper-parameter optimization,
and \citet[Table 1]{lorraine2020optimizing} gives
a particularly good overview of approximations to these.
Furthermore, memory-efficient methods for training neural networks
and recurrent models with backpropagation and
unrolling such as \citet{gruslys2016memory,chen2016training}
can also help improve the memory utilization in amortization models.

\textbf{Parameterizing and learning the objective.}
While this section has mostly not considered the setting when
the objective $f$ is also learned,
the second-order derivatives appearing in
\cref{eq:second-derivatives-k-steps}
also cause issues in when the objective is parameterized
and learned.
In addition to learning an initial iterate, \citet{belanger2017end}
learn the objective $f$ representing an energy function.
They parameterize $f$ as a neural network and use softplus
activation functions rather than ReLUs to ensure the
objective's second-order derivatives are non-zero.

\subsection{Models based on differentiable optimization}
As discussed in \cref{sec:learning}, the model typically needs
to be (sub-)differentiable with respect to the parameters
to attain the Jacobian $\D_\theta[\hat y_\theta]$
(or compute Jacobian-vector products with it)
necessary to optimize the loss.
These derivatives are standard backprop when the model
is, for example, a full-amortized neural network, but
in the semi-amortized case, the model itself is often an
optimization process that needs to be differentiated through.
When the model updates are objective-based as in
\cref{eq:gd} and \cref{eq:rec}, the derivatives with respect
to $\theta$ through the sequence of gradient updates
in the domain can be attained by seeing the updates
as a sequence of computations that are differentiated through,
resulting in second-order derivatives.
When more general optimization methods are used for
the amortization model that may not have a closed-form
solution, the tools of differentiable optimization
\citep{domke2012generic,gould2016differentiating,amos2017optnet,amos2019differentiable,agrawal2019differentiable}
enable end-to-end learning.

\subsection{Practically choosing a model}
This section has taxonomized how to instantiate an amortization
model in an application-agnostic way. As in most machine learning
settings in practice, the modeling choice is often
application-specific and needs to take into consideration many factors.
This may include 1) the speed and expressibility of the model,
2) adapting the model to specific context space $\gX$.
An MLP may be good for fixed-dimensional
real-valued spaces but a convolutional neural network
is likely to perform better for image-based spaces.
3) taking the solution space $\gY$ into consideration.
For example, if the solution space is an image space,
then a standard vision model capable of predicting
high-dimensional images is reasonable, such as a
U-net \citep{ronneberger2015u},
dilated convolutional network \citep{yu2015multi}
or fully convolutional network \citep{long2015fully}.
4) the model also may need to adapt to a
\emph{variable-length} context or solution space.
This arises in VeLO \citep{metz2022velo} for learning
to optimize machine learning models where the model needs
to predict the parameters of different models that may
have different numbers of parameters.
Their solution is to decompose the structure of the
parameter space and to formulate the semi-amortized
model as a sequence model that predicts smaller MLPs
that operate on smaller groups of parameters.

\section{Learning the model's parameters $\theta$}
\label{sec:learning}

\begin{figure}[t]
  \centering
  \includegraphics[width=0.32\textwidth]{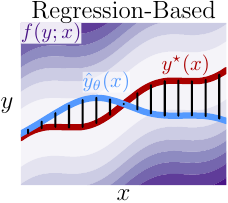}
  \hspace{10mm}
  \includegraphics[width=0.32\textwidth]{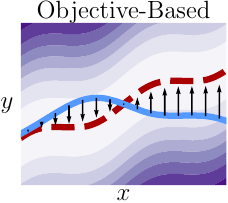}
  \caption{Overview of key losses for optimizing
    the parameters $\theta$ of the amortization model $\hat y_\theta$.
    Regression-based losses optimize a distance between
    the model's prediction $\hat y_\theta(x)$ and the
    ground-truth $y^\star(x)$. Objective-based
    methods update $\hat y_\theta$ using
    local information of the objective $f$
    and \emph{without} access to the ground-truth
    solutions $y^\star$.
  }
  \label{fig:learning}
\end{figure}

After specifying the amortization model $\hat y_\theta$, the other major
design choice is how to learn the parameters $\theta$
so that the model best-solves \cref{eq:opt}.
Learning is often a \emph{bi-level} optimization problem
where the \emph{outer level} is the parameter
learning problem for a model $\hat y_\theta(x)$ that solves
the \emph{inner-level} optimization problem in \cref{eq:opt}
over the domain $\gY$.
While defining the best loss is application-specific,
most approaches can be roughly categorized as
1) regressing a ground-truth solution (\cref{sec:learning:reg}),
or 2) minimizing the objective (\cref{sec:learning:grad,sec:learning:rl}),
which \cref{fig:learning} illustrates.
Optimizing the model parameters here can in theory be done
with most parameter learning methods that incorporate
zeroth-, first-, and higher-order information
about the loss being optimized, and this section mostly focuses
on methods where $\theta$ is learned with a
first-order gradient-based method
such as \citet{nesterov1983method,duchi2011adaptive,zeiler2012adadelta,kingma2014adam}.
The rest of this section discusses approaches for
designing the loss and optimizing the parameters
with first-order methods
(\cref{sec:learning:reg,sec:learning:grad})
when differentiation is easy
or zeroth-order methods (\cref{sec:learning:rl})
otherwise, \eg, in non-differentiable settings.

\subsection{Choosing the objective for learning}
\subsubsection{Regression-based learning}
\label{sec:learning:reg}
Learning can be done by regressing the
model's prediction $\hat y_\theta(x)$ onto a
ground-truth solution $y^\star(x)$.
These minimize some distance between the predictions
and ground-truth so that the expectation over
the context distribution $p(x)$ is minimal.
With a Euclidean distance, for example,
regression-based learning solves
\begin{equation}
  \argmin_\theta \gL_{\rm reg}(\hat y_\theta) \qquad
  \gL_{\rm reg}(\hat y_\theta) \defeq
  \E_{x \sim p(x)} \|y^\star(x) - \hat y_\theta(x) \|_2^2.
\label{eq:reg-loss}
\end{equation}
$\gL_{\rm reg}$ is typically optimized with an adaptive first-order
gradient-based method that is able to directly differentiate
the loss with respect to the model's parameters.

Regression-based learning works the best for distilling
known solutions into a faster model that can be deployed
at a much lower cost, but can otherwise start failing to work.
In RL and control, regression-based amortization
methods are referred to as \emph{behavioral cloning} and
is a widely-used way of recovering a policy using
trajectories observed from an expert policy.
Using regression is also advantageous when evaluating
the objective $f(y;x)$ incurs a computationally intensive
or otherwise complex procedure, such as an evaluation of
the environment and dynamics in RL, or for
computing the base model gradients when learning parameter optimizers.
These methods work well when the ground-truth solutions
are unique and semi-tractable, but can fail otherwise,
\ie if there are many possible ground-truth
solutions for a context $x$ or if computing them
is too intractable.
After all, solving \cref{eq:opt} from scratch may be
computationally expensive and amortization methods
should improve the computation time.

\begin{remark}
  \Cref{eq:reg-loss} can be extended to other distances
  defined on the domain, such as non-Euclidean distances
  or the likelihood of a probabilistic model that predicts
  a distribution of possible candidate solutions.
  \citet{adler2017learning} propose to use the Wasserstein
  distance for learning to predict the solutions to
  inverse imaging problems.
\end{remark}

\subsubsection{Objective-based learning}
\label{sec:learning:grad}

Instead of regressing onto the ground-truth solution,
\emph{objective-based} learning methods seek for
the model's prediction to be minimal under the objective $f$ with:
\begin{equation}
  \argmin_\theta \gL_{\rm obj}(\hat y_\theta) \qquad \gL_{\rm obj}(\hat y_\theta) \defeq \E_{x \sim p(x)} f(\hat y_\theta(x); x).
\label{eq:grad-loss}
\end{equation}
These methods use local information of the objective
to provide a descent direction for the model's
parameters $\theta$.
A first-order method optimizing \cref{eq:grad-loss} uses
updates based on the gradient
\begin{equation}
  \begin{aligned}
  \nabla_\theta \gL_{\rm obj}(\hat y_\theta) &= \nabla_\theta \left[ \E_{x\sim p(x)}f(\hat y_\theta(x); x)\right] \\
  &= \E_{x\sim p(x)} \D_\theta\left[ \hat y_\theta(x)\right]^\top
  \nabla_{y} \left[ f(\hat y_\theta(x); x) \right],
  \end{aligned}
  \label{eq:grad-loss-grad}
\end{equation}
where the last step is obtained by the chain rule.
This has the interpretation that the model's parameters $\theta$
are updated by combining the gradient information around the prediction
$\nabla_{y} \left[ f(\hat y_\theta(x); x) \right]$
shown in \cref{fig:learning} along with
how $\theta$ impacts the model's predictions with the derivative
$\D_\theta\left[ \hat y_\theta(x)\right]$.
While this tutorial mostly focuses on optimizing
\cref{eq:grad-loss-grad} with
first-order methods that explicitly differentiate
the objective, \cref{sec:learning:rl} discusses alternatives
to optimizing it with reinforcement learning and
zeroth-order methods.

Objective-based methods thrive when the gradient information
is informative and the objective and models are easily differentiable.
Amortized variational inference methods and
actor-critic methods both make extensive use of
objective-based learning.

\begin{remark}
  A standard gradient-based optimizer for \cref{eq:opt} (without amortization)
  can be recovered from $\gL_{\rm obj}$ by setting the model to the identity
  of the parameters, \ie $\hat y_\theta(x)\defeq\theta$, and $p(x)$ to be a
  Dirac delta distribution.
  \label{remark:obj-loss}
\end{remark}

This can be seen by taking $\D_\theta[\hat y_\theta(x)]=I$
in \cref{eq:grad-loss-grad}, resulting in
$\nabla_\theta \gL_{\rm obj}(\hat y_\theta)=\nabla_y f(\hat y_\theta(x); x)$.
Thus optimizing $\theta$ of this parameter-identity model with gradient descent
is identical to solving \cref{eq:opt} with gradient descent.
\Cref{remark:obj-loss} shows a connection between a model trained with
gradients of an objective-based loss and a non-amortized gradient-based
solver for \cref{eq:opt}.
The gradient update that would originally have been applied to an
iterate $y\in\gY$ of the domain is now transferred into the
model's parameters that are shared across all problem instances.
This also leads to a hypothesis that objective-based amortization
works best when a gradient-based optimizer is able to successfully
solve \cref{eq:opt} from scratch. However, there may be settings
where a gradient-based optimizer performs poorly but an amortized
optimizer excels because it is able to use information from the
other problem instances.

\begin{remark}
The objective-based loss in \cref{sec:learning:grad} provides a starting
point for amortizing with other optimality conditions or reformulations
of the optimization problem. This is done when amortizing for fixed-point
computations and convex optimization in \cref{sec:apps:convex}, as well as
in optimal transport \cref{sec:apps:ot}.
\label{rmk:optimality-loss}
\end{remark}

\subsubsection{Comparing the regression- and objective-based losses}
\begin{figure}[t]
\centering
\includegraphics[width=2in]{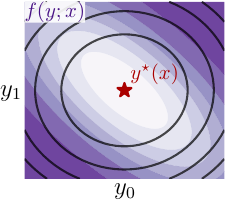}
\caption{Contours of the regression-based amortization loss
  $\gL_{\rm reg}$ (in black) alongside the contours of the objective
  (in purple where darker colors indicate higher values).
  This shows the inaccuracies of the regression-based loss, \eg
  along a level set, may impact the overall objective differently.
}
\label{fig:reg-contours}
\end{figure}

Choosing between the regression- and objective-based losses
is challenging as they measure the solution quality in
different ways and have different convergence and
locality properties.
\citet{liu2022teaching} experimentally compare these
losses for learning to optimize with fully-amortized set-based models.
\Cref{fig:reg-contours} illustrates that the $\ell_2$-regression
loss (the black contours) ignores the objective
values (the purple contours) and thus gives the same loss
to solutions that result in significantly different
objective values.
This could be potentially addressed by normalizing or
re-weighting the dimensions for regression to be more aware
of the curvature of the objective, but this is often not done.
Another idea is to combine both the objective and regression losses.
Combining the losses could work especially well when only a
few contexts are labeled, such as the regression and residual terms
in the physics-informed neural operator paper \citep{li2021physics}.
The following summarizes some other advantages ($+$) and
disadvantages ($-$):
\vspace{2mm}

\noindent\begin{minipage}[t]{2.7in}
\noindent\textbf{Regression-based losses $\gL_{\rm reg}$}
\begin{itemize}[leftmargin=*,noitemsep]
\con Often does not have access to $f(y; x)$
\pro If $f(y; x)$ is computationally expensive, does not need to compute it
\pro Uses global information with $y^\star(x)$
\con It may be expensive to compute $y^\star(x)$
\pro Does not need to compute $\nabla_y f(y; x)$
\con May be hard when $y^\star(x)$ is not unique
\end{itemize}
\end{minipage}\hspace{5mm}\begin{minipage}[t]{3in}
\noindent\textbf{Objective-based losses $\gL_{\rm obj}$}
\begin{itemize}[leftmargin=*,noitemsep]
\pro Uses objective information of $f(y; x)$
\con Can get stuck in local optima of $f(y; x)$
\pro Faster, does not require $y^\star(x)$
\con Often requires computing $\nabla_y f(y; x)$
\pro Easily learns non-unique $y^\star(x)$
\end{itemize}
\end{minipage}
\hspace{-1in}

\subsection{Learning iterative semi-amortized models}
\label{sec:learning:iter}

Fully-amortized or semi-amortized models can be learned
with the regression- and objective-based losses.
This section discusses how the loss can be further opened up
and crafted to learn iterative semi-amortized methods.
For example, if the model produces intermediate predictions
$\hat y_\theta^i$ in every iteration $i$, then instead of
optimizing the loss of just the final prediction,
\ie $\gL(\hat y_\theta^K)$, a more general loss
$\gL^\Sigma$ may consider the impact of every iteration
of the model's prediction
\begin{equation}
  \argmin_\theta \gL^\Sigma(\hat y_\theta) \qquad \gL^\Sigma(\hat y_\theta) \defeq \sum_{i=0}^K w_i \gL(\hat y_\theta^i),
\label{eq:iter-loss}
\end{equation}
where $w_i\in\R_+$ are weights in every iteration $i$
that give a design choice of how important
it is for the earlier iterations to produce reasonable
solutions.
For example, setting $w_i=1$ encourages every iterate
to be low.

Learning iterative semi-amortized methods also has (loose)
connections to sequence learning models that arise in,
\eg text, audio, and language processing.
Given the context $x$, an iterative semi-amortized model
seeks to produce a sequence of predictions that ultimately
result in the intermediate and final predictions,
which can be analogous to a language model predicting
future text given the previous text as context.
One difference is that semi-amortized models do not necessarily attempt
to model the probabilistic dependencies of a structured output space
(such as language) and instead only needs to predict
intermediate computation steps for solving an optimization problem.
The next section discusses concepts that arise when computing the derivatives
of a loss with respect to the model's parameters.

\subsubsection{Unrolled optimization and backpropagation through time}
\label{sec:unrolled}
\begin{center}
\begin{tikzpicture}
  \matrix (m) [
      matrix of math nodes,row sep=2em,column sep=1em,
      minimum width=2em,nodes={anchor=center}
  ] {
    \hat z^{0}_\theta & \hat z_\theta^{1} &
    \ldots & \hat z_\theta^{K} & \hat y_\theta(x) & \gL \\
    \; & \; & \; & \; & \; & \; \\};
  \node at (1.,0.1) () {\color{lightpurple}\ldots};
  \path[-stealth]
    (m-1-1) edge node {} (m-1-2)
    (m-1-2) edge node {} (m-1-3)
    (m-1-3) edge node {} (m-1-4)
    (m-1-4) edge node {} (m-1-5)
    (m-1-5) edge node {} (m-1-6)
    (m-1-6) edge[draw=lightpurple,out=-155,in=-15] node[below] {} (m-1-5)
    (m-1-6) edge[draw=lightpurple,out=-155,in=-15] node {} (m-1-4)
    (m-1-6) edge[draw=lightpurple,out=-155,in=-15] node {} (m-1-2)
    (m-1-6) edge[draw=lightpurple,out=-155,in=-15] node {} (m-1-1)
    ;
\end{tikzpicture}
\end{center}
\vspace{-10mm}

The parameterization of \emph{every} iterate $z_\theta^i$ can
influence the final prediction $\hat y_\theta$ and thus
losses on top of $\hat y_\theta$ need to consider the
entire chain of computations.
Differentiating through this kind of iterative procedure
is referred to as \emph{backpropagation through time}
in sequence models and \emph{unrolled optimization}
\citep{pearlmutter2008reverse,zhang2010multi,maclaurin2015gradient,belanger2016structured,metz2016unrolled,finn2017model,han2017alternating,belanger2017end,belanger2017deep,foerster2017learning,bhardwaj2020differentiable,monga2021algorithm}
when the iterates are solving an optimization problem.
The term ``unrolling'' arises because the model computation
is iterative and computing
$\D_\theta[\hat y_\theta(x)]$
requires saving and differentiating the ``unrolled''
intermediate iterations, as in \cref{sec:second-derivatives}.
The terminology ``unrolling'' here emphasizes that the
iterative computation produces a compute graph of operations
and is likely inspired from
\emph{loop unrolling} in compiler optimization
\citep{aho1986compilers,davidson1995aggressive} where
loop operations are inlined for efficiency and written
as a single chain of repeated operations rather
than an iterative computation of a single operation.

Even though $\D_\theta \hat y_\theta$ through unrolled
optimization is well-defined, in practice it can be
unstable because of exploding gradients
\citep{pearlmutter1996investigation,pascanu2013difficulty,maclaurin2016modeling,parmas2018pipps}
and inefficient for compute and memory resources because
every iterate needs to be stored,
as in \cref{sec:second-derivatives}.
This is why most methods using unrolled optimization for learning
often only unroll through \emph{tens} of iterations
\citep{metz2016unrolled,belanger2017end,foerster2017learning,finn2017model}
while solving the problems from scratch may
require 100k-1M+ iterations.
This causes the predictions to be extremely inaccurate solutions
to the optimization process and has sparked the research directions
that the next section turns to that seek to make unrolled optimization
more tractable.

\subsubsection{Truncated backpropagation through time and biased gradients}
\begin{center}
\begin{tikzpicture}
  \matrix (m) [
      matrix of math nodes,row sep=2em,column sep=1em,
      minimum width=2em,nodes={anchor=center}
  ] {
    \hat z^{0}_\theta & \hat z_\theta^{1} &
    \ldots & \hat z_\theta^{K-H} & \ldots &
    \hat z_\theta^{K} & \hat y_\theta(x) & \gL \\
    \; & \; & \; & \; & \; & \; \\};
  \node at (2.5,0.1) () {\color{lightpurple}\ldots};
  \path[-stealth]
    (m-1-1) edge node {} (m-1-2)
    (m-1-2) edge node {} (m-1-3)
    (m-1-3) edge node {} (m-1-4)
    (m-1-4) edge node {} (m-1-5)
    (m-1-5) edge node {} (m-1-6)
    (m-1-6) edge node {} (m-1-7)
    (m-1-7) edge node {} (m-1-8)
    (m-1-8) edge[draw=lightpurple,out=-155,in=-15] node[below] {} (m-1-7)
    (m-1-8) edge[draw=lightpurple,out=-155,in=-15] node {} (m-1-6)
    (m-1-8) edge[draw=lightpurple,out=-155,in=-15] node {} (m-1-4)
    ;
\end{tikzpicture}
\end{center}
\vspace{-10mm}

\emph{Truncated backpropagation through time (TBPTT)}
\citep{werbos1990backpropagation,jaeger2002tutorial}
is a crucial idea that has enabled the
training of sequence models over long sequences.
TBPTT's idea is that not every iteration needs to be
differentiated through and that the derivative
can be computed using smaller subsequences from
the full sequence of model predictions by
truncating the history of iterates.
For example, the derivative of a model running
for $K$ iterations with a truncation length of $H$
can be approximated by considering the
influence of the last $H$ iterates
$\left\{z_\theta^{i}\right\}_{i=K-H}^H$ on the loss $\gL$.

Truncation significantly helps improve the computational
and memory efficiency of unrolled optimization procedure
but results in harmful \emph{biased gradients} as
these approximate derivatives do not contain
the full amount of information that the model used
to compute the prediction.
This is especially damaging in approaches
such as MAML \citep{finn2017model} that \emph{only}
parameterize the first iterate and is why MAML-based
approaches often don't use TBPTT.
\citet{tallec2017unbiasing,wu2018understanding,liao2018reviving,shaban2019truncated,vicol2021unbiased}
seek to further theoretically understand the properties
of TBPTT, including the bias of the estimator
and how to unbias it.

\subsubsection{Other gradient estimators for sequential models}
In addition to truncating the iterations, other approaches
attempt to improve the efficiency of learning through
unrolled iterations with other approximations
that retain the influence of the entire sequence
of predictions on the loss
\citep{finn2017model,nichol2018first,lorraine2020optimizing}
which will be further discussed in
\cref{sec:second-derivatives}.
Some optimization procedures, such as gradient descent with momentum,
can also be ``reversed'' without needing to retain the
intermediate states \citep{maclaurin2015gradient,franceschi2017forward}.
\emph{Real-Time Recurrent Learning} (RTRL) by \citet{williams1989learning}
uses forward-mode automatic differentiation to compute unbiased
gradient estimates in an online fashion.
\emph{Unbiased Online Recurrent} (UORO) by \citet{tallec2017unbiased}
improves upon RTRL with a rank-1 approximation of the gradient
of the hidden state with respect to the parameters.
\citet{silver2022learning} considers the directional derivative
of a recurrent model along a candidate direction, which can
be efficiently computed to construct a descent direction.

\subsubsection{Semi-amortized learning with shrinkage and implicit differentiation}
\begin{figure}[t]
\centering
\includegraphics[width=2in]{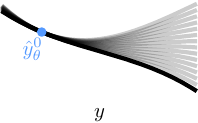}
\caption{Illustration of the penalty used in the Implicit MAML
  by \citet{rajeswaran2019meta} in \cref{eq:iMAML}.
  The original loss $f(y; x)$ is shown in black for
  a fixed context $x$ and the lighter grey colors
  show the impact of varying $\lambda$.
  This shows that the quadratic term of the penalization eventually
  overtakes the original loss and makes an optimum appear
  close to $\hat y_\theta^0$}
\label{fig:imaml}
\end{figure}

A huge issue arising in semi-amortized models is that adapting
to long time horizons is computationally and memory inefficient
and even if it wasn't, causes exploding, vanishing, or
otherwise unstable gradients.
An active direction of research seeks to solve these issues by
solving a smaller, local problem with the semi-amortized model,
such as in \citet{chen2019modular,rajeswaran2019meta}.
Implicit differentiation is an alternative to unrolling through
the iterations of a semi-amortized model in settings where the
model is able to successfully solve an optimization problem.

This section briefly summarizes \emph{Implicit MAML} (iMAML) by \citet{rajeswaran2019meta},
which notably brings this insight to MAML.
MAML methods usually only take a few gradient steps and
are usually not enough to globally solve \cref{eq:opt},
especially at the beginning of training.
\citet{rajeswaran2019meta} observe that adding a penalty
to the objective around the initial iterate
$\hat y^0_\theta$ makes it easy for the model to
\emph{globally} (!) solve the problem
\begin{equation}
  \label{eq:iMAML}
  \hat y_\theta(x) \in \argmin_y f(y; x) + \frac{\lambda}{2}\|y-\hat y^0_\theta\|_2^2,
\end{equation}
where the parameter $\lambda$ encourages the
solution to stay close to some initial iterate.
\Cref{fig:imaml} visualizes a function $f(y; x)$
in black and add penalties in grey with
$\lambda\in[0,12]$ and see that a global
minimum is difficult to find without adding a penalty
around the initial iterate.
This global solution can then be implicitly differentiated to obtain
a derivative of the loss with respect to the model's parameters
\emph{without} needing to unroll, as it requires
significantly less computational and memory resources.
\citet{huszar2019imaml} further analyzes and discuses iMAML.
They compare it to a Bayesian approach and observe that the insights
from iMAML can transfer from gradient-based meta-learning to
other amortized optimization settings.

\textbf{Warning.} Implicit differentiation is only useful
when optimization problems are exactly solved and satisfy
the conditions of the implicit function theorem in \cref{thm:dini}.
This is why \citet{rajeswaran2019meta} needed to add a penalty
to MAML's inner optimization problem in \cref{eq:iMAML} to
make the problem exactly solvable.
While they showed that this works and results in significant
improvements for differentiation, it comes at the expense of
changing the objective to penalize the distance from the
previous iterate.
In other words, iMAML modifies MAML's semi-amortized model
and in general is not helpful for estimating the derivative
through the original formulation of MAML.
Furthermore, computing the implicit derivative by
solving the linear system with the Jacobian
in \cref{eq:implicit-derivative} may be memory and
compute expensive to form and estimate exactly.
In practice, some methods such as \citet{bai2019deep}
successfully use indirect and approximation methods
to solve for the system in \cref{eq:implicit-derivative}.

\subsection{Learning with zeroth-order methods
  and RL}
\label{sec:learning:rl}

\begin{figure}[t]
\centering
\includegraphics[width=2in]{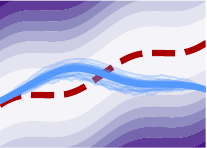}
\caption{Illustration of perturbing $\hat y_\theta$.
  A zeroth-order optimizers may make perturbations
  like this to search for an improved parameterization
}
\label{fig:perturbing}
\end{figure}

Computing the derivatives to learn $\hat y_\theta$ with a
first-order method may be impossible or unstable.
These problems typically arise when learning components
of the model that are truly non-differentiable, or
when attempting to unroll a semi-amortized model for
a lot of steps.
In these settings, research has successfully explored
other optimizers that do \emph{not} need the gradient information.
These methods often consider settings that improve
an objective-based loss with small local perturbations
rather than differentiation.
\Cref{fig:perturbing} illustrates that most of these
methods can be interpreted as locally perturbing the
model's prediction and updating the parameters to
move towards the best perturbations.

\subsubsection{Reinforcement learning}
\citet{li2016learning,li2017learning,ichnowski2021accelerating} view their
semi-amortized models as a Markov decision process (MDP)
that they solve with reinforcement learning.
The MDP interpretation uses the insight
that the iterations $x^i$ are the actions,
the context and previous iterations or losses are typically the states,
the associated losses $\gL(x^i)$ are the rewards,
and $\hat y_\theta^i(x)$ is a (deterministic) policy,
and transitions given by updating the current iterate,
either with a quantity defined by the policy or by running
one or more iterations from an existing optimizer.
Once this viewpoint is taken, then the optimal amortized model
can be found by using standard reinforcement learning methods,
\eg \citet{li2016learning,li2017learning} uses Guided Policy Search \citep{levine2013guided}
and \citet{ichnowski2021accelerating} uses TD3 \citep{fujimoto2018td3}.
The notation $\gL^{\rm RL}$ indicates that a loss is optimized with reinforcement learning,
typically on the objective-based loss.

\subsubsection{Loss smoothing and optimization with zeroth-order methods}
\label{sec:smooth}

Objective-based losses can have a high-frequency
structure with many poor local minimum.
\citet{metz2019understanding} overcome this by smoothing
the loss with a Gaussian over the \emph{parameter} space, \ie,
\begin{equation}
  \gL^{\rm smooth}(\hat y_\theta)\defeq \E_{\epsilon\sim\gN(0,\sigma^2 I)}\left[\gL(\hat y_{\theta+\epsilon})\right],
  \label{eq:smooth-loss}
\end{equation}
where $\sigma^2$ is a fixed variance.
\Cref{fig:smooth-loss} illustrates a loss function $\gL$ in
black and shows smoothed versions in color.
They consider learning the loss with reparameterization
gradients and zeroth-order evolutionary methods.
\citet{merchant2021learn2hop} further
build upon this for learned optimization in
atomic structural optimization
and study 1) clipping the values of the gradient estimator,
and 2) parameter optimization with genetic algorithms.

\begin{figure}[t]
\centering
\includegraphics[width=2.1in]{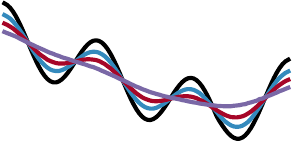}
\caption{Gaussian smoothing of a loss using \cref{eq:smooth-loss}.
  The colors show different values of the variance $\sigma^2$
  of the Gaussian. Selecting a high enough variance results in
  smoothing out most of the suboptimal minima.}
\label{fig:smooth-loss}
\end{figure}

\begin{remark}
  While smoothing can help reduce suboptimal local
  minima, it may also undesirably change the
  location of the global minimum.
  One potential solution to this is to decay
  the smoothing throughout training, as done
  in \citet[Appendix~A.1]{amos2021model}.
\end{remark}

\textbf{Connection to smoothing in reinforcement learning.}
The Gaussian smoothing of the objective $\gL$ in \cref{eq:smooth-loss}
is conceptually similar to Gaussian smoothing of the
objective in reinforcement learning, \ie the $-Q$-value,
by a Gaussian policy. This happens in
\cref{eq:Q-opt-sto-exp} and is further discussed
in \cref{sec:apps:ctrl}.
The policy's variance is typically controlled to match a
target entropy \citet{haarnoja2018soft} and the learning
typically starts with a high variance so the policy has a
broad view of the objective landscape and is then able to focus in
on a optimal region of the value distribution.
\citet{amos2021model} uses a fixed entropy decay schedule to
explicitly control this behavior.
In contrast, \citet{metz2019understanding,merchant2021learn2hop}
do not turn the loss into a distribution and more directly
smooth the loss with a Gaussian with a fixed variance $\sigma^2$
that is not optimized over.

\section{Extensions}
\label{sec:extensions}

I have intentionally scoped \cref{def:amor} to optimization problems
over \emph{deterministic, unconstrained, finite-dimensional, Euclidean}
domains $\gY$ where the context distribution $p(x)$
remains \emph{fixed} the
entire training time to provide a simple mental model that
allows us to focus on the core amortization principles
that consistently show up between applications.
This section cover extensions from this setting that may come up in practice.

\subsection{Extensions of the domain $\gY$}
\label{sec:extensions:domain}
\subsubsection{Deterministic $\rightarrow$ stochastic optimization}
\label{sec:extensions:sto}
A crucial extension is from optimization over deterministic vector
spaces $\gY$ to \emph{stochastic optimization}
where $\gY$ represents a space of distributions,
turning $y\in\gY$ from a vector in Euclidean space
into a distribution.
This comes up in \cref{sec:apps:ctrl} for control,
for example..

\textbf{Transforming parameterized stochastic problems
  back into deterministic ones.}
This portion will mostly focus on settings that
optimize over the parametric distributions.
This may arise in stochastic domains for variational inference in
\cref{sec:apps:avi} and stochastic control in \cref{sec:apps:ctrl}.
These settings optimize over a constrained parametric family
of distributions parameterized by some $\lambda$, for example
over a multivariate normal $\gN(\mu, \Sigma)$ parameterized
by $\lambda\defeq (\mu, \Sigma)$.
Here, problems can be transformed back to \cref{eq:opt} by
optimizing over the parameters with
\begin{equation}
  \lambda^\star(x) \in \argmin_{\lambda} f(\lambda; x),
  \label{eq:normal-opt}
\end{equation}
where $\lambda$ induces a distribution that the
objective $f$ may use.
When $\lambda$ is not an unconstrained real space, the
differentiable projections discussed in \cref{sec:constraints}
could be used to transform $\lambda$ back into this form for amortization.

\begin{figure}[t]
\centering
\includegraphics[width=2.3in]{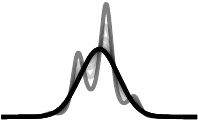}
\caption{
  The Gaussian distribution can be characterized as the result of
  the optimization problem in \cref{eq:maxent}:
  constrained to the space of continuous distributions with
  a given mean and variance, the Gaussian distribution has
  the maximum entropy in comparison to every other distribution.
  This example parameterizes a non-Gaussian density (shown in grey)
  and optimizes over it using gradient steps of \cref{eq:maxent},
  eventually converging to a Gaussian.
  An animated version is available
  \href{https://github.com/facebookresearch/amortized-optimization-tutorial/blob/main/paper/fig/maxent.gif}{in the repository associated with this tutorial}.
  While the Gaussian is the known closed-form solution to
  this optimization problem and analytically known,
  more general optimization problems over densities
  without known solutions can also be amortized.
}
\end{figure}
\textbf{Optimizing over distributions and densities.}
More general stochastic optimization settings involve optimizing over
spaces representing distributions, such as the functional space
of all continuous densities.
Many standard probability distributions can be obtained and
characterized as the solution to a maximum entropy
optimization problem and is explored, \eg, in
\citet[Ch.~12]{cover2006elements},
\citet[p.~47]{guiasu1985principle}, and
\citet[\S6.2]{pennec2006intrinsic}.
For example, a Gaussian distribution $\gN(\mu, \Sigma)$
solves the following constrained maximum entropy
optimization problem over the space of continuous densities $\gP$:
\begin{equation}
  p^\star(\mu, \Sigma)\in \argmax_{p\in\gP} \HH_p[X]\; \subjectto\; \E_p[X] = \mu\; {\rm and}\; \Var_p[X] = \Sigma,
  \label{eq:maxent}
\end{equation}
where $\HH_p[X]\defeq -\int p(x)\log p(x){\rm d}x$ is the \emph{differential entropy}
and the constraints are on the mean $\E_p[X]$ and covariance $\Var_p[X]$.
\citet[Theorem~8.6.5 and Example~12.2.8]{cover2006elements} prove that the closed-form solution
of $p^\star$ is the Gaussian density.
This Gaussian setting therefore does not need amortization as the
closed-form solution is known and easily computed, but
more general optimization problems over densities do not necessarily
have closed-form solutions and could benefit from amortization.
While this tutorial does not study amortizing these problems, in some cases it may
be possible to again transform them back into deterministic optimization problems
over Euclidean space for amortization by approximating the density $g_\theta$
with an expressive family of densities parameterized by $\theta$.

\subsubsection{Unconstrained $\rightarrow$ constrained optimization}
\label{sec:constraints}
Amortized constrained optimization problems may naturally arise, for example
in the convex optimization settings in \cref{sec:apps:convex}
and for optimization over the sphere in \cref{sec:impl:sphere}.
Constrained optimization problems for amortization can often be represented as
an extension of \cref{eq:opt} with
\begin{equation}
  y^\star(x) \in \argmin_{y\in\gC} f(y; x),
  \label{eq:opt-constrained}
\end{equation}
where the constraints $\gC$ may also depend on the context $x$.
\Cref{rmk:optimality-loss} suggests one way of amortizing
\cref{eq:opt-constrained} by amortizing the objective-based loss
associated with the optimality conditions of the constrained problem.
A budding research area studies how to more generally include
constraints into the formulation.
\citet{baker2019learning,dong2020smart,zamzam2020learning,pan2020deepopf,klamkin2025dualinteriorpointoptimization,vanhentenryck2025optimizationlearning}
predict solutions to optimal power flow problems.
\citet{misra2021learning} learn active sets for constrained optimization.
\citet{krivachy2020fast} solves constrained feasibility semi-definite programs
with a fully-amortized neural network model using an
objective-based loss.
\citet{donti2021dc3} learns a fully-amortized model and optimizes an
objective-based loss with additional completion and correction terms
to ensure the prediction satisfies the constraints of the original problem.

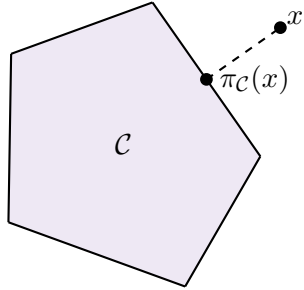
\begin{figure}[t]
  \centering
  \begin{tikzpicture}[every path/.style={thick}]
    \path[fill=lightpurple!20] (0,0) coordinate(p1) -- ++(20:2) coordinate(p2) -- ++(-55:2.5)
    coordinate(p3) -- ++(-120:2.) coordinate(p4) --  ++(160:2.5) coordinate(p5);
    \node at (barycentric cs:p1=1,p2=1,p3=1,p4=1,p5=1) {$\gC$};
    \foreach \X [count=\Y] in {2,...,6} {
      \ifnum\X=6
      \path (p\Y) -- (p1) coordinate[pos=0.](a\Y) coordinate[pos=1.](a1)
      coordinate[pos=0.5](m1);
      \draw (a\Y) -- (a1);
      \else
      \path (p\Y) -- (p\X) coordinate[pos=0.](a\Y) coordinate[pos=1.](a\X)
      coordinate[pos=0.5](m\X);
      \draw (a\Y) -- (a\X);
      \fi}
    \draw[dashed] (m3) -- ($ (m3)!1.2cm!90:(p3) $) node[pos=1.2]{$x$}; %
    \filldraw[black] (m3) circle(2pt);
    \filldraw[black] ($ (m3)!1.2cm!90:(p3) $) circle(2pt);
    \node[xshift=1pt,anchor=west] at (m3) {$\pi_\gC(x)$};
  \end{tikzpicture}
  \label{fig:projection}
  \caption{Illustration of \cref{def:proj} showing
    a Euclidean projection $\pi_\gC(x)$ of a point $x$
    onto a set $\gC$.}
\end{figure}

\textbf{Differentiable projections.}
When the constraints are relatively simple, a differentiable projection
can transform a constrained optimization problem into an unconstrained one,
\eg, in reinforcement learning constrained action spaces can be transformed
from the box $[-1,1]^n$ to the reals $\R^n$ by using
the $\tanh$ to project from $\R^n$ to $[-1,1]^n$.
\Cref{sec:impl:sphere} also uses a differentiable projection from $\R^n$
onto the sphere $\gS^{n-1}$.
These are illustrated in \cref{fig:projection} and defined as:
\begin{definition}
  A \emph{projection} from $\R^n$ onto a set $\gC\subseteq \R^n$ is
  \begin{equation}
    \pi_\gC: \R^n\rightarrow \gC \qquad \pi_\gC(x) \in \argmin_{y\in\gC} D(x, y) + \Omega(y),
    \label{eq:proj}
  \end{equation}
  where $D: \R^n\times\R^n\rightarrow \R$ is a distance and $\Omega:\R^n\rightarrow \R$ is
  a regularizer that can ensure invertibility or help spread $\R^n$ more uniformly throughout $\gC$.
  A \emph{(sub)differentiable projection} has (sub)derivatives $\nabla_x \pi_\gC(x)$.
  I sometimes omit the dependence of $\pi$ on the choice of $D$, $\Omega$, and $\gC$
  when they are given by the surrounding context.
  \label{def:proj}
\end{definition}

\textbf{Lack of idempotency.} In linear algebra, a projection is defined to
be \emph{idempotent}, \ie applying the projection twice gives the same result
so that $\pi\circ \pi=\pi$.
Unfortunately, projections as defined in \cref{def:proj},
such as Bregman projections, are \emph{not} idempotent in general
and often $\pi_\gC\circ \pi_\gC\neq \pi_\gC$
as the regularizer $\Omega$ may cause points that are already on $\gC$
to move to a different position on $\gC$.

\textbf{Differentiable projections for constrained amortization.}
These can be used to cast \Cref{eq:opt-constrained} as the unconstrained
problem \cref{eq:opt} by composing the objective with a projection
$f\circ \pi_\gC$.
(Sub)differentiable projections enable gradient-based learning through the projection
and is the most easily attainable when the projection has an explicit closed-form solution.
For intuition, the ReLU, sigmoid, and softargmax can be interpreted as
differentiable projections that solve convex optimization problems
in the form of \cref{eq:proj}.
\citet[\S2.4.4]{amos2019differentiable} further discusses these
and proves them using the KKT conditions:
\begin{itemize}
\item The standard Euclidean projection onto the
  \emph{non-negative orthant} $\R^n_+$ is defined by
  \begin{equation}
    \pi(x) \in \argmin_y \;\; \frac{1}{2}\|x-y\|_2^2 \;\; \st \;\; y\geq 0,
    \label{eq:relu-proj}
  \end{equation}
  and has a closed-form solution given by the
  ReLU, \ie $\pi(x) \defeq \max\{0, x\}$.
\item The interior of the \emph{unit hypercube} $[0,1]^n$ can
  be projected onto with the entropy-regularized
  optimization problem
  \begin{equation}
    \pi(x) \in \argmin_{0<y<1} \;\; -x^\top y -H_b(y),
    \label{eq:sigmoid-proj}
  \end{equation}
  where
  \begin{equation}
  H_b(y) \defeq \left(\sum_i y_i\log y_i + (1-y_i)\log (1-y_i)\right)
  \end{equation}
  is the
  binary entropy function.
  \Cref{eq:sigmoid-proj} has a closed-form solution given by
  the \emph{sigmoid} or \emph{logistic} function,
  \ie $\pi(x) \defeq (1+e^{-x})^{-1}$.
\item The interior of the $(n-1)$-\emph{simplex} defined by
  \begin{equation}
    \Delta_{n-1}\defeq\{p\in\R^n\; \vert\; 1^\top p = 1 \; \; {\rm and} \;\; p \geq 0 \}
    \label{eq:simplex}
  \end{equation}
  can be projected onto with the entropy-regularized
  optimization problem
  \begin{equation}
    \pi(x) \in \argmin_{0<y<1} \;\; -x^\top y - H(y) \;\; \st\;\; 1^\top y = 1
    \label{eq:simplex-proj}
  \end{equation}
  where $H(y) \defeq -\sum_i y_i \log y_i$ is the entropy function.
  \Cref{eq:simplex-proj} has a closed-form solution given by
  the \emph{softargmax}, \ie $\pi(x)_j = e^{x_j} / \sum_i e^{x_i}$,
  which is historically referred to as the \emph{softmax}.
\end{itemize}

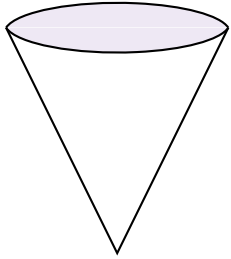
\begin{figure}[t]
  \centering
  \begin{tikzpicture}[every path/.style={thick}]
    \path[fill=lightpurple!20]
      (0,0) arc (-170:-10:1.5cm and 0.4cm)coordinate[pos=0]
      -- (0,0) arc (170:10:1.5cm and 0.4cm)coordinate;
    \draw (0,0) arc (-170:-10:1.5cm and 0.4cm)coordinate[pos=0] (a);
    \draw (0,0) arc (170:10:1.5cm and 0.4cm)coordinate (b);
    \draw (a) -- ([yshift=-3cm]$(a)!0.5!(b)$) -- (b);
  \end{tikzpicture}
  \label{fig:lorentz-cone}
  \caption{Illustration of the second-order cone in \cref{eq:lorentz-cone}.}
\end{figure}

This section goes beyond these to differentiable projections onto
\emph{convex cones}. These can also be softened or regularized
to help with continuity when composed with learning and
amortization methods.
\citet{ali2017semismooth,busseti2019solution} discuss
differentiating the standard Euclidean projections
onto these, including:
\begin{itemize}
\item
  The \emph{second-order, Lorentz, or ice cream cone}
  defined by
  \begin{equation}
    \gK_{\rm soc}\defeq\{(x,y)\in\R^{m-1}\times\R : \|x\|_2\leq y\},
    \label{eq:lorentz-cone}
  \end{equation}
  which is illustrated in \cref{fig:lorentz-cone}.
  The standard Euclidean projection is given in closed form as
  \begin{equation}
    \pi(x, y) \defeq
    \begin{cases}
      0 & \|x\|_2 \leq -y \\
      (x,y) & \|x\|_2 \leq y \\
      \frac{1}{2} (1 + \frac{y}{\|x\|_2}) (x, \|x\|_2) & \textrm{otherwise}.
    \end{cases}
  \end{equation}
  and can be explicitly differentiated.
\item The \emph{positive semidefinite cone} $\gS^m_+$ of the
  space of $m\times m$ positive semidefinite matrices.
  The Euclidean projection is obtained in closed-form
  by projecting the eigenvalues to be non-negative with
  $\pi(X)\defeq\sum_i\max\{\lambda_i, 0\}q_iq_i^\top$,
  where the eigenvalue decomposition of $X$ is given by
  $X=\sum_i\lambda_iq_iq_i^\top$.
  The derivative can be computed by differentiating
  through the eigenvalue decomposition and projection
  of the eigenvalues.
\item The \emph{exponential cone} is given by
  \begin{equation}
  \begin{aligned}
    \gK_{\rm exp} \defeq & \{(x,y,z) : x\in\R, y>0, z\geq y\exp(x/y)\} \\
    & \cup \{(x,0,z): x\leq0, z\geq 0\}.
  \end{aligned}
  \label{eq:exp-cone}
  \end{equation}
  The standard Euclidean projection onto this does
  \emph{not} have a known closed-form solution
  but can be computed using a Newton method
  as discussed in
  \citet[\S6.3.4]{parikh2014proximal}.
  \citet{ali2017semismooth} differentiate through
  this projection using implicit differentiation
  of the KKT system.
\end{itemize}

\noindent Other uses of projections in machine learning include:
\begin{itemize}
\item \citet{adams2011ranking,santa2017deeppermnet,mena2018learning}
  project onto the \emph{Birkhoff polytope} of
  \emph{doubly stochastic} matrices with row and
  column sums of 1, \ie
  \begin{equation}
    \label{eq:birkhoff}
    \gB_m\defeq \{X\in\R^{m\times m}: X1=X^\top1=1\}.
  \end{equation}
\item \citet{amos2019limited} project onto the capped
  simplex for a differentiable top-$k$ layer.
\item \citet{blondel2019structured} perform structure
  prediction and learning methods building on Fenchel-Young losses
  \citep{blondel2020learning} and use projections onto the
  simplex, unit cube, knapsack polytope, Birkhoff polytope,
  permutahedron, and order simplex.
\end{itemize}

In many of these, the projections have explicit closed-form
solutions that make it easy to compute and differentiate
through the projections for learning.
When a closed-form solution to the projection is not available to
the project, but the projection can be numerically computed,
projections can often still be differentiated through using
implicit differentiation.

\subsubsection{Euclidean $\rightarrow$ non-Euclidean optimization}
\emph{Manifold optimization} \citep{absil2009optimization,hu2019brief}
over non-Euclidean spaces is a thriving topic in optimization
as these problems arise frequently over complex geometries in nature.
One form of manifold optimization takes $\gY$ to be a Riemannian
manifold rather than a real-valued spaced.
This area of research has studied acceleration methods,
\citep{duruisseaux2022accelerated}, but less exploration
has been done on amortized optimization.
\Cref{sec:impl:sphere} discusses amortizing a simple
constrained spherical optimization problem that can be
transformed into an unconstrained Euclidean optimization
problem by using projections from ambient Euclidean space.
When this is not possible, a budding area of research investigates
more directly including the manifold structure into the
amortization process. \citet{gao2020learning} amortize optimization
problems over SPD spaces.

\subsection{Extensions of the model $\hat y_\theta$}
Finding the best model for an amortized optimization setup
is an active research topic in many areas.
While the tutorial is mostly scoped to differentiable
parametric models that are end-to-end learned,
variations and extensions can be considered.

\subsubsection{Symbolic models: uncovering human-interpretable update rules}
A huge issue of neural-network based amortization models
is that they are uninterpretable and it is often impossible
for us as humans to learn any new insights about the optimization
problems being modeled, \eg how to better-solve them.
Symbolic models are one potential answer to this that attempt
to search over a symbolic space that is much closer to the
operations that humans would use to implement update rules for
an optimization solver.
Early studies of these methods include
\citet{bengio1994use,runarsson2000evolution}.
\citet{bello2017neural} significantly advances this direction
by posing the learned optimizer as a reinforcement learning problem
where the actions produce the operations underlying the update rules.
They show how existing methods can be symbolically implemented
using this formulation, and learn better update rules for
classification and machine translation tasks.
Symbolic methods are further studied and scaled in
\citet{real2020automl,zheng2022symbolic}.
\citet{maheswaranathan2021reverse} reverse engineer learned
optimizers and show that they have learned interpretable behavior,
including momentum, gradient clipping, learning rate schedules,
and learning rate adaptation.

This direction of work challenges the best accelerated and adaptive
gradient-based optimizers that are used for machine learning.
Nesterov acceleration \citep{nesterov1983method} has a provably
optimal convergence rate among first-order methods for solving
convex optimization problems, but unfortunately breaks down
in the non-convex setting.
This has led to a stream of variations of acceleration methods
for non-convex problems that come up in machine learning,
such as \citet{duchi2011adaptive,zeiler2012adadelta,kingma2014adam},
that typically add components that adapt the update rules to
how much the objective is moving in each dimension.
None of these algorithms are theoretically or provably the
best in non-convex settings, and is often empirically validated
depending on the domain.
Using amortized optimization with a symbolic model to search
the design space of optimizers can result is significantly
better optimizers and insights into the optimization problems
being solved, especially when this is done on new classes
of problems beyond the parameter learning problems typically
considered in machine learning settings.

\subsection{Uncertainty-aware and Bayesian optimization}
An active research direction combines \emph{uncertainty} estimation
and amortized optimization:

\textbf{Amortization for Bayesian optimization.}
\citet{chen2017learning} propose to use an RNN-based
  amortization in Bayesian optimization settings that
  predict the optimal solution to commonly used
  acquisition functions such as the expected improvement
  and observed improvement.
  This is powerful as optimizing the acquisition function
  is often a computational bottleneck.
\citet{swersky2020amortized} consider amortized
  Bayesian optimization in discrete spaces and show
  applications to protein sequence design.
\citet{ravi2018amortized} propose amortized
  Bayesian meta-learning for meta-learning with uncertainty
  estimation over the posterior and show applications to
  contextual bandits and few-shot learning.

\textbf{Bayesian methods for amortization.}
\citet{you2022bayesian} investigate \emph{optimizer uncertainty}
or \emph{Bayesian learning to optimize}.
This setting explores the uncertainty that an optimizer,
\eg the amortization model, is the best optimizer for
the problem.

\subsection{Settings with additional learnable contexts $\varphi$}
The amortization model is often a component within a larger system
with other learnable parameters that are being optimized over.
This is done in, for example, 1) variational autoencoders
where the ELBO also depends on the decoder's parameters
that are also being optimized over,
2) deep equilibrium models where the fixed point
is parameterized and optimized over, and
3) reinforcement learning where the value estimate is
also parameterized and learned over.

These dependencies can be captured by writing an explicit
dependence of the context distribution and objective
on an additional parameter $\varphi$, \ie as $p(x; \varphi)$
and $f(y; x, \varphi)$.
$\varphi$ can be learned with a higher-level optimization
process with a loss $\ell$ defined on the \emph{solutions}.
This could take the form of the bi-level problem
\begin{equation}
  \argmin_{\varphi} \E_{x\sim p(x)} \ell(y^\star(x, \varphi); x, \varphi)\;
  \subjectto\; y^\star(x, \varphi)\in\argmin_y f(y; x, \varphi)
\label{eq:learn-context}
\end{equation}
where $y^\star(x, \varphi)$ can be replaced with an approximation
by a learned amortization model $\hat y_\theta \approx y^\star$.
The parameters $\varphi$ in \cref{eq:learn-context} can
often by end-to-end learned \emph{through the solution} of
\cref{eq:opt} to update the influence that $\varphi$
has on the solutions.
The next section turns to methods that show how to differentiate
through the value $f(y^\star(x, \varphi); x, \varphi)$ and solution
$y^\star(x, \varphi)$ to enable gradient-based learning of
$\varphi$ in \cref{eq:learn-context}.

\subsubsection{Learning $\varphi$ by differentiating
  the objective value $f(y^\star(x, \varphi); x, \varphi)$}
Methods can end-to-end learn through the \emph{optimal objective value}
$f(y^\star(x, \varphi); x, \varphi)$ to update parameters $\varphi$
that show up in the context --- \ie by taking $\ell=f$ in \cref{eq:learn-context}.
For example, variational autoencoders differentiate through
the objective value, \ie the best approximation to the ELBO,
to optimize the data log-likelihood of a parameterized
decoder $\log p(x\mid z; \varphi)$.
The theory around this is rooted in the optimization
community's studies of \emph{envelope theorems},
which describe settings where the minimum value
can be differentiated by just differentiating
the objective.
\emph{Danskin's envelope theorem} \citep{danskin1966theory}
in convex settings is one of the earliest and has been
extended into more general settings, \eg, in
\citep[Prop. A.22]{bertsekas1971control}
and \cite{carter2001foundations,milgrom2002envelope,bonnans2013perturbation}.
In the unconstrained and non-convex \cref{eq:opt},
the envelope theorem gives
\begin{equation}
  \nabla_\varphi \min_y f(y; x, \bar\varphi) = \nabla_\varphi f(y^\star(x, \bar\varphi); x, \bar\varphi)
  \label{eq:envelope}
\end{equation}
at a point $\bar\varphi$ under mild assumptions, showing
that differentiating through the $\min$ operation is equivalent to differentiating
through just the objective at the optimal solution $y^\star(x, \varphi)$.

\subsubsection{Learning $\varphi$ by differentiating the solution $y^\star(x, \varphi)$}
In addition to differentiating through the objective value, the solution
$y^\star(x, \varphi)$ can be implicitly differentiated.
The derivative $\D_\varphi y^\star(x, \varphi)$ is referred to
as the \emph{adjoint derivative}, and it is often used for
end-to-end learning
\citep{domke2012generic,gould2016differentiating,amos2017optnet,barratt2018differentiability,amos2019differentiable,agrawal2019differentiable,bai2019deep,bai2020multiscale}
and perturbation and sensitivity analysis
\citep{bank1982non,fiacco1990sensitivity,shapiro2003sensitivity,klatte2006nonsmooth,bonnans2013perturbation,still2018lectures,fiacco2020mathematical}.

Computing the adjoint derivative $\D_\varphi y^\star(x, \varphi)$ is more involved
than the value derivative using the envelope theorem
in \cref{eq:envelope} as more components of $y^\star(x)$
can change as $x$ moves infinitesimally.
An explicit closed-form solution to $y^\star(x)$ is
not available in most cases, which means that standard
methods for explicitly computing the derivative through
this computation may not work well or may break down.
For example, an optimizer to compute $y^\star(x)$ may be
explicitly unrolled through, but this may be
unstable and extremely memory- and compute-intensive
to track all of the iterations.
The adjoint derivative is typically computed with implicit
differentiation by seeing $y^\star(x)$ as an \emph{implicit}
function of $x$.
This uses the implicit function theorem,
which is originally from \citet{dini1878analisi},
and is presented in \citet[Theorem 1B.1]{dontchev2009implicit} as:
\begin{theorem}[Dini's implicit function theorem]
  \label{thm:dini}
  Let the roots of $g(y; \varphi)$ define an implicit
  mapping $Y^\star(\varphi)$ given by $Y^\star(\varphi)\defeq\{y \mid g(y;\varphi)=0\}$,
  where $\varphi\in\R^m$, $y\in\R^n$, and
  $g: \R^n\times\R^m\rightarrow\R^n$.
  Let $g$ be continuously differentiable in a neighborhood of $(\bar y, \bar \varphi)$
  such that $g(\bar y; \bar \varphi)=0$, and let the Jacobian of $g$
  with respect to $y$ at $(\bar y, \bar \varphi)$,
  \ie $\D_y g(\bar y; \bar \varphi)$, be non-singular.
  Then $Y^\star$ has a single-valued localization $y^\star$
  around $\bar \varphi$ for $\bar y$ which is continuously differentiable
  in a neighborhood $Q$ of $\bar \varphi$ with Jacobian satisfying
  \begin{equation}
    \D_\varphi y^\star(\tilde \varphi) = -\D_y^{-1} g(y^\star(\tilde \varphi); \tilde \varphi) \D_\varphi g(y^\star(\tilde \varphi); \tilde \varphi)
    \qquad \mathrm{for\ every}\; \tilde \varphi\in Q.
    \label{eq:implicit-derivative}
  \end{equation}
\end{theorem}

The adjoint derivative $D_\varphi y^\star(\varphi)$ can be computed
by seeing $y^\star$ as the root of an implicit
function $g(y;x,\varphi)$, which needs to be selected to
make the solution equivalent to the solution of \cref{eq:opt}.
Typically $g(y;x,\varphi)$ is an optimality system of
the optimization problem, \eg the KKT system
for constrained convex optimization problems.
For the unconstrained problem here,
the first-order optimality of
the objective $g(y;x, \varphi) \defeq \nabla_y f(y; x, \varphi)$
can be used with \cref{thm:dini} to compute
$\D_\varphi y^\star(x, \varphi)$.

\chapter{Applications of amortized optimization}
\label{sec:apps}

\begin{table}[t]
  \caption{Applications of amortized optimization covered in \cref{sec:apps}}
  \vspace{2mm}
\resizebox{\textwidth}{!}{
\begin{tabular}{ccccccc}
\S & Application & Objective $f$ & Domain $\gY$ & Context Space $\gX$ & Amortization model $\hat y_\theta$ & Loss $\gL$ \\ \toprule
\ref{sec:apps:avi} & VAE & $-\ELBO$ & variational posterior & data & full & $\gL_{\rm obj}$ \\
& SAVAE/IVAE & | & | & | & semi & | \\
\midrule
\ref{sec:apps:lista} & PSD & reconstruction & sparse code & data & full & $\gL_{\rm reg}$ \\
& LISTA & | & | & | & semi & | \\
\midrule
\ref{sec:apps:meta} & HyperNets & task loss & model parameters & tasks & full & $\gL_{\rm obj}$ \\
& LM & | & | & | & semi & $\gL^{\rm RL}_{\rm obj}$ \\
& MAML & | & | & | & | & $\gL_{\rm obj}$ \\
& Neural Potts & pseudo-likelihood & | & protein sequences & full & $\gL_{\rm obj}$ \\
\midrule
\ref{sec:apps:convex} & NeuralFP & FP residual & FP iterates & FP contexts & semi & $\gL_{\rm obj}^\Sigma$ \\
& HyperAA & | & | & | & | & $\gL_{\rm reg}^\Sigma$ \\
& NeuralSCS & CP residual & CP iterates & CP parameters & | & $\gL_{\rm obj}^\Sigma$ \\
& HyperDEQ & DEQ residual & DEQ iterates & DEQ parameters & | & $\gL_{\rm reg}^\Sigma$ \\
& NeuralNMF & NMF residual & factorizations & input matrices & | & $\gL_{\rm obj}^\Sigma$ \\
& RLQP & $R_{\rm RLQP}$ & QP iterates & QP parameters & | & $\gL^{\rm RL}_{\rm obj}$ \\
\midrule
  \ref{sec:apps:ot} & Meta OT & dual OT cost & optimal couplings & input measures & full & $\gL_{\rm obj}$ \\
  & CondOT & dual OT cost & optimal couplings & contextual information & | & $\gL_{\rm obj}$ \\
   & AmorConj & $c$-transform obj & ${\rm supp}(\alpha)$ & ${\rm supp}(\beta)$ & | & $\gL_{\rm obj}$ \\
  & $\gA$-SW & max-sliced dist & slices $\Theta$ & mini-batches & | & $\gL_{\rm obj}$ \\
\midrule
\ref{sec:apps:ctrl} & BC/IL & $-Q$-value & controls & state space & full & $\gL_{\rm reg}$ \\
& (D)DPG/TD3 & | & | & | & | & $\gL_{\rm obj}$ \\
& PILCO & | & | & | & | & $\gL_{\rm obj}$ \\
& POPLIN & | & | & | & full or semi & $\gL_{\rm reg}$ \\
& DCEM & | & | & | & semi & $\gL_{\rm reg}$ \\
& IAPO & | & | & | & | & $\gL_{\rm obj}$ \\
& SVG & $\D_\gQ$ or $-\gE_Q$ & control dists & | & full & $\gL_{\rm obj}$ \\
& SAC & | & | & | & | & $\gL_{\rm obj}$ \\
& GPS & | & | & | & | & $\gL_{\rm KL}$ \\
\bottomrule
\end{tabular}}

  \label{tab:rw}
\end{table}

This section takes a review and tour of many key applications
of amortized optimization to show some unifying ideas
that can be shared between all of these topics.
\Cref{tab:rw} summarizes the methods.
The subsections in here are meant to be standalone
and can be randomly accessed and read in any order.
I scope closely to providing the relevant context for
just the amortized optimization components and
under-emphasize the remaining context of each research area.

\textbf{Warning.}
Even though I try to provide the relevant background and notation to
present the amortized optimization components, each section is
meant to be a review rather than an introduction to
these research topics.
I defer further background to the original literature.

\section{Variational inference and
  variational autoencoders}
\label{sec:apps:avi}
Key ideas in amortized optimization originated in the
variational inference (VI) community's interest in
approximating intractable densities and integrals
via optimization.
This section focuses only on the relevant components of amortized
variational inference (AVI) used in machine learning
for the variational autoencoder (VAE) and
related generative models
\citep{kingma2013auto,rezende2014stochastic,mnih2014neural,rezende2015variational,higgins2016beta,doersch2016tutorial,kingma2019introduction}
and refer to references such as
\citet{jordan1999introduction,wainwright2008graphical,blei2017variational}
for a complete background in variational inference.
\citet{kim2020deep,marino2021learned} provide additional background
on the use of amortization and semi-amortization in these settings.
Historically, the use of an encoder network for amortized inference
is often traced back to the Helmholtz machine \citep{dayan1995helmholtz},
which uses a fully-amortized model but without a proper gradient estimator.
\citet{sjolund2023tutorial,zammit2024neural} provide further background
information and tutorials on parametric, variational,
and amortized inference.

\subsection{The variational autoencoder (VAE) by \citet{kingma2013auto}}
\label{sec:apps:vae}
A VAE models a density $p(x)$ over a high-dimensional space,
for example images, text, or videos, given samples $x\sim p(x)$.
They introduce a lower-dimensional latent space
with a known distribution $p(z)$, such as an isotropic Gaussian,
designed to capture hidden structure present
in $p(x)$.
VAEs parameterize a likelihood model $p(x; \varphi)$ with $\varphi$.
Optimizing the log-likelihood
$\log p(x; \varphi)=\log\int_z p(x\mid z; \varphi)p(z)dz$
with this latent structure is typically intractable because
of the integral over the latent space.
Variational methods overcome this by introducing
a tractable lower-bound called the
\emph{evidence lower bound} ($\ELBO$) defined by
\begin{equation}
  \log p(x; \varphi)\geq \ELBO_\varphi(\lambda; x) \defeq \E_{q(z; \lambda)}[\log p(x\mid z; \varphi)] - \kl{q(z;\lambda)}{p(z)},
  \label{eq:elbo}
\end{equation}
where $q(z; \lambda)$ is a variational distribution
over the latent space parameterized by $\lambda$
and $p(z)$ is the prior.
Given a sample $x\sim p(x)$ and fixed encoder's parameters $\varphi$
the \emph{best}
lower bound $\lambda^\star$ satisfying
\begin{equation}
  \log p(x)\geq \ELBO_\varphi(\lambda^\star; x) \geq \ELBO_\varphi(\lambda; x)
  \label{eq:best-elbo}
\end{equation}
for all $\lambda$ can be obtained by solving the
optimization problem
\begin{equation}
  \lambda^\star_\varphi(x) \in \argmax_\lambda\; \ELBO(\lambda; x, \varphi).
  \label{eq:elbo-opt}
\end{equation}
Gaussians are a common choice of the variational distribution
$q(z; \lambda)$ is in \citet{kingma2013auto},
but may cause a loose inequality in \cref{eq:best-elbo}.
\citet{rezende2015variational,cremer2018inference} explore
more expressive distributions to help make
$\ELBO(\lambda^\star; x, \varphi)$ equal to $\log p(x)$.

Amortized VI (AVI) methods predict the solution to
\cref{eq:elbo-opt} while stochastic variational
methods \citep{hoffman2013stochastic} explicitly solve it.
AVI methods learn a model $\hat \lambda_\theta: \gX\rightarrow \Lambda$
with parameters $\theta$, which is usually a feedforward neural network,
to predict the maximum value of the $\ELBO$ by optimizing
the objective-based loss
\begin{equation}
  \argmax_\theta \E_{x\sim p(x)} \ELBO_\varphi(\hat \lambda_\theta(x); x)
  \label{eq:vae-amor}
\end{equation}
where the expectation is usually approximated with a
Monte Carlo estimate from the samples.

\textbf{Summary.} This standard AVI formulation is therefore an amortized optimization method
$\gA_{\rm VAE}\defeq (-\ELBO, \Lambda, \gX, p(x), \hat \lambda_\theta, \gL_{\rm obj})$
over the (negated) $\ELBO$ where the domain of the optimization
problem is the variational parameter space $\Lambda$,
the context space $\gX$ is the sample space for the
generative model,
the samples are given from $p(x)$,
$\hat\lambda_\theta: \gX\rightarrow\Lambda$ is the
\emph{fully-amortized} model optimized with the
gradient-based loss $\gL_{\rm obj}$ over $-\ELBO$.

\textbf{Extensions.}
Analyzing and extending the amortization components
has been a key development in AVI methods.
\citet{cremer2018inference} investigate suboptimality in
these models and categorize it as coming from an
\emph{amortization gap} where the amortized model for
\cref{eq:vae-amor} does not properly solve it,
or the \emph{approximation gap} where the variational
posterior is incapable of approximating the true distribution.
Semi-amortization plays a crucial role in addressing the
amortization gap and is explored in the
semi-amortized VAE (SAVAE) by \citet{kim2018semi}
and iterative VAE (IVAE) by \citet{marino2018iterative}.
AVI methods are also used in
hierarchical \citep{sonderby2016ladder}
and sequential settings \citep{chung2015recurrent}.

\textbf{The full VAE loss.}
This section has left the parameterization
$\varphi$ of the model $p(x; \varphi)$
fixed to allow us to scope into the
amortized optimization component in isolation.
For completeness, the final step necessary to train
a VAE is to optimize the $\ELBO$ over the training
data of \emph{both} $p(x; \varphi)$ along with
the $\hat \lambda_\theta(x)$ with
\begin{equation}
  \argmax_{\varphi,\theta} \E_{x\sim p(x)} \ELBO_\varphi(\hat \lambda_\theta(x); x).
  \label{eq:vae-full}
\end{equation}

\section{Sparse coding}
\label{sec:apps:lista}
Another early appearance of amortized optimization has been in
sparse coding \citep{kavukcuoglu2010fast,gregor2010learning}.
The connection to the broader amortized optimization and
learning to optimize work has also been made by,
\eg, \citet{chen2021learning}.
\emph{Sparse coding methods} seek to reconstruct an input
from a sparse linear combination of bases
\citep{olshausen1996emergence,chen2001atomic,donoho2003optimally}.
Given a \emph{dictionary} $W_d$ of the basis vectors
and an input $x\in\gX$, the \emph{sparse code} $y^\star$ is
typically recovered by solving the optimization problem
\begin{equation}
  y^\star(x) \in \argmin_y E(y; x) \qquad
  E(y; x)\defeq \frac{1}{2}\|x-W_dy\|_2^2 + \alpha\|y\|_1,
  \label{eq:sparse-coding}
\end{equation}
where $E$ is the regularized reconstruction energy
and $\alpha\in\R_+$ is a coefficient of the $\ell_1$ term.
\Cref{eq:sparse-coding} is traditionally solved
with the Iterative Shrinkage and Thresholding Algorithm (ISTA)
such as in \citet{daubechies2004iterative}.
Fast ISTA (FISTA) by \citet{beck2009fast} improves ISTA even more
by adding a momentum term.
The core update of ISTA methods is
\begin{equation}
  y^{t+1}\defeq h_\beta\left(W_ex + Sy^t\right)
  \label{eq:ista}
\end{equation}
$W_e\defeq (1/L)W_d^\top$ is the \emph{filter} matrix,
$L$ is an \emph{upper bound on the largest eigenvalue} of $W_d^T W_d$,
$S\defeq I-W_eW_d$ is the \emph{mutual inhibition} matrix,
and
$h_\beta(v)\defeq\sign(v)\max\left\{0, |v|-\beta\right\}$
is the \emph{shrinkage function} with threshold $\beta$, usually
set to $\alpha/L$.
ISTA methods are remarkably fast ways of solving
\cref{eq:sparse-coding} and
the machine learning community has explored the use
of learning to make ISTA methods even faster
that can be seen as instances of amortized optimization.

\subsection{Predictive Sparse Decomposition (PSD) by \citet{kavukcuoglu2010fast}}
PSD predicts the best sparse code using fully-amortized models
of the form
\begin{equation}
  \hat y_\theta(x) \defeq D\tanh(Fx),
  \label{eq:fast-ista}
\end{equation}
where the parameters are $\theta=\{D,F\}$.
Then, given a distribution over vectors $p(x)$,
PSD regresses the prediction onto the true sparse code $y^\star(x)$
by solving
\begin{equation}
  \argmin_\theta \E_{x\sim p(x)} \|\hat y_\theta(x) - y^\star(x)\|_2^2,
  \label{eq:psd-loss}
\end{equation}
where instead of solving for $y^\star(x)$ directly
with (F)ISTA, they also iteratively approximate it while
iteratively learning the model.

\textbf{Summary.}
$\gA_{\rm PSD}\defeq (E, \gY, \gX, p(x), \hat y_\theta, \gL_{\rm reg})$

\subsection{Learned ISTA (LISTA) by \citet{gregor2010learning}}
LISTA further explores the idea of predicting solutions
to sparse coding problems by proposing a semi-amortized model
that integrates the iterative updates of ISTA into the model.
LISTA leverages the soft-thresholding operator $h$ and
considers a semi-amortized model over the domain $\gY$
that starts with $x^0_\theta\defeq 0$
and iteratively updates $x^{t+1}_\theta\defeq h_{\beta}(Fx+Gx_\theta^t)$.
Running these updates for $K$ steps results in a
final prediction $\hat y(x)\defeq x^K_\theta$ parameterized
by $\theta=\{F, G, \beta\}$ that is also optimized with
the regression-based loss to the ground-truth
sparse codes as in \cref{eq:psd-loss}.

\textbf{Summary.}
$\gA_{\rm LISTA}\defeq (E, \gY, \gX, p(x), \hat y_\theta, \gL_{\rm reg})$

\section{Multi-task learning and meta-learning}
\label{sec:apps:meta}

Many multi-task learning and meta-learning methods also
use amortized optimization for parameter learning.
This section takes a glimpse at this viewpoint, which has
also been observed before in \citet{shu2017amortized,gordon2018meta}.

\textbf{Background.}
\emph{Multi-task learning} \citep{caruana1997multitask,ruder2017overview}
methods use shared representations and models to learn
multiple tasks simultaneously.
\emph{Meta-learning} methods \citep{ward1937reminiscence,harlow1949formation,schmidhuber1987evolutionary,kehoe1988layered,schmidhuber1995learning,thrun1998learning,baxter1998theoretical,hochreiter2001learning,vilalta2002perspective,lv2017learning,li2016learning,li2017learning,lake2017building,weng2018metalearning,hospedales2020meta}
seek to learn how to learn and are often used
in multi-task settings.
Multi-task and meta-learning settings typically define
\emph{learning tasks} $\gT\sim p(\gT)$ that each consist of
a classification or regression task.
The tasks could be different hyper-parameters of a model,
different datasets that the model is trained on, or
in some cases different samples from the same dataset.
Each task has an associated
\emph{task loss} $\gL_{\gT}(\hat y_\theta)$
that measures how well a parameterized model $\hat y_\theta$
performs on it.
There is typically a distribution over tasks $p(\gT)$ and
the goal is to find a model that
best-optimizes the expectation over task losses by solving
\begin{equation}
  \argmin_\theta \E_{\gT\sim p(\gT)} \gL_{\gT}(\hat y_\theta). \label{eq:multi-task}
\end{equation}
The motivation here is that there is likely shared structure
and information present between the tasks that learning
methods can leverage.
The next section goes through methods that solve \cref{eq:multi-task}
using objective-based amortized optimization methods.

\subsection{Fully-amortized hyper networks (HyperNets)}
HyperNEAT \citep{stanley2009hypercube} and
Hypernetworks \citep{ha2016hypernetworks}
predict the optimal parameters to a network given a
data sample and can be seen as fully-amortized optimization.
The tasks here $\gT=(x,y^\star(x))$ usually consist of
a sample from some data distribution
$x\sim p(x)$ along with a target $y^\star(x)$
for classification or regression,
inducing a task distribution $p(\gT)$.
HyperNets propose to predict $y^\star(x)$ with a \emph{prediction model}
$\hat y_\varphi(x)$ parameterized by $\varphi$.
Instead of learning this model directly, they propose to
use an \emph{amortization model} $\hat \varphi_\theta(x)\in\Phi$
to predict the parameters to the model
$\hat y_{\hat \varphi_\theta(x)}(x)\eqdef \hat y_\theta(x)$
that best-optimize the \emph{task loss} $\gL_\gT(\hat y_\theta(x), y^\star(x))$
for each data point.
The amortization model is usually a
black-box neural network that is fully-amortized and
predicts the parameters from only the task's data
without accessing the task loss.
The models are learned with an objective-based loss
\begin{equation}
  \argmin_\theta \E_{\gT\sim p(\gT)} \gL_\gT(\hat y_\theta(x), y^\star(x)).
  \label{eq:hypernet}
\end{equation}

\textbf{Summary.}
$\gA_{\rm HyperNet}\defeq(\gL_\gT, \Phi, \gT, p(\gT), \hat \varphi_\theta, \gL_{\rm obj})$

\subsection{Learning to optimize (LM) by \citet{li2016learning}}
\citet{li2016learning} consider three multi-task settings for
logistic regression, robust linear regression, and neural network classification
where the different tasks are different datasets the models are
trained on.
Given a dataset $\gT=\{x_i,y_i\}_{i=1}^N$ to train on,
they again search for the parameters
$\hat \varphi_\theta(\gT)\in\Phi$
of another prediction model $\hat y_{\hat \varphi_\theta(\gT)}(x)\eqdef \hat y_\theta(x)$
that performs well on a loss
$\gL_\gT(\hat y_\theta)$
that measures how well the model fits to the dataset.
In contrast to HyperNets, LM consider each task to be an
entire dataset rather than just a single data point,
and LM considers semi-amortized models that are able
to iteratively refine the prediction.
They use a semi-amortized model that starts with
an initial iterate $\hat \varphi^0_\theta(\gT)$ and then
predicts the next iterate with
\begin{equation}
  \hat \varphi^{t+1}_\theta = g_\theta(\{\varphi^i, \gL_\gT(\hat \varphi^i),
  \nabla_\varphi \gL_\gT(\hat \varphi^i), \Delta^i\}),
  \label{eq:LM-model}
\end{equation}
where the update model $g_\theta$ takes the
last $i\in\{t-H,\ldots,t\}\cap \gZ_{\geq 0}$ iterates as the input,
along with the objective, gradient, and objective improvement
as auxiliary information.
This model generalizes methods such as gradient descent
that would just use the previous iterate and gradient.
The experiments use $H=25$ and typically run
the model updates for 100 iterations.
They want to learn the model with an objective-based
loss here and take the viewpoint that it can be seen as
an MDP that can be solved with the guided policy search
\citep{levine2013guided} method for reinforcement learning.
\citet{li2017learning} further develops these ideas for
learning to optimize neural network parameters.

\textbf{Summary.}
$\gA_{\rm LM}\defeq(\gL_\gT, \Phi, \gT, p(\gT), \hat \varphi_\theta, \gL_{\rm obj}^{\rm RL})$

\subsection{Model-agnostic meta-learning (MAML) by \citet{finn2017model}}
As discussed in \cref{sec:semi-domain},
MAML can be seen as a semi-amortized optimization method.
They also seek to predict the parameters
$\hat \varphi_\theta(\gT)\in\Phi$
of prediction model $\hat y_{\hat \varphi_\theta(\gT)}(x)\eqdef \hat y_\theta(x)$
in a multi-task setting with tasks $\gT\sim p(\gT)$.
They propose to only learn to predict
an initial iterate $\hat \varphi^0_\theta(\gT)=\theta$ and then
update the next iterates with gradient-based updates such as
\begin{equation}
  \hat \varphi^{t+1}_\theta = \varphi^t_\theta - \alpha\nabla_\varphi \gL_\gT(\hat \varphi^t_\theta),
  \label{eq:MAML-model}
\end{equation}
where $\gL_\gT(\varphi)$ is the task loss obtained by
the model $\hat y_\varphi$ parameterized by $\varphi$.
MAML optimizes this model with an objective-based
loss through the final prediction.

\textbf{Summary.}
$\gA_{\rm MAML}\defeq(\gL_\gT, \Phi, \gT, p(\gT), \hat \varphi_\theta, \gL_{\rm obj})$

\subsection{Protein MSA modeling with the Neural Potts Model}
\label{sec:apps:potts}
\citet{sercu2021neural} proposes a fully-amortized solution to
fit a Potts model to a protein's multiple sequence alignment (MSA).
Each task consists of a finite MSA $\gM\defeq\{x_i\}$
and they use a fully-amortized model $\varphi_\theta(\gM)\in\Phi$
to predict the optimal parameters of a Potts model
$p(\gM; \varphi)$ fit to the data.
The model $\varphi_\theta$ is a large attention-based sequence
model that takes the MSA as the input.
Learning is done with the objective-based loss
\begin{equation}
  \argmin_\theta \E_{\gM\sim p(\gM)} \gL_{\rm PL}(\varphi_\theta(\gM))
  \label{eq:neural-potts-loss}
\end{equation}
to optimize the \emph{pseudolikelihood} $\gL_{\rm PL}$ of
the Potts model.

\citet{sercu2021neural} surprisingly observes that amortization results
in \emph{better} solutions than the classical method
for the Potts model parameter optimization with a finite MSA.
They refer to this as the \emph{inductive gain} and attribute
it to the fact that they only have a finite MSA from each
protein and thus amortizing effectively shares information
between proteins

\textbf{Summary.} $\gA_{\rm NeuralPotts}=(\gL_{\rm PL}, \Phi, \gM, p(\gM), \hat \varphi_\theta, \gL_{\rm obj})$

\subsection{Other relevant multi-task and meta-learning work}
The literature of multi-task learning and meta-learning
methods is immense and often build on the preceding concepts.
The following selectively summarizes a few other relevant ideas:

\begin{enumerate}
\item \citet{ravi2016optimization} also propose optimization-based
  semi-amortized models that use a recurrent neural network
  to predict parameter updates for meta-learning in multi-task
  learning settings.
\item Latent embedding optimization (LEO) by \citet{rusu2018meta}
  and fast context adaptation (CAVIA) by \citet{zintgraf2019fast}
  perform semi-amortization over a learned latent space.
  This uses the powerful insight that semi-amortizing over
  the low-level parameters $\varphi$ had a lot of redundancies
  and may not be able to easily capture task-specific
  variations that can be learned in a latent space.
\item \citet{andrychowicz2016learning}
  consider semi-amortized models based on recurrent neural networks
  and show applications to amortizing quadratic optimization,
  neural networks for MNIST and CIFAR-10 classification,
  and neural style transfer.
\item \citet{chen2017learning} consider RNN-based semi-amortized models
  in settings where the gradient of the objective is not
  used as part of the model and show applications in Bayesian
  optimization, Gaussian process bandits, control, global optimization,
  and hyper-parameter tuning.
\item \citet{wichrowska2017learned} continue studying
  RNN-based semi-amortized models for classification.
  They scale to Inception V3 \citep{szegedy2016rethinking}
  and ResNet V2 \citep{he2016identity} architectures
  and scale to classification tasks on ImageNet
  \citep{russakovsky2015imagenet}, presenting many
  insightful ablations along the way.
\item \citet{franceschi2018bilevel} further analyze the
  bilevel optimization aspects of gradient-based meta-learning
  and present new theoretical convergence results and
  empirical demonstrations.
\item MetaOptNet \citep{lee2019meta} and R2D2 \citep{bertinetto2018meta}
  consider semi-amortized models based on differentiable
  optimization and propose to use differentiable SVMs
  and ridge regression as part of the amortization model.
\item Almost No Inner Loop by \citet{raghu2019rapid}
  study what parameters should be adapted within the amortization
  model and demonstrate settings where adapting only
  the final layer performs well, indicating that the shared model
  between tasks works well because it is learning
  shared features for all the tasks to use.
\item \citet{wang2021bridging} further connect gradient-based meta
  learning methods to multi-task learning methods.
\item HyperTransformer \citep{zhmoginov2022hypertransformer}
  study amortized models based on transformers
  \citep{vaswani2017attention}
  and show applications to few-shot classification.
\item \citet{metz2021gradients} study and emphasize the difficulty
  of optimizing objective-based loss with just gradient
  information due to natural chaotic-based failure models
  of the amortization model.
  They focus on iterated dynamical systems and study where
  chaotic losses arise in physics and meta-learning.
  They identify the spectrum of the Jacobian as one
  source of these issues and give suggestions for
  remedying these undesirable behaviors to have learning
  systems that are well-behaved and stable.
\item \citet{metz2019using} learn optimizers for robust
  classification tasks. They find that optimizers can
  uncover ways of quickly finding robust parameterizations
  that generalize to settings beyond the corruptions
  used during training.
\item \citet{metz2019understanding} study semi-amortized
  optimization of convolutional architectures and identify
  and focus on key issues of
  1) biased gradients from truncated BPTT and 2) exploding gradient
  norms from unrolling for many timesteps.
  They overcome both of these issues by optimizing the
  smoothed loss in \cref{eq:smooth-loss}
  with a variant of the gradient estimator proposed in
  \citet{parmas2018pipps} for reinforcement learning.
  This estimator re-weights reparameterization gradients and
  likelihood ratio gradients using inverse variance
  weighting \citep{fleiss1993review}.
  \citet{parmas2021unified} further unify the likelihood ratio
  and reparameterization gradients by connecting them
  with the divergence theorem which enables them to
  create a generalized estimator combining them.
\item \citet{merchant2021learn2hop} further build on
  the advancements of \citet{metz2019understanding} for
  semi-amortized atomic structural optimization, which
  is a setting rife with poor local minima.
  Their models learn to ``hop'' out of these minima
  and are able to generalize more efficiently to
  new elements or atomic compositions.
\item \citet{zhang2018graph,knyazev2021parameter} explore
  the fully-amortized HyperNets for architecture search
  for predicting parameters on CIFAR-10 and ImageNet.
  These models take a model's compute graph as the
  context and use a graph neural network to predict
  the optimal parameters of that architecture on a task.
\item \citet{huang2022optimizer} show how to use
  information from existing classes of ``teacher''
  optimizers to learn a new ``student'' one that
  can result in even better performance,
  which they refer to as \emph{optimizer amalgamation}.
  This is done by optimizing for the objective-based
  loss with additional regression-based terms that
  encourage the learned optimizer to match one or
  more trajectories of the existing optimizers.
\item \citet{harrison2022closer} look at the stability
  of learned optimization methods from a dynamical
  systems perspective and propose a number of modifications
  to improve the stability and generalization.
\item \citet{metz2022velo} continue scaling a semi-amortized
  learned optimizer that predicts parameter updates for
  training machine learning models with millions of parameters.
  They train the optimizer for four thousand ATP-months
  and show that it outperforms many standard parameter optimization
  methods on standard learning tasks.
  One standout feature is that their amortization model does
  not assume a fixed-size context or prediction space
  but instead is able to predict updates for models with varying
  numbers of parameters.
  The key insight to supporting a variable number of parameters
  is to decompose the amortization model across parameter groups
  using an LSTM and hyper-network.
\item MetaOptimize \citep{sharifnassab2024metaoptimize} predicts
  the solutions to hyper-parameter optimization problems for machine learning,
  showing results on image classification and language models.
\end{enumerate}

\section{Fixed-point computations and convex optimization}
\label{sec:apps:convex}

\begin{definition}
  A \emph{fixed point} $y^\star\in\R^n$ of a map
  $g:\R^n\rightarrow\R^n$ is where $g(y^\star) = y^\star$.
  \label{def:fixed}
\end{definition}

Continuous fixed-point problems as in \cref{def:fixed}
and illustrated in
\cref{fig:fp}
are ubiquitous in engineering and science and
amortizing their solutions is an activate research area.
Let $\gR(y; x)\defeq g(y; x)-y$ be the \emph{fixed-point residual}
with squared norm $\gN(y; x)\defeq \|\gR(y; x)\|_2^2$.
Fixed-point computations are connected to continuous
unconstrained optimization as any fixed-point problem can be
transformed into \cref{eq:opt} by optimizing the
residual norms with:
\begin{equation}
  y^\star(x)\in\argmin_y \gN(y; x),
  \label{eq:fixed-point-opt}
\end{equation}
and conversely \cref{eq:opt} can be transformed into
a fixed-point problem via first-order optimality
to find the fixed-point of
$\nabla f(y; x) - y = 0$.
Thus methods that to amortize the solutions to \cref{eq:opt}
can help amortize the solutions
to fixed-point computations \cref{def:fixed}.

\begin{figure}[t]
\vspace{-3mm}
\centering
\includegraphics[width=1.8in]{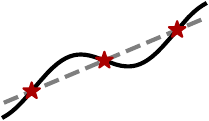}
\caption{Illustration of the fixed points of a map $f(x)$.
  The map is shown in black and the fixed points (red stars)
  are where the map is equal to the identity (shown in grey),
  \ie $f(x)=x$.
}
\label{fig:fp}
\end{figure}

\textbf{Solving and accelerating fixed-point computations.}
Fixed points can be found with \emph{fixed-point iterations}
$y^{t+1}\defeq f(y^t)$ or by using
\emph{Newton's root-finding method} on the
fixed-point residuals with
\begin{equation}
  y^{t+1}\defeq y^t - \left(\D_y g(y^t)\right)^{-1} g(y^t).
  \label{eq:newton}
\end{equation}
These methods can also be \emph{accelerated} by using
a sequence of past iterates instead of just the most
recent iterate.
\emph{Anderson acceleration} methods
\citep{anderson1965iterative,walker2011anderson,zhang2020globally}
are standard and generate updates that combine the
previous $M+1$ iterates $\{y^i\}_{i=t-M}^t$  with an update of the form
\begin{equation}
  \label{eq:aa-update}
  {\rm AA\_Update}^t(\{y_i\}, \alpha, \beta) \defeq
  \beta \sum_{i=0}^M \alpha_i g(y^{t-M+i}) +
  (1-\beta)\sum_{i=0}^M\sum_{i=0}^M \alpha_i y^{t-M+i},
\end{equation}
where $\beta\in[0,1]$ is a coefficient that controls
whether the iterates or application of $g$ on the iterates
should be used, and $\alpha\in\R^{M+1}$ where $1^\top\alpha=1$
are the coefficients used to combine the previous iterates.
A basic AA method sets $\beta=1$ and solves
\begin{equation}
  \label{eq:aa-alpha}
  \alpha^\star\defeq \argmin_\alpha \|\gR(y^i)\alpha\|_2\; \subjectto\; 1^\top \alpha = 1
\end{equation}
for $i\in\{t-M,t\}$ with least squares.
Other methods such as
\emph{Broyden's method} \citep{broyden1965class} can
also accelerate fixed-point computations by turning them
into root-finding problems.

\subsection{Neural fixed-point acceleration (NeuralFP) and conic optimization with the splitting cone solver (NeuralSCS)}
\label{sec:apps:neural-fp}
\emph{Neural fixed-point acceleration} \citep{venkataraman2021neural}
proposes a semi-amortized method for computing fixed-points
and use it for convex cone programming.
Representing a latent state at time $t$ with $\hat h^t$,
they parameterize the initial iterate
$\left(\hat y^0, \hat h^0\right) = {\rm init}_\theta(x)$
with an \emph{initialization model} ${\rm init_\theta}$
and perform the fixed-point computations
\begin{equation}
  \begin{aligned}
    \tilde x^{t+1} &= f(\hat y^t; x) \\
    \left(\hat y^{t+1},\hat h^{h+1}\right) &= {\rm acc}_\theta(\hat y^t, \tilde x^{t+1}, \hat h^t)
  \end{aligned}
  \label{eq:neural-fp}
\end{equation}
using an \emph{acceleration model} ${\rm acc_\theta}$
that is typically a recurrent model that predicts the
next iterate given the past sequence of iterates.
\citet[Prop. 1]{venkataraman2021neural} discuss how
this model captures standard AA as an instance
by setting the models equal to the standard update
that doesn't use learning.
They learn this model to amortize \cref{eq:fixed-point-opt}
over a distribution of contexts $p(x)$ with an objective-based loss
that solves
\begin{equation}
  \argmin_\theta \E_{x\sim p(x)} \sum_{t=0}^K \gN(\hat y^t_\theta(x)),
  \label{eq:neural-fp-loss}
\end{equation}
where the fixed-point residual norm $\gN$ is scaled by a
context-specific \emph{normalization factor}.

\emph{NeuralSCS} \citep{venkataraman2021neural} applies this
neural fixed-point acceleration to solve
\emph{constrained convex cone programs}
solved by the splitting cone solver (SCS)
\citep{o2016conic}
of the form
\begin{equation}
  \label{eq:cone_problem}
\begin{aligned}[c]
  & \minimize\; c^{T}x \\
  & \st \; Ax + s = b \\
  & \qquad (x,s) \in \mathbb{R}^n \times \mathcal{K}
\end{aligned}
\qquad
\begin{aligned}[c]
  & \maximize\; -b^{T}y \\
  & \st\; -A^{T}y + r = c \\
  & \qquad (r,y) \in \{0\}^n \times \mathcal{K}^*
\end{aligned}
\end{equation}
where $x \in \mathbb{R}^n$ is the primal variable, $s \in
\mathbb{R}^m$ is the primal slack variable, $y \in \mathbb{R}^m$ is
the dual variable, and $r \in \mathbb{R}^n$ is the dual residual. The
set $\mathcal{K} \in \mathbb{R}^m$ is a non-empty convex cone with
dual cone $\mathcal{K}^* \in \mathbb{R}^m$.
where $x \in \mathbb{R}^n$ is the primal variable, $s \in
\mathbb{R}^m$ is the primal slack variable, $y \in \mathbb{R}^m$ is
the dual variable, and $r \in \mathbb{R}^n$ is the dual residual. The
set $\mathcal{K} \in \mathbb{R}^m$ is a non-empty convex cone.
SCS uses the \emph{homogeneous self-dual embedding} to view
\cref{eq:cone_problem} as a fixed-point computation
over $\gZ=\R^{n\times m\times 1}$ with a scalar-valued scaling factor
as the last component.

NeuralSCS proposes a semi-amortized model to predict
the solution to the fixed point of the self-dual embedding
that solves \cref{eq:cone_problem}.
Their semi-amortized model $\hat z_\theta(\phi)$ takes
a context $\phi$ as the input and outputs a solution to the
self-dual embedding by composing the SCS iterations $f$ with
the learned fixed-point acceleration modules
(${\rm init}_\theta,{\rm acc}_\theta$).

\textbf{Summary.}
$\gA_{\rm NeuralSCS}\defeq\left(\gN, \gZ, \phi, p(\phi), \hat z_\theta, \gL_{\rm obj}^\Sigma\right)$

\subsection{Neural acceleration for matrix factorization (NeuralNMF)}
\citet{sjolund2022graphbased} use semi-amortized neural acceleration
to find low-rank factorizations of an input matrix $V$ of the form:
\begin{equation}
  V\approx WH^\top, \qquad W\geq 0, H\geq 0,
  \label{eq:nmf}
\end{equation}
where the basis matrix $W\in\R^{m\times r}$ and
mixture matrix $H\in\R^{n\times r}$ are elementwise
non-negative matrices of rank $r\leq\min(m,n)$.
Let $Z=(W,H)$. Taking the norm of the residual of \cref{eq:nmf}
leads to the optimization formulation
\begin{equation}
  Z^\star(V)\in \argmin_{W,H\geq 0} \gN_{\rm NMF}(W, H; V) \qquad
  \gN_{\rm NMF}(W, H; V)\defeq \frac{1}{2}\|WH^\top - V\|_F^2,
  \label{eq:nmf-opt}
\end{equation}
which can be solved with ADMM \citep{boyd2011distributed}
using alternating steps on $H$ and $V$ as done in
\citet{huang2016flexible}.
Given a distribution over input matrices $V$,
\citet{sjolund2022graphbased}, augment the ADMM approach
from \citet{huang2016flexible} with transformer-based
initialization and acceleration modules.
This semi-amortized model is learned with an
objective-based loss and unrolls through
the ADMM iterations for learning.

\textbf{Summary.}
$\gA_{\rm NeuralNMF}\defeq\left(\gN_{\rm NMF}, \gZ, V, p(V), \hat Z_\theta, \gL_{\rm obj}^\Sigma\right)$

\subsection{HyperAnderson acceleration and
  deep equilibrium models (HyperDEQ)}
\label{sec:apps:deq}

\citet{bai2022neural} similarly proposes a semi-amortized
method for computing fixed-points and use it to
improve \emph{Deep equilibrium (DEQ) models}
\citep{bai2019deep,bai2020multiscale,gurumurthy2021joint}.
Their learned variant of AA, called \emph{HyperAnderson acceleration}
uses models that predict the initial point $\hat y^0_\theta(x)$
and coefficients $\alpha_\theta(x; G)$ and $\beta_\theta(x; G)$
and result in iterations of the form
\begin{equation}
  G^{t+1}_\theta, \hat y^{t+1}_\theta \defeq
  {\rm AA\_Update}(\{\hat y^t_\theta\},
    \alpha^t_\theta(x; G^t_\theta), \beta^t_\theta(x, G^t_\theta)),
  \label{eq:hyperaa}
\end{equation}
where $G^t\defeq \gR(x^t)$ is the fixed-point residual at
iteration $t$ and the model's final prediction is
$\hat y_\theta(x)\defeq x^K$.
Learning is performed by optimizing a summed regression-based loss
that encourages the fixed-point iterations to converge
as fast as possible by optimizing
\begin{equation}
  \argmin_\theta \gL_{\rm HyperAA}(\hat y_\theta) \qquad \gL_{\rm HyperAA}(\hat y_\theta) \defeq \E_{x\sim p(x)} \sum_{t=0}^K w_t \|y^\star - \hat y_\theta^t\|_2^2 + \Omega(\alpha^t),
  \label{eq:hyperaa-loss}
\end{equation}
where $\Omega$ is a regularizer on $\alpha^t$ that is annealed
to equal zero by the end of training and the
weights $(w_t)$ are set to be monotonically increasing to
penalize the later iterations for not reaching the fixed point.

\emph{Deep equilibrium (DEQ) models} \citep{bai2019deep,bai2020multiscale}
investigate implicit layers that parameterize and solve
fixed-point computations and have been a flagship for
``infinite depth'' vision and language models.
Given an input $x\in\gX$, such as an image or language sequence
to process, a DEQ model finds a fixed point $y^\star(x)$
of $g_\varphi(y; x)$ to make a prediction for a task,
such as regression or classification.
This fixed-point problem can again be interpreted as
finding the minimum norm of the residual
$\gN(y; x)\defeq ||y-g_\varphi(y;x)||_2^2$
as
\begin{equation}
  y^\star(x)\in \argmin_x \gN(y; x).
  \label{eq:deq-opt}
\end{equation}
\citet{bai2022neural} propose to use the HyperAnderson Acceleration
model and loss to semi-amortize DEQs by learning the initial
iterate and AA update coefficients, resulting in a setup of the form
$\gA_{\rm HyperDEQ}\defeq(\gN, \gY, \gX, p(x), \hat y_\theta, \gL_{\rm HyperAA})$.

\subsection{Comparing NeuralFP and HyperAA}
Neural fixed-point acceleration (NeuralFP) by \citet{venkataraman2021neural}
and
HyperAnderson Acceleration (HyperAA) by \citet{bai2022neural}
can both generally be applied to semi-amortize
fixed-point computations by parameterizing updates
based on Anderson Acceleration, even though they are
instantiated and evaluated on very different classes of
fixed-point computations.
Here is a brief comparison of the methods:
\vspace{8mm}

\noindent\begin{minipage}[t]{0.5\textwidth}
\noindent\textbf{Neural fixed-point acceleration}
\begin{itemize}[leftmargin=*,noitemsep]
\item Learn the initial iterate
\item Learn the entire update
\item Use an objective-based loss
\end{itemize}
\end{minipage}\hfill\begin{minipage}[t]{0.5\textwidth}
\textbf{HyperAnderson Acceleration}
\begin{itemize}[leftmargin=*,noitemsep]
\item Learn the initial iterate
\item Learn $\alpha,\beta$ for the update
\item Use an regression-based loss
\end{itemize}
\end{minipage}

\subsection{RLQP by \citet{ichnowski2021accelerating}}
\label{sec:apps:qprl}
RLQP \citep{ichnowski2021accelerating} amortizes solutions to constrained
convex quadratic optimization problems of the form
\begin{equation}
  x^\star(\phi)\in \argmin_x \frac{1}{2} x^\top Px+ q^\top x\; \subjectto\; l\leq Ax\leq u,
  \label{eq:qp}
\end{equation}
where $x\in\R^n$ is the \emph{domain} of the optimization problem
and $\phi=\{P,q,l,A,u\}$ is the \emph{context} or \emph{parameterization}
(from a larger space $\phi\in\Phi$)
of the optimization problem with $P\succ 0$ (symmetric positive semi-definite).
They build on the OSQP solver \citep{stellato2018osqp} for these
optimization problems, which is based on operator splitting.
Without over-relaxation, the core of OSQP uses updates
that first solve the system
\begin{equation}
  \begin{bmatrix}
    P+\sigma I & A^\top \\
    A & -{\rm diag}(\rho^t)^{-1} \\
  \end{bmatrix}
  \begin{bmatrix}
    x^{t+1} \\
    v^{t+1}
  \end{bmatrix}
  =
  \begin{bmatrix}
    \sigma x^t - q \\
    z^t - {\rm diag}(\rho^t)^{-1} y^t
  \end{bmatrix}
  \label{eq:osqp-system}
\end{equation}
and then updates
\begin{equation}
  \begin{aligned}
  \tilde z^{t+1} &\defeq z^t+{\rm diag}(\rho^t)^{-1}(v^{t+1}-y^t) \\
  z^{t+1} &\defeq \Pi\left(\tilde z^{t+1}+{\rm diag}(\rho^t)^{-1}y^t\right) \\
  y^{t+1} &\defeq x^t + {\rm diag}(\rho)\left(\tilde z^{t+1}-z{t+1}\right),
  \end{aligned}
  \label{eq:osqp-update}
\end{equation}
where $y,v$ are dual variables,
$z,\tilde z$ are auxiliary operator splitting variables,
$\sigma$ is a regularization parameter, and
$\rho^t\in\R^m_+$ is a step-size parameter.
Combining all of the variables into a state
$s\defeq(y, \lambda, \tilde z, z)$ living in
$s\in\gS$, the update can be written as
$s^{t+1}\defeq{\rm OSQP\_UPDATE}(s^t, \rho^t)$.

RLQP proposes to use these OSQP iterates as a
semi-amortized model with the iterates
$\{s^t, \rho^t\}$. They propose to only parameterize and
learn to predict the step size
$\rho^{t+1}\defeq \pi_\theta(s^t)$, with a
neural network amortization model $\pi_\theta$.
They model the process of predicting the optimal $\rho$ as
an MDP and define a reward $R_{\rm RLQP}(s, \rho)$ that
is $-1$ if the QP is not solved
(based on thresholds of the primal and dual residuals)
and $0$ otherwise, \ie each episode rolls
out the OSQP iterations with a policy predicting
the optimal step size.
They solve this MDP with TD3 by \citet{fujimoto2018td3}
to find the parameters $\theta$.

\textbf{Summary.}
$\gA_{\rm RLQP}\defeq (R_{\rm RLQP}, \gS\times\R^m_+, \Phi, p(\phi), \pi_\theta, \gL^{\rm RL}_{\rm obj})$

\section{Optimal transport}
\label{sec:apps:ot}

\textbf{Preliminaries.}
Optimal transport methods seek to optimally move mass between measures.
Standard references and introductions include
\citet{villani2009optimal,santambrogio2015optimal,peyre2019computational},
and this section concisely reviews key concepts.
Given two measures $(\alpha, \beta)$ supported on spaces $(\gX, \gY)$,
the \emph{Kantorovich problem} (\eg in \citet[Remark~2.13]{peyre2019computational})
solves
\begin{equation}
  \pi^\star(\alpha, \beta, c)\in\arginf_{\pi\in\gU(\alpha, \beta)} \int_{\gX\times \gY} c(x,y){\rm d}\pi(x,y),
  \label{eq:kantorovich}
\end{equation}
where the \emph{coupling} $\pi$ is joint distributions over the product of the measures,
$\gU$ is the set of admissible couplings,
$c$ is a cost.
The \emph{dual} of \cref{eq:kantorovich},
\eg in \citet[Eq.~2.31]{peyre2019computational},
can be represented by
\begin{equation}
  f^\star(\alpha, \beta, c)\in\argsup_{f} J(f)
  \label{eq:kantorovich-dual}
\end{equation}
where the \emph{dual objective} is defined by
\begin{equation}
  J(f)\defeq\int_\gX f(x){\rm d}\alpha(x) + \int_\gY f^c(y) {\rm d}\beta(y)
  \label{eq:kantorovich-dual-obj}
\end{equation}
over \emph{continuous dual potential functions} $f: \gX \rightarrow \R$ where
\begin{equation}
  f^c(y)\defeq\inf_x c(x, y)-f(x)
  \label{eq:c-transform}
\end{equation}
is the \emph{$c$-transform} operation.
In Euclidean spaces with the negative inner product cost \citep[eq.~5.12]{villani2009optimal},
$f$ is a convex function and the $c$-transform operation $f^c$
in \cref{eq:c-transform} is the standard Legendre-Fenchel transform,
also known as the convex conjugate.
When the measures $\alpha,\beta$ are discrete, a coupling $\pi$
in the primal (\cref{eq:kantorovich}) can be represented as a matrix,
and the potential $f$ in the dual (\cref{eq:kantorovich-dual})
can be represented as a finite-dimensional vector and solved with
a linear programming formulation or Sinkhorn iterations
\citep{cuturi2013sinkhorn} in the entropic setting.

The following sections overview methods that amortize these optimization
problems arising for optimal transport.
\Cref{sec:apps:meta-ot} discusses methods that amortize multiple OT
problems and map from the measures and cost to the optimal coupling in
\cref{eq:kantorovich} or duals in \cref{eq:kantorovich-dual}.
\Cref{sec:apps:amor-conj} discusses methods that amortize
the $c$-transform operation in \cref{eq:c-transform} that arises
as a repeatedly-solved subproblem when solving a single OT problem.
\Cref{sec:apps:amor-sliced} overviews amortizing a slicing optimization
problem that arises when projecting measures down to a single dimension
for more computationally efficient solves.

\begin{remark}
\textbf{The Wasserstein GAN (WGAN) by \citet{arjovsky2017wasserstein} is not amortized optimization.}
While the continuous Wasserstein-1 potentials are estimated using
samples from the generated and real data distributions,
this is not performing amortized optimization because these potentials
only optimize a single optimization transport problem between the generated
and real data distributions.
Changing the generated distribution indeed changes the optimal transport
problem, but it's not important to solve the optimal transport
problem between older generated distributions.
The WGAN potentials during training can better be interpreted
as warm-starting new optimal transport problems given by the generator's distribution.
\end{remark}

\subsection{Amortizing solutions to optimal transport problems}
\label{sec:apps:meta-ot}

Many computational OT methods focus on computing the solution mapping from
the measures and cost to the optimal coupling (\cref{eq:kantorovich})
or duals (\cref{eq:kantorovich-dual}). When repeatedly solving
optimal transport problems, this solution mapping is an optimization
problem that can be amortized.
These methods for predicting the optimal duals of OT problems
are also related to \citet{dinitz2021faster,khodak2022learning},
which predicts the dual solutions to matching problems.
They are also related to other heuristic-based initializations
for OT problems that do not use amortization
such as \citet{thornton2022rethinking}.

\subsubsection{Meta Optimal Transport by \citet{amos2022meta}}
Meta OT proposes to use a hypernetwork for this amortization.
Here, they consider settings where there is a \emph{meta-distribution}
over the measures to couple and costs to use, \ie $p(\alpha, \beta, c)$
and map directly from representations of the input measures and cost to
the optimal dual variables, \ie $f_\theta(\alpha, \beta, c)$.
They instantiate this idea for optimal transport between discrete
and continuous measures where the prediction $f_\theta$ warm-starts
standard dual-based solvers and can then be mapped to the optimal
primal coupling $\pi^\star$.

\textbf{Summary.}
$\gA_{\rm MetaOT}\defeq (g, \gF, \gP(\gX\times\gY)\times\gC, \gD, \gL_{\rm obj})$,
where $g$ is the dual objective,
$\gF$ represents the space of dual potentials, \ie $f\in\gF$,
$\gP(\gX\times\gY)$ is a space of distributions over $\gX\times\gY$
that $\alpha,\beta$ are sampled from, $\gC$ is a representation of the
space of costs, and $\gD$ is a meta-distribution over $\alpha,\beta,c$.

\subsubsection{Conditional Optimal Transport by \citet{bunne2022supervised}}
CondOT parameterizes a partially input-convex neural network (PICNN) \citep{amos2017input}
to condition the amortized OT solution on contextual information.
They focus primarily on the application of optimal transport in predicting
the effect of drugs on cellular populations for patients.
One of their key insights is to observe that OT problems need to
be repeatedly solved in this setting for different combinations of
drugs and patients. Instead of conditioning the amortization model
directly on the input measures, they condition it using auxiliary information
about the patient and drug. This enables them to obtain an OT coupling
and prediction of the drug effect even without knowing the target measure!
Parameterizing their amortization model as a PICNN has connections to
conditional neural processes \citep{garnelo2018conditional} and
is especially useful for conditioning the high-dimensional potentials
on the contextual information.

\textbf{Summary.}
$\gA_{\rm CondOT}\defeq (g, \gF, \gZ, \gD, \gL_{\rm obj})$,
where $g$ is the dual objective,
$\gF$ represents the space of dual potentials, \ie $f\in\gF$,
$\gZ$ is contextual information for the OT problems,
and $\gD$ is a meta-distribution over the contexts.

\subsection{Amortized convex conjugates (AmorConj) and $c$-transforms}
\label{sec:apps:amor-conj}

Most methods for computing the dual in \cref{eq:kantorovich-dual} between
a \emph{single} pair of measures needs to repeatedly compute the $c$-transform
in \cref{eq:c-transform} to evaluate the objective value $J$
in \cref{eq:kantorovich-dual-obj}.
This transform is typically not a computational bottleneck in
discrete settings when the measures have a few thousand points,
\eg in Sinkhorn solvers such as \citep{cuturi2013sinkhorn}.
Otherwise in some continuous settings, the conjugate operation may
be computationally challenging because it is an optimization problem
that needs to be repeatedly solved from scratch to obtain a
single Monte-Carlo estimate of $\int_\gY f^c(y) {\rm d}\beta(y)$
to evaluate $J$ in \cref{eq:kantorovich-dual-obj} \emph{once}.

Scoping to computing the optimal transport maps between Euclidean measures with
the negative inner product cost \citep[eq.~5.12]{villani2009optimal},
$f$ is a convex function and the $c$-transform operation $f^c$
in \cref{eq:c-transform} is the standard Legendre-Fenchel transform.
\citet{taghvaei2019wasserstein} discusses a lot of the theoretical foundations
underpinning Wasserstein-2 optimal transport and experimentally use a
numerical method that solves each conjugate operation from scratch.
To alleviate the computational bottleneck of this, many methods using
similar theoretical foundations amortize this conjugate operation
\citep{nhan2019threeplayer,makkuva2020optimal,korotin2019wasserstein,amos2023amortizing}.
The simplest instantiation of this amortization uses a fully-amortized model $\hat x_\theta(y)$
trained with objective-based learning of the conjugate objective.
\citet{amos2023amortizing} further discusses the modeling and loss choices
in this setting and also discusses the idea of fine-tuning the amortized
prediction with a numerical solver to ensure the dual objective is
accurately estimated.

\textbf{Summary.}
$\gA_{\rm AmorConj}\defeq (c(\cdot, y)-f(\cdot), \gX, \gY, \beta, \hat x_\theta, \gL_{\rm obj})$.

\begin{remark}
  Separate from amortizing the convex conjugate for optimal transport,
  \citet{garcia2023fishleg} proposes to amortize the solution to a
  convex conjugation problem arising when computing the natural gradient.
  Also related is \citet{bae2022amortized}, which amortizes the solution
  to a proximal optimization problem for updating parameters and can also
  amortize the solution to the natural gradient update
  (but doesn't explicitly go through the convex conjugate perspective).
\end{remark}

\subsection{Amortized Sliced Wasserstein ($\gA$-SW) by \citet{nguyen2022amortized}}
\label{sec:apps:amor-sliced}
Computing the Wasserstein distance between measures is computationally challenging.
The \emph{max-sliced} Wasserstein distance \citep{deshpande2019max} approximates
$W(\alpha, \beta)$ between measures with $\gX=\gY=\R^d$ by linearly projecting
(or slicing) the atoms of the measures down into 1-dimensional measures
where the 1-Wasserstein distance has a closed-form solution.
Max-sliced Wasserstein distances search over slices on the $d$-dimensional
sphere $\gS^{d-1}$ with
\begin{equation}
  \textnormal{Max-SW}(\alpha, \beta)\defeq \max_{\theta\in\gS^{d-1}} W(\theta; \alpha, \beta) \qquad W(\theta; \alpha, \beta)\defeq W(\theta_\#\alpha, \theta_\#\beta),
  \label{eq:max-sw}
\end{equation}
where $\theta_\#\alpha$ is the push-forward measure of $\mu$ through
$T_\theta(x)=\theta^\top x$.
\citet{nguyen2022amortized}
propose to amortize \cref{eq:max-sw} over mini-batches $\gD$ of size $m$
sampled from each measure.

\textbf{Summary.}
$\gA_{\rm SW}\defeq (W(\theta), \gS^{d-1}, \gX^m\times\gY^m, \gD, \gL_{\rm obj})$.

\section{Policy learning for control
  and reinforcement learning}
\label{sec:apps:ctrl}

\begin{figure}[t]
  \centering
  \includegraphics[width=\textwidth]{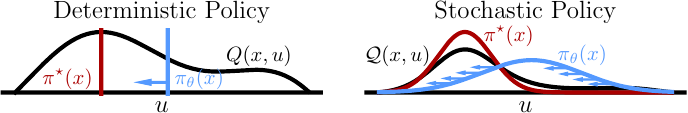}
  \caption{
    Many policy learning methods amortize
    optimization problem over the $Q$-values.
    Given a fixed input state $x$,
    the policy $\pi_\theta(x)$ predicts the
    maximum value $\pi^\star(x)$.
    A stochastic policy predicts
    a distribution that minimizes some probabilistic distance to
    the $Q$-distribution, such as the expected
    value or KL.
  }
  \label{fig:ctrl}
\end{figure}

Many control and reinforcement learning methods amortize the
solutions to a control optimization problem as illustrated in
\cref{fig:overview,fig:ctrl}.

\textbf{Distinction.} This section is on \emph{amortization for
reinforcement learning and control} and \emph{not} the
opposite direction of using
\emph{reinforcement learning for amortization}
that \cref{sec:learning:rl} discusses for parameter learning.

\subsection{Background}
\textbf{Preliminaries in Markov decision processes.}
\begin{definition}
  A \emph{Markov decision process} (MDP) can be represented with
  $\gM\defeq(\gX, \gU, p, r)$,
  where $\gX$ are the continuous \emph{states},
  $\gU$ are the continuous \emph{actions} or \emph{controls},
  $p(x' \mid x, u)$ is the \emph{transition dynamics}
  (also referred to as the \emph{system}, \emph{model},
  or \emph{world model}),
  which is a \emph{Markov kernel} providing
  the probability the next state $x'$
  given the current state-action $(x, u)$,
  and $r: \gX\times\gU$ is the \emph{reward function}.
\end{definition}
This section scopes to methods that control a
fully-observable and continuous MDP.
A \emph{policy} $\pi$ that \emph{controls}
the MDP provides a distribution over actions to sample from
for every state $x$ and induces \emph{state} and
\emph{state-control marginals} $\rho_t^\pi(\cdot)$ for each
time step $t$, which can be constrained to start from an initial
state $x_0$ as $\rho_t(\cdot|x)$.
In the non-discounted, infinite-horizon case
an \emph{optimal policy} $\pi^\star$ maximizes the reward over
rollouts of the MDP with
\begin{equation}
  \pi^\star(x)\in \argmax_\pi \E_{x\sim p_{\rm init}(x)} V^\pi(x)
  \qquad
  V^\pi(x)\defeq \sum_t \E_{(x_t, u_t)\sim\rho_t^\pi(\cdot\mid x)} r(x_t, u_t),
  \label{eq:pistar}
\end{equation}
where $p_{\rm init}$ is the \emph{initial state distribution}
and $V^\pi(x)$ is the expected \emph{value} of a policy $\pi$
starting from a state $x$ and that is taken over all possible
future rewards induced by the stochastic policy and dynamics.
Given the \emph{action-conditional value} $Q$ of a policy
defined by
\begin{equation}
  Q^\pi(x, u)\defeq r(x, u) + \E_{x'\sim p(\cdot|x,u)}\left[V^\pi(x')\right].
  \label{eq:Q}
\end{equation}
In the deterministic setting with a fixed $Q$ function,
an optimal policy can be obtained by solving the
\emph{max-Q} optimization problem
\begin{equation}
  \pi^\star(x) \in \argmax_u Q(x, u),
  \label{eq:Q-opt}
\end{equation}
which is the form that can be used to
interpret many control and reinforcement
learning methods as amortized optimization.
Instead of amortizing the solution to \cref{eq:Q-opt},
methods such as \citet{lowrey2018plan,ryu2019caql}
explicitly solve the max-Q problem.

\textbf{Control of deterministic MDPs with deterministic policies.}
If all of the components of the MDP are known,
no learning needs to be done to obtain an optimal policy
and standard \emph{model predictive control} (MPC) methods often work well.
In \emph{deterministic MDPs}, the dynamics are deterministic
and can be written as $x'\defeq p(x, u)$.
These can be solved with deterministic policies,
which turns the expected value and marginal distributions
into Dirac delta distributions that can be computed with single rollout.
An optimal controller from an initial state $x_1$ can thus
be obtained by solving the \emph{finite-horizon} problem over the
(negated) value approximated with a horizon length of
$H$ timesteps with
\begin{equation}
  u^\star_{1:H}(x_1) \defeq \argmin_{u_{1:H}} \sum_t C_t(x_t, u_t)\; \subjectto\; x_{t+1}=p(x_t, u_t),
  \label{eq:mpc}
\end{equation}
where the \emph{cost} $C$ at each time is usually the negated reward
$C_t(x_t, u_t)\defeq -r(x_t, u_t)$.
The field of optimal control studies methods for solving control
optimization problems of the form \cref{eq:mpc} and standard references
include \citet{bertsekas2000dp,kirk2004optimal,camacho2013model}.
\Cref{eq:mpc} induces the policy $\pi(x)\defeq u^\star_1(x)$ that
solves the MDP if the horizon $H$ is long enough.
Using a \emph{terminal cost} at the end of the horizon can
also give the controller information about how the system will
perform beyond the finite-horizon rollouts being used,
for example with
$C_H(x_H, u_H)\defeq -Q^\pi(x_H, u_H)$.

\textbf{Reinforcement learning when the dynamics aren't known.}
Optimal control methods work well when the dynamics $p$
of the MDP are known, which is an unfortunately strong
assumption in many settings where the system can only
be sampled from.
In these settings \emph{reinforcement learning} (RL)
methods thrive and solve the MDP given access to
\emph{only} samples from the dynamics.
While summarizing all of the active RL methods is out-of-scope
for this tutorial, the core of these methods is typically on
1) \emph{policy evaluation} to estimate the \emph{value}
of a policy given only samples from the system, and
2) \emph{policy improvement} to improve the policy
using the value estimation.

\textbf{Extensions in stochastic control.}
The max-Q problem in \cref{eq:Q-opt} can be extended to the
stochastic optimization settings \cref{sec:extensions:sto}
briefly covered when the policy $\pi$
represents a \emph{distribution} over the action space $\gU$.
The most common objectives for stochastic policies are
1) the expected $Q$-value under the policy with
\begin{equation}
  \pi^\star(x) \in \argmax_{\pi\in\Pi} \gE_Q(\pi; x) \qquad \gE_\gQ(\pi; x)\defeq \E_{u\sim\pi(\cdot)} Q(x,u),
  \label{eq:Q-opt-sto-exp}
\end{equation}
or 2) the KL distance
\begin{equation}
  \pi^\star(x) \in \argmin_{\pi\in\Pi} \D_\gQ(\pi; x) \qquad \D_\gQ(\pi; x)\defeq \kl{\pi(\cdot)}{\gQ(x, \cdot)},
  \label{eq:Q-opt-sto-kl}
\end{equation}
where $\gQ(x, \cdot)\propto \exp\left\{Q(x, \cdot)/\alpha\right\}$
is a $Q$-distribution induced by the $Q$ values that is
inversely scaled by $\alpha\in\R_+$.
The policy $\pi$ is usually represented as the parameters of a distribution
and thus $\Pi$ becomes the space of these parameters.
In most cases, $\pi$ is a Gaussian $\gN(\mu, \Sigma)$
with a diagonal covariance $\Sigma$ and thus
\cref{eq:Q-opt-sto-exp,eq:Q-opt-sto-kl} can be turned into unconstrained continuous
optimization problems of the form \cref{eq:opt} by
projecting onto the Gaussian parameters.
Stochastic value gradient methods such as \citet{heess2015learning}
often amortize \cref{eq:Q-opt-sto-exp}, while
\citet{levine2013guided,haarnoja2018soft} propose methods
that build on \cref{eq:Q-opt-sto-kl}, often adding additional
softening and entropic terms to encourage the policy
and value estimate to explore more and not converge too
rapidly to a suboptimal minima.
One last note is that the smoothing that the policy performs in
\cref{eq:Q-opt-sto-exp} is nearly identical to the objective
smoothing considered in \cref{sec:smooth}, except in that setting
the variance of the smoother remains fixed.

\textbf{Connecting stochastic control back to deterministic control.}
This portion shows that taking stochastic policies to be Dirac delta distributions
in \cref{eq:Q-opt-sto-exp,eq:Q-opt-sto-kl} recovers the solutions to the
deterministic control problem in \cref{eq:Q-opt}.
Taking a larger classes of policy distributions, such as
Gaussians, can then be interpreted as smoothing the $Q$
values to avoid 1) falling into poor local optima and
2) unstable regions where only a few actions have
high value, but the rest have poor values.
For the following, let $\delta_u(\cdot)$ be Dirac delta
distribution with a parameter $u\in\R^n$ indicating the
location of the mass.

\begin{proposition}
  Let $\pi$ be a Dirac delta distribution $\delta_u(\cdot)$.
  Then the solution $\pi^\star(x)$ to the expected $Q$ problem in
  \cref{eq:Q-opt-sto-exp} is the solution to the deterministic
  max-$Q$ problem in \cref{eq:Q-opt}.
\end{proposition}

\begin{proof}
  Let $\Pi=\R^n$ be the parameter space of $\pi$ and transform
  \cref{eq:Q-opt-sto-exp} to optimize over it:
  \begin{equation}
  \pi^\star(x) \in \argmax_{u\in\R^n} \E_{\tilde u\sim\delta_u(\cdot\mid x)}{Q(x, \tilde u)}.
  \label{eq:exp-dirac-opt}
  \end{equation}
  The expectation over the Dirac evaluates to
  \begin{equation}
    \E_{\tilde u\sim\delta_u(\cdot\mid x)}{Q(x, \tilde u)} = Q(x,u)
    \label{eq:exp-expansion}
  \end{equation}
  and thus \cref{eq:exp-dirac-opt}
  which is the max-$Q$ operation in \cref{eq:Q-opt}.
\end{proof}

\noindent Similarly for the for the KL problem in \cref{eq:Q-opt-sto-kl}:
\begin{proposition}
  Let $\pi$ be a Dirac delta distribution $\delta_u(\cdot)$.
  Then the solution $\pi^\star(x)$ to the KL problem in
  \cref{eq:Q-opt-sto-kl} is the solution to the deterministic
  max-$Q$ problem in \cref{eq:Q-opt}.
\end{proposition}

\begin{proof}
  Let $\Pi=\R^n$ be the parameter space of $\pi$ and transform
  \cref{eq:Q-opt-sto-kl} to optimize over it:
  \begin{equation}
  \pi^\star(x) \in \argmin_{u\in\R^n} \kl{\delta_u(\cdot)}{\gQ(x, \cdot)}.
  \label{kl-dirac-opt}
  \end{equation}
  Expanding the KL distance then gives the density of $\gQ$ at the mass:
  \begin{equation}
    \begin{aligned}
      \kl{\delta_u(\cdot)}{\gQ(x, \cdot)}
      & = \E_{\tilde u\sim \delta_u(\cdot)} \left[ \log{\delta_u(\tilde u)} - \log \gQ(x,\tilde u)\right] \\
      & = -\log \gQ(x,u)+C \\
      & = -\log \frac{1}{\gZ_x}\exp\left\{Q(x,u)\right\} + C\\
      & = -Q(x, u) + \log \gZ_x + C \\
    \end{aligned}
    \label{eq:kl-expansion}
  \end{equation}
  where $C$ is a constant that does not depend on $u$,
  $\gQ(x,u)$ is the density at $u$, and
  $\gZ_x$ is the normalization constant for $\gQ(x,\cdot)$
  that does not depend on $u$.
  Putting \cref{eq:kl-expansion} back into
  \cref{kl-dirac-opt} and removing the constants that
  do not depend on $u$ gives
  \begin{equation}
  \pi^\star(x)\in \argmin_{u\in\R^n} -Q(x,u),
  \end{equation}
  which is the max-$Q$ operation in \cref{eq:Q-opt}.
\end{proof}

\subsection{Behavioral cloning and imitation learning}
This section starts in the setting where regression-based amortization
is done to predict the solution of a controller that
solves \cref{eq:pistar}.
These settings assume access to a controller, or samples from it,
that uses traditional methods, and \emph{not} learning,
to solve \cref{eq:pistar} with the true or approximated dynamics.
These solutions are typically available as samples
from a policy $\pi^\star(x)$ that provides the solution
to the max-Q problem in \cref{eq:Q-opt} for regression-based amortization.
In some settings these methods also use the
optimal finite-horizon sequence $u^\star_{1:H}(x)$
from a solution to \cref{eq:mpc}.

Imitation learning methods such as behavioral cloning can be
seen as a regression-based amortization
that seek to distill, or clone,
the expert's behavior into a learned model $\pi_\theta(x)$
that predicts the expert's action given the state $x$
\citep[Chapter 3]{osa2018algorithmic}.
Deterministic BC methods often regress onto the expert's
state-action pairs $(x, \pi^\star(x))$
sampled some distribution of states $p(x)$ with
\begin{equation}
  \argmin_\theta \E_{x\sim p(x)} \|\pi^\star(x)-\pi_\theta(x)\|_2^2,
  \label{eq:bc}
\end{equation}
where, for example, the model $\pi_\theta$ could
be a neural network.
Thus BC in this setting performs regression-based amortization
$\gA_{\rm BC}\defeq (-Q, \gU, \gX, p(x), \pi_\theta, \gL_{\rm reg})$.
Extensions from this setting could include when
1) the model $\pi_\theta$ is a sequence model that
predicts the entire sequence $u^\star_{1:H}$, and
2) the MDP or policy is stochastic and
\cref{eq:bc} turns into an amortized maximum likelihood
problem rather than a regression problem.

\textbf{Warning.}
One crucial difference between behavioral cloning
and all of the other applications considered here
is that behavioral cloning does not assume knowledge
of the original objective or cost used by the expert.
This section assumes that an optimal objective
exists for the purposes of seeing it as regression-based
amortization.
Settings such as inverse control and RL
explicitly recover the expert's objective, but are
beyond the scope of our amortization focus.

\subsection{The deterministic policy gradient}
The policy learning of many model-free actor-critic
reinforcement learning algorithms on continuous spaces
can be seen as objective-based amortization of the
max-$Q$ operation in \cref{eq:Q-opt}.
This includes the policy improvement step for
deterministic policies as pioneered by the
deterministic policy gradient (DPG) by \citet{silver2014dpg},
deep deterministic policy gradient (DDPG) by \citet{lillicrap2015continuous},
and twin delayed DDPG (TD3) by \citet{fujimoto2018td3}.
All of these methods interweave
1) \emph{policy evaluation} to estimate the
state-action value $Q^{\pi_\theta}$ of the current policy
$\pi_\theta$, and
2) \emph{policy improvement} to find the policy
that best-optimizes the max-Q operation with
\begin{equation}
  \argmax_\theta \E_{x\sim p(x)} Q(x, \pi_\theta(x)).
  \label{eq:dpg-loss}
\end{equation}

\textbf{Summary.}
The DPG family of methods perform objective-based
amortization of the max-$Q$ optimization problem with
\begin{equation}
\gA_{\rm DPG}\defeq (-Q, \gU, \gX, p(x), \pi_\theta, \gL_{\rm obj}).
\end{equation}

\subsection{The stochastic value gradient and soft actor-critic}
This amortization viewpoint can be extended to \emph{stochastic} systems
and policies that use the \emph{stochastic value gradient} (SVG)
and covers a range of model-free to model-based method depending on
how the value is estimated
\citep{byravan2019imagined,hafner2019dream,byravan2021evaluating,amos2021model}.
As observed in
\citet{haarnoja2018soft,amos2021model},
the policy update step in the soft actor-critic
can also be seen as a model-free
stochastic value gradient with the value estimate
regularized or ``softened'' with an entropy term.
These methods learn policies $\pi_\theta$ that
amortize the solution to a stochastic optimization
problem, such as \cref{eq:Q-opt-sto-exp,eq:Q-opt-sto-kl},
over some distribution of states $p(x)$, such as the
stationary state distribution or an approximation of it
with a replay buffer.
Taking the expected value under the policy gives
the policy loss
\begin{equation}
  \argmax_\theta \gL_{\rm SVG,\gE}(\pi_\theta) \qquad \gL_{\rm SVG,\gE}(\pi_\theta)\defeq \E_{x\sim p(x)} \E_{u\sim\pi_\theta(\cdot|x)} Q(x, u),
  \label{eq:svg-loss-exp}
\end{equation}
and taking the minimum KL distance gives the policy loss
\begin{equation}
  \argmin_\theta \gL_{\rm SVG,KL}(\pi_\theta) \qquad \gL_{\rm SVG,KL}(\pi_\theta)\defeq \E_{x\sim p(x)} \kl{\pi_\theta(\cdot\mid x)}{\gQ(x, \cdot)},
  \label{eq:svg-loss-kl}
\end{equation}
which softens the policy and value estimates with an entropy regularization
as in \citet{haarnoja2018soft,amos2021model}.
\Cref{eq:svg-loss-kl} can be seen as an entropy-regularized value
estimate by expanding the KL
\begin{equation}
  \begin{aligned}
  & \nabla\E_{x\sim p(x)} \kl{\pi_\theta(\cdot\mid x)}{\gQ(x, \cdot)} = \\
  & \hspace{0.5in} \nabla\E_{x\sim p(x)} \E_{u\sim \pi_\theta(u\mid x)} \left[ \log(\pi_\theta(u\mid x)) - Q(x, u)/\alpha \right].
  \end{aligned}
\end{equation}
\citet{mohamed2020monte} discusses standard ways of optimizing
\cref{eq:svg-loss-exp,eq:svg-loss-kl}, which could be with
a likelihood ratio gradient estimator \citep{williams1992simple}
or via reparameterization.

\textbf{Summary.}
The policy update of methods based on the SVG and SAC perform
objective-based amortization of a stochastic control
optimization problem with
\begin{equation*}
\gA_{\rm SVG}\defeq (\D_\gQ\;\mathrm{or}\;-\gE_Q, \gP(\gU), \gX, p(x), \pi_\theta, \gL_{\rm KL}).
\end{equation*}

\textbf{SVG-based amortization for model-free and model-based methods.}
SVG-based policy optimization provides a spectrum of model-free
to model-based algorithms depending on if the value estimate is
approximated in a model-free or model-based way, \eg with rollouts
of the true or approximate dynamics.
\citet{heess2015learning} explored this spectrum and proposed
a few key instantiations.
\citet{byravan2019imagined,amos2021model} use learned
models for the value approximation.
The variation in \citet{hafner2019dream} amortizes a model-based
approximation using the \emph{sum} of the value estimates
of a model rollout.
\citet{henaff2019model} explore uncertainty-based regularization to
amortized policies learned with value gradients.
\citet{xie2020latent} consider hierarchical amortized optimization
that combine a higher-level planner based on CEM with a
lower-level fully-amortized policy learned
with stochastic value gradients.
\citet{byravan2021evaluating} perform a large-scale evaluation of
amortization methods for control and study algorithmic variations
and generalization capabilities, and consider methods based on
the stochastic value gradient and behavioral cloning.

\subsection{PILCO by \citet{deisenroth2011pilco}}
PILCO is an early example that uses gradient-based
amortization for control.
They assume that only samples from the dynamics model are
available, fit a Gaussian process to them, and then use that
to estimate the finite-horizon $Q$-value of a
deterministic RBF policy $\pi_\theta$.
The parameters $\theta$ are optimized by taking gradient
steps to find the policy resulting in the maximum value
over some distribution of states $p(x)$, \ie
\begin{equation}
  \argmax_\theta \E_{x\sim p(x)} Q(x, \pi_\theta).
\end{equation}

\textbf{Summary.}
$\gA_{\rm PILCO}\defeq (-Q, \gU, \gX, p(x), \pi_\theta, \gL_{\rm obj})$

\subsection{Guided policy search (GPS)}
The GPS family of methods
\citep{levine2013guided,levine2014learning,levine2016end,montgomery2016guided}
fits an approximate dynamics model to data and then
amortizes the solutions of a controller that solves
the MPC problem in \cref{eq:mpc} within a local region of
the data to improve the amortization model.
Given samples of $\pi^\star$ from a controller that solves
\cref{eq:Q-opt-sto-kl},
GPS methods typically optimize the KL divergence between
the controller and these samples with
\begin{equation}
  \argmin_\theta \E_{x\sim p(x)} \kl{\pi_\theta(\cdot\mid x)}{\pi^\star(\cdot\mid x)}
  \label{eq:gps}
\end{equation}

\textbf{Summary.}
$\gA_{\rm GPS}\defeq (\D_\gQ, \gP(\gU), \gX, p(x), \pi_\theta, \gL_{\rm KL})$

Related methods such as \citet{sacks2022learning} study learning to
optimize for imitating other controllers and are capable of
working in cases when derivative information isn't available.

\subsection{POPLIN by \citet{wang2019exploring}}
POPLIN explores behavioral cloning methods based on regression
and generative modeling and observes that the
parameter space of the amortized model is a reasonable
space to solve new control optimization problems over.
The methods discussed here

\textbf{Distillation and amortization.}
POPLIN first trains a fully-amortized
model on a dataset of trajectories with behavioral cloning
or generative modeling.
This section only consider the BC variant trained with
$\gA_{\rm BC}$, which provides an \emph{optimal}
fully-amortized model $\pi_{\theta^\star}$.

\textbf{Control.}
Next they explore ways of using the learned policy
$\pi_{\theta^\star}$ to solve the model predictive
control optimization problem in \cref{eq:mpc}.
The \emph{POPLIN-A} variant (``A'' stands for ``action'')
uses $\pi_{\theta^\star}$ to predict an initial
control sequence $\hat u^0_{1:H}$ that is then passed as
an input to a controller based on the cross-entropy method
over the action space that uses a learned model on the trajectories.
The \emph{POPLIN-P} variant (``P'' stands for ``parameter'')
suggests that the \emph{parameter space} of the fully-amortized
model has learned useful information about the structure
of the optimal action sequences.
As an alternative to solving the MPC problem in \cref{eq:mpc}
over the action space,
POPLIN-P proposes to use CEM to find a
perturbation $\omega$ to the optimal parameters $\theta^\star$
that maximizes the value from a state $x$ with
\begin{equation}
  \omega^\star(x) \in \argmax_\omega Q(x, u;
    \pi_{\theta^\star+\omega}).
  \label{eq:POPLIN-P}
\end{equation}
Thus the action produced by
$\pi_{\theta^\star+\omega^\star}(x)$
is a solution the control problem in \cref{eq:mpc}
obtained by adapting the policy's parameters to the
state $x$.

\textbf{Summary.} POPLIN is an extension of
of behavioral cloning that amortizes control optimization
problems with
\begin{equation}
\gA_{\rm POPLIN}\defeq (-Q, \gU, \gX, p(x), \pi_\theta, \gL_{\rm reg}).
\end{equation}
The initial phase is fully-amortized behavioral cloning.
The second fine-tuning phase can be seen as semi-amortization
that learns only the initialization $\theta$ and
finds $\omega$ with CEM, with a regression-based loss $\gL_{\rm reg}$
that \emph{only} has knowledge of the initial model and does not include
the adaptation.

\subsection{The differentiable cross-entropy method by \citet{amos2019dcem}}

\begin{figure*}[t]
  \centering
  \includegraphics[width=\textwidth]{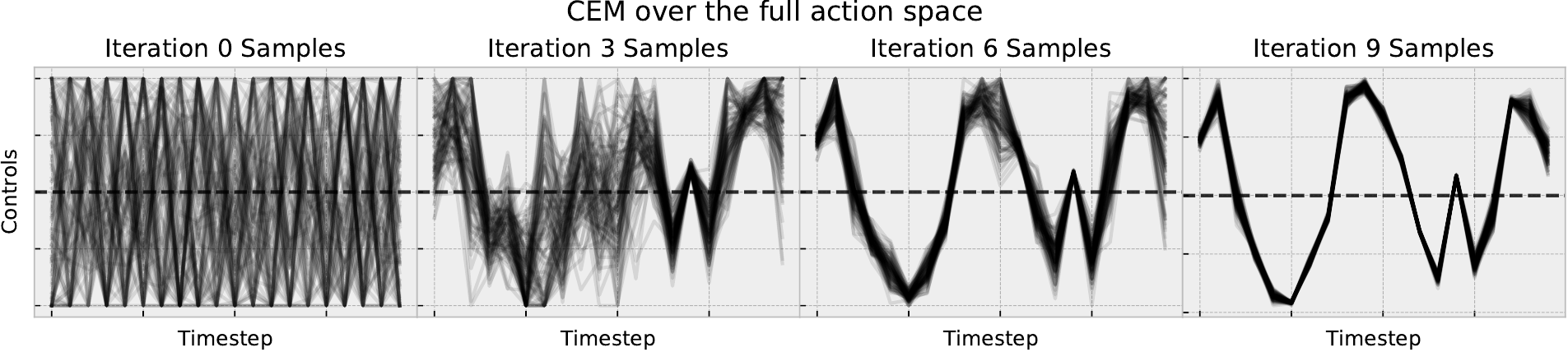} \\[4mm]
  \includegraphics[width=\textwidth]{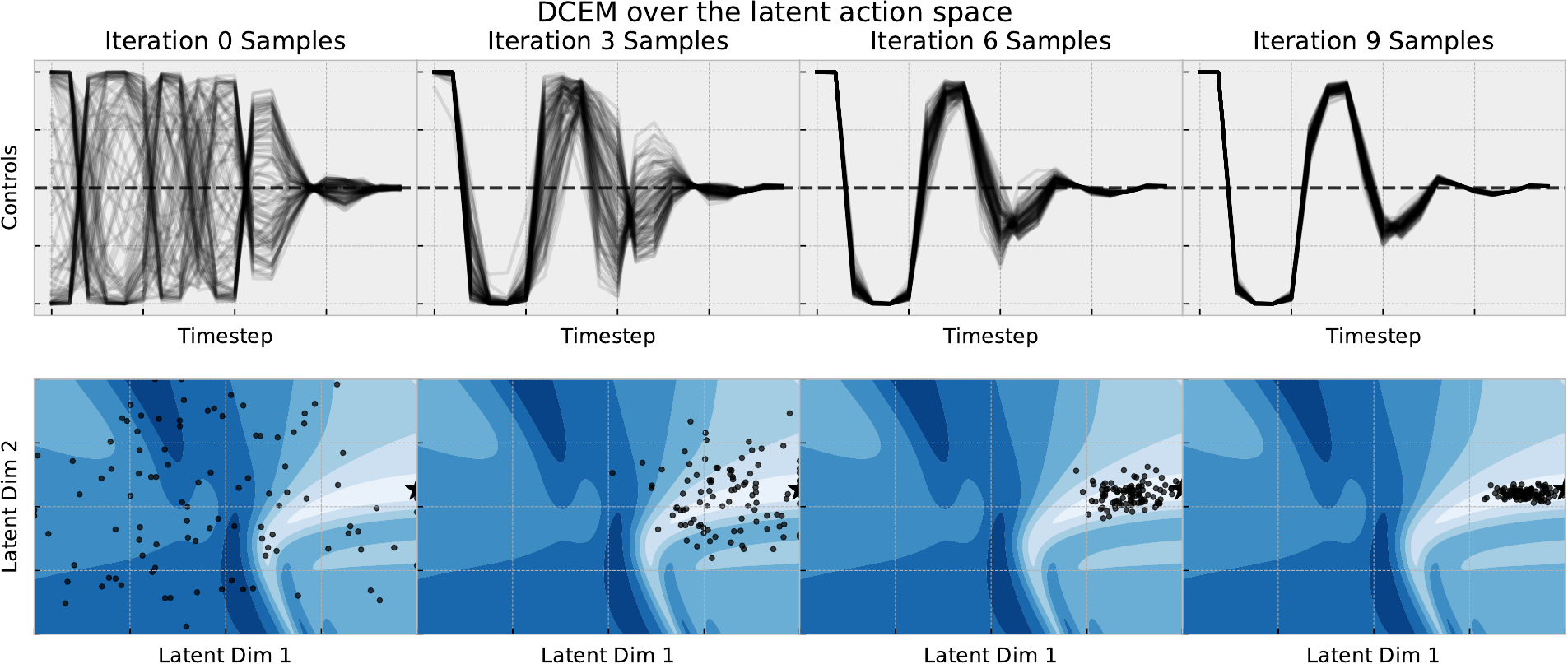}
  \caption{
    The differentiable cross-entropy method (DCEM)
    \citep{amos2019dcem} can be used
    to create semi-amortized controllers that learn a
    latent space $\gZ$ over control sequences.
    This visualization taken from the DCEM paper shows the samples that
    CEM and DCEM generate to solve the cartpole task starting from the
    same initial system state.
    The plots starting at the top-left show that CEM initially starts with
    no temporal knowledge over the control space whereas
    the latent space learned through DCEM generates a more
    feasible distribution over control sequences
    in each iteration that make them smooth and cyclic.
    The contours on the bottom show the controller's cost
    surface $C(z; x)$ for an initial state $x$ where
    the lighter colors show regions with lower costs.
  }
  \label{fig:dcem}
\end{figure*}

Differentiable control \citep{amos2018dmpc} is a budding area
of work with the goal of integrating controllers into
end-to-end modeling pipelines to overcome problems
such as objective mismatch \citep{lambert2020objective}.
The differentiable cross-entropy method (DCEM) was created
towards the goal of doing this with controllers
based on the cross-entropy method.
Otherwise, as in \citet{wang2019exploring}, CEM needs to be done as
a secondary step \emph{after} learning and the learning process is
not aware of the final policy that running CEM will induce.
The key step to differentiate through CEM is to make the
top-$k$ operation smooth and differentiable by using the
differentiable top-$k$ operation proposed in
\citet{amos2019limited} called the limited multi-label projection layer.

\citet{amos2019dcem} considers a semi-amortized learning setting
that learns a latent domain for control, which can be seen as
a similarly-motivated alternative to the parameter-space
control done in POPLIN.
The key piece of \emph{latent control}
is to learn a \emph{decoder} $\varphi_\theta:\gZ\rightarrow \gU^H$
that maps from a low-dimensional latent space $\gZ$ to the
$H$-step control space that solves \cref{eq:mpc}.
Learning a latent space is useful if there are many redundancies
and possibly bad local minima on the original control space
$\gU^H$ that the latent space can get rid of.
Given an initial state $x_1$, the optimal latent representation
can be obtained by solving the control optimization problem
over $\gZ$ with
\begin{equation}
  \hat z_\theta(x_1)\in \argmin_{z\in\gZ} C_\theta(z; x_1)\;
  \label{eq:dcem-latent}
\end{equation}
where $C_\theta(z; x_1)$ is the expected cost of rolling out the
control sequence $u_{1:H}=\varphi(z)$ from the initial state
$x_1$, for example on deterministic systems $C$
could be the sum of negated rewards
\begin{equation}
  C(z; x) \defeq -\sum_{t=1}^H r(x_t, x_t)\;
  \subjectto\;x_{t+1}=p(x_t, u_t)\; {\rm and}\; x_{1:H}=\varphi_\theta(z)
  \label{eq:latent-cost}
\end{equation}
Solving \cref{eq:dcem-latent} with DCEM enables the optimal
solution $\hat z_\theta(x)$ to be differentiated with respect
to the parameters $\theta$ of the decoder $\varphi_\theta$.
The final predicted control sequence can be
obtained with the decoder
$\hat u_{1:H}(x; \theta)\defeq \varphi_\theta(\hat z_\theta(x))$
and the decoder can be learned by regressing
onto ground-truth control sequences $u^\star(x)$ with
\begin{equation}
  \argmin_\theta \gL_{\rm DCEM}(\hat u_{1:H}(\cdot; \theta))
\label{eq:dcem-obj}
\end{equation}
where the loss is given by
\begin{equation}
  \gL_{\rm DCEM}(\hat u_{1:H}(\cdot; \theta))\defeq \E_{x \sim p(x)} \|u^\star_{1:H}(x) - \hat u_{1:H}(x; \theta)\|_2^2.
  \label{eq:dcem-loss}
\end{equation}
\Cref{fig:dcem} visualizes an example on the cartpole task
where this is able to learn a latent space that captures the
cyclic and smoothness structure of the optimal control
sequence space.

\textbf{Overview.} Learning a latent domain with the differentiable
cross-entropy method is a semi-amortization method with
\begin{equation}
\gA_{\rm DCEM}\defeq (-Q, \gU, \gX, p(x), \pi_\theta, \gL_{\rm reg}),
\end{equation}
where the decoder $\varphi_\theta$ shows up from
the policy $\pi_\theta$ solving the latent
optimization problem with DCEM.

\subsection{Iterative amortized policy optimization (IAPO) by \citet{marino2020iterative}}
IAPO takes a probabilistic view and
starts with the observation that DPG and SVG methods
are amortized optimization problems with fully-amortized
models with an objective-based loss.
They then suggest to replace the model with an iterative
semi-amortized model where the policy $\pi_\theta$ internally
takes gradient steps on the actions of the underlying
control optimization problem, and explore this
semi-amortized policy in model-free and model-based
reinforcement learning settings.
Thus in the deterministic setting IAPO performs semi-amortized
optimization
$\gA_{\rm IAPO}\defeq (-Q, \gU, \gX, p(x), \pi_\theta, \gL_{\rm obj})$.

\textbf{Warning.}
There’s an interplay between the accuracy and quality of the policy
optimizer and the value estimator.
Because the value estimator is used to create it's own target estimates,
a better policy optimizer and controller can
exploit optimistic inaccuracies in the value network.
In other words, a seemingly better policy optimizing
an inaccurate value estimate may result in a worse
policy on the real system.
This issue also arises when fully-amortized policies over-optimize
the value estimate too early on in training, but is
exacerbated with semi-amortized and iterative policies.

\subsection{Learning the value function}
\label{sec:Q-learning}
The $Q$-value function \cref{eq:Q} for finding an MDP
policy is often unknown and needs to be estimated from
data along with the policy $\pi$ that is amortizing it,
\eg in \cref{eq:Q-opt}.
Actor-critic methods are a common reinforcement learning approach
that jointly learn a policy $\pi$ (the actor) and $Q$-value
estimate (the critic) \citep{konda1999actor,sutton2018reinforcement}.
The policy amortizes the current $Q$ estimate, \eg by using
an approach previously discussed in this section, and the
$Q$ function is fit to data sampled from the system.
One way of learning the $Q$ function is to replace the value estimate $V^\pi$
in \cref{eq:Q} with the $Q$-value estimate to yield the relationship
\begin{equation}
  Q^\pi(x, u)\defeq r(x, u) + \E_{x'\sim p(\cdot|x,u),u'\sim\pi(x')}\left[Q^\pi(x', u')\right],
  \label{eq:Q-bellman}
\end{equation}
which is referred to as the \emph{Bellman equation}
\citep{bellman1966dynamic,sutton2018reinforcement}.
\Cref{eq:Q-bellman} is an equality that should hold over all
states and actions in the system and a value estimate can be
parameterized and learned to satisfy the relationship.
While the best way of learning the $Q$ estimate is an
open research topic \citep{watkins1992q,baird1995residual,ernst2005tree,maillard2010finite,scherrer2010should,geist2017bellman,le2019batch,fujimoto2022should},
a common way is to optimize residual of \cref{eq:Q-bellman} with
\begin{equation}
  \argmin_\phi \E_{(x,u)\sim\gD} \left|Q_\phi ^\pi(x, u) - \left(r(x, u) + \E_{x'\sim p(\cdot|x,u),u'\sim\pi(x')}Q^\pi_{\bar\phi}(x', u')\right)\right|^2,
  \label{eq:Q-bellman-estimation}
\end{equation}
where $\bar\phi$ is a detached version of the
rolling mean of the parameters.

\chapter{Implementation and software examples}
\label{sec:implementation}

\begin{table}[t]
  \centering
  \caption{Dimensions for the settings considered in this section}
  \label{tab:sizes}
\resizebox{\textwidth}{!}{
\begin{tabular}{lll}\toprule
  Setting & Context dimension $|\gX|$ & Solution dimension $|\gY|$ \\ \midrule
  VAE on MNIST (\ref{sec:impl:vaes}) & $784$ {\color{gray} (=$28\cdot 28$, MNIST digits)} & $20$ {\color{gray}(parameterizing a 10D Gaussian)} \\
  Model-free control (\ref{sec:impl:model-free}) & $45$ {\color{gray}(humanoid states)} & $17$ {\color{gray} (action dimension)} \\
  Model-based control (\ref{sec:impl:model-based}) & $45$ {\color{gray}(humanoid states)} & $51$ {\color{gray} (=$17\cdot 3$, short action sequence)} \\
  Sphere (\ref{sec:impl:sphere}) & $16$ {\color{gray}($c$-convex function parameterizations)} & $3$ {\color{gray} (sphere)} \\ \bottomrule
\end{tabular}}
\end{table}

Turning now to the implementation details, this section
looks at how to develop and analyze amortization software.
The standard and easiest approach in most settings is to use
automatic differentiation software such as
\citet{maclaurin2015autograd,al2016theano,abadi2016tensorflow,bezanson2017julia,agrawal2019tensorflow,paszke2019pytorch,bradbury2020jax}
to parameterize and learn the amortization model.
There are many open source implementations and re-implementations
of the methods in \cref{sec:apps} that provide a concrete
starting point to start building on them.
This section looks closer at three specific implementations:
\cref{sec:impl:eval} evaluates the amortization components
behind existing implementations of variational autoencoders
\cref{sec:apps:avi} and control \cref{sec:apps:ctrl} and
\cref{sec:impl:sphere} implements and trains an amortization model
to optimize functions defined on a sphere.
\Cref{tab:sizes} summarizes the concrete dimensions of the amortization
problems considered here and \cref{sec:impl:software}
concludes with other useful software references.
The source code behind this section is available at
\url{https://github.com/facebookresearch/amortized-optimization-tutorial}.

\section{Amortization in the wild: a deeper look}
\label{sec:impl:eval}

This section shows code examples of how existing implementations
using amortized optimization define and optimize their models
for variational autoencoders (\cref{sec:impl:vaes})
and control and policy learning (\cref{sec:impl:model-free,sec:impl:model-based}).
The amortization component in these systems is often a part
of a larger system to achieve a larger task:
VAEs also reconstruct the source data
after amortizing the ELBO computation in \cref{eq:vae-full}
and policy learning methods also estimate the
$Q$-value function in \cref{sec:Q-learning}.
This section scopes to the amortization components to show
how they are implemented.
I have also added evaluation code to the pre-trained amortization models
from existing repositories and show that the amortized approximation
often obtains a solution up to \textbf{25000 times} faster than
solving the optimization problems from scratch on an
NVIDIA Quadro GP100 GPU.

\subsection{The variational autoencoder (VAE)}
\label{sec:impl:vaes}

This section looks at the code behind
standard VAE \citep{kingma2013auto} that follows the
amortized optimization setup described in \cref{sec:apps:vae}.
While there are many implementations for training and
reproducing a VAE, this section will use the implementation at
\url{https://github.com/YannDubs/disentangling-vae},
which builds on the code behind
\citet{dupont2018learning} at
\url{https://github.com/Schlumberger/joint-vae}.
While the repository is focused on disentangled representations
and extensions of the original VAE formulation, this
section only highlights the parts corresponding to the
original VAE formulation.
The code uses standard PyTorch in a minimal way that
allow us to easily look at the amortization components.

\textbf{Training the VAE.}
\Cref{lst:vae} paraphrases the relevant snippets of code to
implement the main amortization problem in \cref{eq:vae-amor}
for image data where the likelihood is given by a Bernoulli.
\Cref{lst:vae.encoder} defines an encoder $\hat \lambda_\theta$,
to predicts a solution to the ELBO implemented in \cref{lst:vae.elbo},
which is optimized in a loop over the training data (images)
in \cref{lst:vae.loop}.
The \path{README} in the repository contains instructions
for running the training from scratch.
The repository contains the binary of a model trained
on the MNIST dataset \citep{lecun1998mnist}, which
the next portion evaluates.

\begin{figure}[t]
  \centering
  \vspace{-4mm}
  \includegraphics[width=0.49\textwidth]{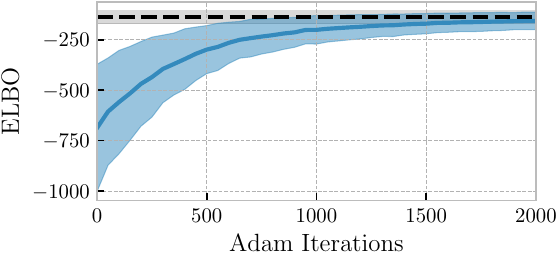}
  \hfill
  \includegraphics[width=0.49\textwidth]{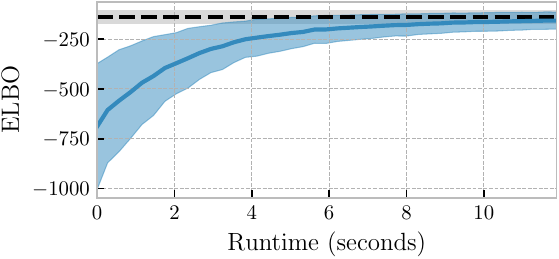} \\
  \cblock{0}{0}{0} Amortized encoder $\hat\lambda_\theta(x)$ --- runtime: 0.4ms
  \caption{
    Runtime comparison between Adam and an amortized encoder $\hat\lambda_\theta$
    to solve \cref{eq:elbo-opt} for a VAE on MNIST.
    This uses a batch of 1024 samples and was
    run on an unloaded NVIDIA Quadro GP100 GPU.
    The values are normalized so that $\lambda(x)=0$ takes a value of -1 and
    the optimal $\lambda^\star$ takes a value of 0.
    The amortized policy is approximately
    \textbf{25000} times faster than solving the
    problem from scratch.
  }
  \label{fig:vae-performance}
\end{figure}

\begin{figure}
  \centering
\begin{subfigure}[b]{\textwidth}
\begin{lstlisting}
class Encoder(nn.Module): # From disvae.models.encoders
    def forward(self, x): # x is the amortization context: the original data
        mu_logvar = self.convnet(x)
        mu, logvar = mu_logvar.view(-1, self.latent_dim, 2).unbind(-1) # Split
        return (mu, logvar) # = latent_dist or \lambda
\end{lstlisting}
\caption{Forward definition for the encoder $\hat \lambda_\theta(x)$.
  \path{self.convnet} uses the architecture
  from \citet{burgess2018understanding}.}
\label{lst:vae.encoder}
\end{subfigure}
\begin{subfigure}[b]{\textwidth}
\begin{lstlisting}
# From disvae.models.losses.BetaHLoss with a Bernoulli likelihood
def estimate_elbo(data, latent_dist):
    mean, logvar = latent_dist

    reconstructed_batch = sample_and_decode(latent_dist)
    log_likelihood = -F.binary_cross_entropy(
        reconstructed_batch, x, reduce=False).sum(dim=[1,2,3])

    # Closed-form distance to the prior
    latent_kl = 0.5 * (-1 - logvar + mean.pow(2) + logvar.exp())
    kl_to_prior = latent_kl.sum(dim=[-1])

    loss = log_likelihood - kl_to_prior
    return loss.mean()
\end{lstlisting}
\caption{Definition of the $\ELBO$ in \cref{eq:elbo}}
\label{lst:vae.elbo}
\end{subfigure}

\begin{subfigure}[b]{\textwidth}
\begin{lstlisting}
model = Encoder()
for batch in iter(data_loader):
    latent_dist = model(batch)
    loss = -estimate_elbo(batch, latent_dist)
    self.optimizer.zero_grad()
    loss.backward()
    self.optimizer.step()
\end{lstlisting}
\caption{Main VAE training loop for the encoder}
\label{lst:vae.loop}
\end{subfigure}

\caption{Paraphrased PyTorch code examples of the key
  amortization components of a VAE from
  \url{https://github.com/YannDubs/disentangling-vae}.}
\label{lst:vae}
\end{figure}

\textbf{Evaluating the VAE.}
This section looks at how well the amortized encoder
$\hat\lambda$ approximates the optimal
encoder $\lambda^\star$ given by explicitly solving
\cref{eq:elbo-opt}, which is
referred to the \emph{amortization gap} \citep{cremer2018inference}.
\cref{eq:elbo-opt} can be solved with a gradient-based optimizer
such as SGD or Adam \citep{kingma2014adam}.
\Cref{lst:evaluation} shows the key parts of the PyTorch code
for making this comparison, which can be run on the pre-trained
MNIST VAE with
\path{code/evaluate_amortization_speed_function_vae.py}.

\Cref{fig:vae-performance} shows that the amortized prediction from
the VAE's encoder predicts the solution to the ELBO \textbf{25000}
times faster (!) than running 2k iterations of Adam on
a batch of 1024 samples.
This is significant as \emph{every training iteration} of
the VAE requires solving \cref{eq:elbo-opt}, and a large model
may need millions of training iterations to converge.
Amortizing the solution makes the difference between the training
code running in a few hours instead of a few months if
the problem was solved from scratch to the same level of optimality.
Knowing only the ELBO values is not sufficient to gauge
the quality of approximate variational distributions.
To help understand the quality of the approximate solutions,
\cref{fig:vae-samples} plots out the decoded samples
alongside the original data.

\begin{figure}[t]
  \centering
  \resizebox{\textwidth}{!}{
  \hspace*{-4mm}
  \begin{tikzpicture}
    \node[align=left,anchor=north west] {\includegraphics[width=430pt]{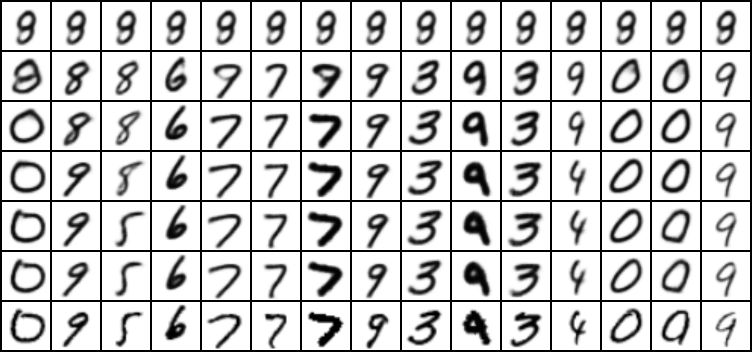}};
    \node[align=right,anchor=east] at (2mm,-6mm) (0) {0};
    \node[align=right,below=10.3mm of 0.east,anchor=east] (250) {${\rm Adam}_{250}$};
    \node[align=right,below=10.3mm of 250.east,anchor=east] (500) {${\rm Adam}_{500}$};
    \node[align=right,below=10.3mm of 500.east,anchor=east] (1000) {${\rm Adam}_{1000}$};
    \node[align=right,below=10.3mm of 1000.east,anchor=east] (2000) {${\rm Adam}_{2000}$};
    \node[align=right,below=10.3mm of 2000.east,anchor=east] (amor) {$\hat\lambda_\theta(x)$};
    \node[align=right,below=9.5mm of amor.east,anchor=east] (data) {Data};
  \end{tikzpicture}}
\caption{Decoded reconstructions of the variational distribution optimizing for the ELBO.
  ${\rm Adam}_n$ corresponds to the distribution from running Adam for $n$
  iterations, $\hat \lambda_\theta$ is the amortized approximation,
  and the ground-truth data, \ie the context, is shown in the bottom row.
}

  \label{fig:vae-samples}
\end{figure}

\begin{figure}[t]
\begin{lstlisting}
# amortization_model: maps contexts to a solution
# amortization_objective: maps an iterate and contexts to the objective

adam_lr, num_iterations = ...
contexts = sample_contexts()

# Predict the solutions with the amortization model
predicted_solutions = amortization_model(contexts)
amortized_objectives = amortization_objective(
    predicted_solutions, contexts
)

# Use Adam (or another torch optimizer) to solve for the solutions
iterates = torch.nn.Parameter(torch.zeros_like(predicted_solutions))
opt = torch.optim.Adam([iterates], lr=adam_lr)

for i in range(num_iterations):
    objectives = amortization_objective(iterates, contexts)
    opt.zero_grad()
    objective.backward()
    opt.step()
\end{lstlisting}
\caption{Evaluation code for comparing the amortized prediction $\hat y$
  to the true solution $y^\star$ solving \cref{eq:opt} with a
  gradient-based optimizer.
  The full instrumented version of this code is available in
  the repository associated with this tutorial at
  \texttt{\detokenize{code/evaluate_amortization_speed_function.py}}.
}
\label{lst:evaluation}
\end{figure}

\subsection{Control with a model-free value estimate}
\label{sec:impl:model-free}

\begin{figure}[t]
  \centering
  \includegraphics[width=0.49\textwidth]{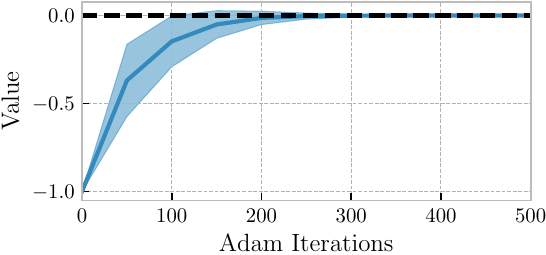}
  \hfill
  \includegraphics[width=0.49\textwidth]{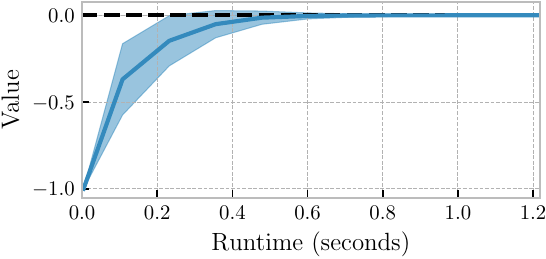} \\[5mm]
  \cblock{0}{0}{0} Policy $\pi(x)$ --- runtime: 0.65ms
  \caption{
    Runtime comparison between Adam and a learned policy $\pi_\theta$
    to solve \cref{eq:Q-opt} on the humanoid MDP.
    This was evaluated as a batch on an expert trajectory with
    1000 states and was run on an unloaded NVIDIA Quadro GP100 GPU.
    The values are normalized so that $\pi(x)=0$ takes a value of -1 and
    the optimal $\pi^\star$ takes a value of 0.
    The amortized policy is approximately
    1000 times faster than solving the
    problem from scratch.
  }
  \label{fig:model-free-performance}
\end{figure}

This section dives into the training and evaluation code for learning a
deterministic model-free policy $\pi_\theta: \gX\rightarrow\gY$ to amortize a model-free
value estimate $Q$ for controlling the humanoid MDP from
\citet{brockman2016openai}
visualized in \cref{fig:overview}.
This MDP has $|\gX|=45$ states (angular positions and velocities
describing the state of the system)
and $|\gY|=17$ actions (torques to apply to the joints).
A model-free policy $\pi$ maps the state to the optimal actions
that maximize the value on the system.
Given a known action-conditional value estimate $Q(x,u)$,
the optimal policy $\pi^\star$ solves the optimization problem in
\cref{eq:Q-opt} that the learned policy $\pi$ tries to match,
\eg using policy gradient in \cref{eq:dpg-loss}.

The codebase behind \citet{amos2021model} at
\url{https://github.com/facebookresearch/svg}
contains trained model-free policy and value estimates
on the humanoid in addition to model-based components
the next section will use.
The full training code there involves parameterizing
a stochastic policy and estimating many additional
model-based components, but the basic training loop
for amortizing a deterministic policy from the solution
there can be distilled into a form similar to
\cref{lst:vae}.

This section mostly focuses on evaluating the performance
of the trained model-free policy in comparison
to maximizing the model-free value estimate
\cref{eq:Q-opt} from scratch for every state encountered.
An exhaustive evaluation of a solver for \cref{eq:Q-opt}
would need to ensure that the solution is not overly
adapted to a bad part of the $Q$ estimate space ---
because $Q$ is also a neural network susceptible to
adversarial examples, it is very likely that directly
optimizing \cref{eq:Q-opt} may result in a deceptively
good policy when looking at the $Q$ estimate that does not
work well on the real system.
For simplicity, this section ignores these
issues and normalizes the values to $[-1,0]$ where
$-1$ will correspond to the value from taking a zero
action and $0$ will correspond to the value
from taking the expert's action.
(This is valid in this example because the zero
action and expert action never coincide.)

\Cref{fig:model-free-performance} shows that the
amortized policy is approximately 1000 times faster
than solving the problem from scratch.
The $Q$ values presented there are normalized and clamped
so that the expert policy has a value of zero and
the zero action has a value of -1.
This example can be run with
\path{code/evaluate_amortization_speed_function_control.py},
which shares the evaluation code also used for the VAE
in \cref{lst:evaluation}.

\subsection{Control with a model-based value estimate}
\label{sec:impl:model-based}

Extending the results of \cref{sec:impl:model-free},
this section compares the trained humanoid policy
from \url{https://github.com/facebookresearch/svg}
to solving a short-horizon ($H=3$) model-based control
optimization problem defined in \cref{eq:mpc}.
The optimal action sequence solving \cref{eq:mpc}
is $u^\star_{1:H}$ can be approximated by interleaving
a model-free policy $\pi_\theta$ with the dynamics $f$.
While standard model predictive control method are
often ideal for solving for $u^\star_{1:H}$ from scratch,
using Adam as a gradient-based shooting method is
a reasonable baseline in this short-horizon setting.

\Cref{fig:model-based-performance} shows that the
amortized policy is approximately 700 times faster
than solving the problem from scratch.
This model-based setting has the same issues with
the approximation errors in the models and the
model-based value estimate is again
normalized and clamped so that the expert
policy has a value of zero and the zero action has a value of -1.
The source code behind this example is also available in
\path{code/evaluate_amortization_speed_function_control.py}.

\begin{figure}[t]
  \centering
  \includegraphics[width=0.49\textwidth]{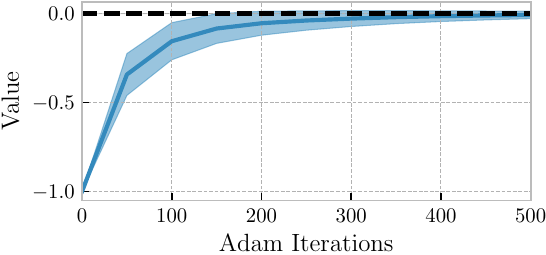}
  \hfill
  \includegraphics[width=0.49\textwidth]{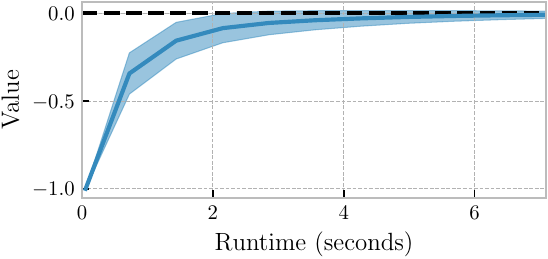} \\[5mm]
  \cblock{0}{0}{0} Policy $\pi(x)$ --- runtime: 5.8ms
  \caption{Runtime comparison between Adam and a learned policy $\pi_\theta$
    to solve a short-horizon ($H=3$) model-based control
    problem (\cref{eq:mpc}) on the humanoid MDP.
    This was evaluated as a batch on an expert trajectory with
    1000 states and was run on an unloaded NVIDIA Quadro GP100 GPU.
    The amortized policy is approximately
    700 times faster than solving the
    problem from scratch.
  }
  \label{fig:model-based-performance}
\end{figure}

\section{Training an amortization model on a sphere}
\label{sec:impl:sphere}

This section contains a new demonstration that applies
the insights from amortized optimization to learn to solve
optimization problems over spheres of the form
\begin{equation}
  y^\star(x) \in \argmin_{y\in\gS^2} f(y; x),
  \label{eq:sphere-opt-con}
\end{equation}
where $\gS^2$ is the surface of the \emph{unit 2-sphere}
embedded in $\R^3$ as $\gS^2\defeq \{y\in\R^3 \mid \|y\|_2=1\}$
and $x$ is some parameterization of the function
$f: \gS^2\times\gX\rightarrow \R$.
\Cref{eq:sphere-opt-con} is relevant to physical and
geographical settings seeking the extreme values of a
function defined on the Earth or other spaces that can
be approximated with a sphere.
The full source code behind this experiment is available
in \path{code/train-sphere.py}.

\textbf{Amortization objective.}
\Cref{eq:sphere-opt-con} first needs to be transformed
from a constrained optimization problem into an unconstrained
one of the form \cref{eq:opt}.
In this setting, one way of doing this
is by using a projection:
\begin{equation}
  y^\star(x) \in \argmin_{y\in\R^3} f(\pi_{\gS^2}(y); x),
  \label{eq:sphere-opt-proj}
\end{equation}
where $\pi_{\gS^2}: \R^3\rightarrow \gS^2$ is the
Euclidean projection onto $\gS^2$, \ie,
\begin{equation}
  \begin{aligned}
    \pi_{\gS^2}(x)\defeq& \argmin_{y\in\gS^2} \|y-x\|_2 \\
     =& \;x/\|x\|_2.
  \end{aligned}
  \label{eq:pi}
\end{equation}

\textbf{$c$-convex functions on the sphere.}
A synthetic class of optimization problems defined
on the sphere using the $c$-convex functions from
\citet{cohen2021riemannian} can be instantiated with:
\begin{equation}
  f(y; x) = {\textstyle \min_{\gamma}} \left\{\frac{1}{2} d(x,z_i)+\alpha_i\right\}_{i=1}^m
  \label{eq:rcpm}
\end{equation}
where $m$ components define the context
$x=\{z_i\} \cup \{\alpha_i\}$
with $z_i\in\gS^2$ and $\alpha_i\in\R$,
$d(x,y)\defeq \arccos(x^\top y)$ is the
Riemannian distance on the sphere in the
ambient Euclidean space, and
$\min_\gamma(a_1,\ldots,a_m)\defeq -\gamma\log\sum_{i=1}^m\exp(-a_i/\gamma)$
is a soft minimization operator
as proposed in \citet{cuturi2017soft}.
The context distribution $p(x)$ is sampled
with $z_i\sim \gU(\gS^2)$, \ie uniformly from the sphere,
and $\alpha_i\sim\gN(0,\beta)$
with variance $\beta\in\R_+$.

\textbf{Amortization model.}
The model $\hat y_\theta: \gX\rightarrow\R$
is a fully-connected MLP.
The predictions to \cref{eq:sphere-opt-con}
on the sphere can again be obtained by composing
the output with the projection
$\pi_{\gS^2}\circ \hat y_\theta$.

\textbf{Optimizing the gradient-based loss.}
Finally, it is reasonable to optimize the
gradient-based loss $\gL_{\rm obj}$ because
the objective and model are tractable and
easily differentiable.
\Cref{fig:sphere} shows the model's predictions
starting with the untrained model and finishing
with the trained model, showing that this setup
indeed enables us to predict the solutions to
\cref{eq:sphere-opt-con} with a single neural network
$\hat y_\theta(x)$ trained with the gradient-based loss.

\textbf{Summary.}
$\gA_{\rm sphere}\defeq (f\circ \pi_{\gS^2}, \R^3, \gX, p(x), \hat y_\theta, \gL_{\rm obj})$

\begin{figure}
  \includegraphics[width=0.25\textwidth]{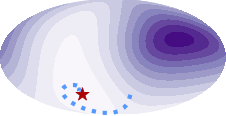}
  \hspace{-2.7mm}
  \includegraphics[width=0.25\textwidth]{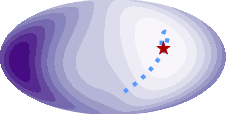}
  \hspace{-2.7mm}
  \includegraphics[width=0.25\textwidth]{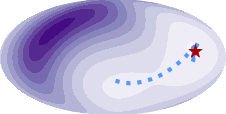}
  \hspace{-2.7mm}
  \includegraphics[width=0.25\textwidth]{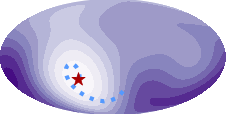} \\[-.8mm]
  \includegraphics[width=0.25\textwidth]{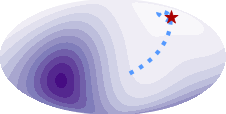}
  \hspace{-2.7mm}
  \includegraphics[width=0.25\textwidth]{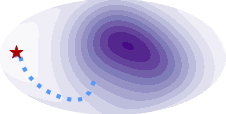}
  \hspace{-2.7mm}
  \includegraphics[width=0.25\textwidth]{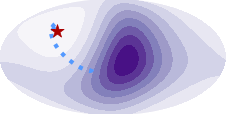}
  \hspace{-2.7mm}
  \includegraphics[width=0.25\textwidth]{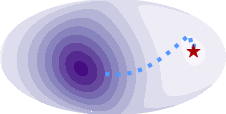} \\[-6mm]
  \begin{center}
    \cblock{73}{19}{134} $f(y; x)$ contours \;
    \cblock{170}{0}{0} Optimal $y^\star(x)$ \;
    \cblock{84}{153}{255} Predictions $\hat y_\theta(x)$ throughout training
  \end{center}
  \vspace{-3mm}
  \caption{Visualization of the predictions of an amortized
    optimization model predicting the solutions
    to optimization problems on the sphere.}
  \label{fig:sphere}
\end{figure}

\section{Other useful software packages}
\label{sec:impl:software}

Implementing semi-amortized models are usually more challenging
than fully-amortized models. Learning an optimization-based
model that internally solves an optimization problem is
not as widespread as learning a feedforward neural network.
While most autodiff packages provide standalone features to implement
unrolled gradient-based optimization, the following specialized
packages provide crucial features that further enable the
exploration of semi-amortized models:
\begin{itemize}
\item \href{https://github.com/cvxgrp/cvxpylayers}{cvxpylayers}
  \citep{agrawal2019differentiable}
  allows an optimization problem to be expressed in the
  high-level language \verb!CVXPY! \citep{diamond2016cvxpy}
  and exported to PyTorch, JAX, and TensorFlow
  as a differentiable optimization layers.
\item \href{https://github.com/google/jaxopt}{jaxopt}
  \citep{blondel2021efficient}
  is a differentiable optimization library for JAX
  and implements many optimization settings and fixed-point
  computations along with their implicit derivatives.
\item \href{https://github.com/facebookresearch/higher}{higher}
  \citep{grefenstette2019generalized}
  is a PyTorch library that adds differentiable higher-order
  optimization support with
  1) monkey-patched functional \verb!torch.nn! modules,
  and 2) differentiable versions of \verb!torch.optim!
  optimizers such as Adam and SGD.
  This enables arbitrary torch modules and optimizers
  to be unrolled and used as a semi-amortized model.
\item \href{https://github.com/metaopt/TorchOpt}{TorchOpt}
  provides a functional and differentiable optimizer in PyTorch
  and has higher performance than \verb!higher! in some cases.
\item \href{https://github.com/pytorch/functorch}{functorch}
  \citep{functorch2021} is a PyTorch library providing
  composable function transforms for batching and
  derivative operations, and for creating functional
  versions of PyTorch modules that can be used in
  optimization algorithms.
  All of these operations may arise in the implementation
  of an amortized optimization method and can become computational
  bottlenecks if not efficiently implemented.
\item \href{https://github.com/jump-dev/DiffOpt.jl}{DiffOpt.jl}
  provides differentiable optimization in Julia's JuMP
  \citep{DunningHuchetteLubin2017}.
\item \href{https://github.com/tristandeleu/pytorch-meta}{Torchmeta}
  \citep{deleu2019torchmeta} and
  \href{http://learn2learn.net}{learn2learn}
  \citep{arnold2020learn2learn}
  are PyTorch libraries and collection of meta-learning
  algorithms that also focus on making data-loading
  and task definitions easy.
\item \href{https://github.com/prolearner/hypertorch}{hypertorch}
  \citep{grazzi2020iteration}
  is a PyTorch package for computing hypergradients with a
  large focus on providing computationally efficient approximations
  to them.
\end{itemize}

\chapter{Discussion}
\label{sec:discussion}
Many of the specialized methods discuss tradeoffs and
limitations within the context of their application,
and more generally papers such as
\citet{chen2021learning,metz2021gradients}
provide even deeper probes into general paradigms for
learning to optimize.
This section emphasizes a few additional discussion points
around amortized optimization.

\section{Surpassing the convergence rates of
  classical methods}
\label{sec:convergence}
Theoretical and empirical optimization research often focuses
on discovering algorithms with theoretically strong convergence
rates in general or worst-case scenarios.
Many of the algorithms with the best convergence
rates are used as the state-of-the-art algorithms in practice,
such as momentum and acceleration methods.
Amortized optimization methods can surpass the results
provided by classical optimization methods because they
are capable of tuning the initialization and updates
to the best-case scenario within the distribution of
contexts the amortization model is trained on.
For example, the fully amortized models for amortized variational
inference and model-free actor-critic methods for RL
presented in \cref{sec:impl:eval} solve
the optimization problems \emph{in constant time} with just
a single prediction of the solution from the context without
even looking at the objective!
Further theoretical characterizations of this are provided
in \citet{khodak2022learning} and related literature on
algorithms with predictions.

\section{Generalization and convergence guarantees}
\label{sec:generalization}
Despite having powerful successes of amortized optimization in
some settings, the field struggles to bring strong success
in other domains.
Despite having the capacity of surpassing the convergence rates of
other algorithms, oftentimes in practice amortized optimization
methods can deeply struggle to generalize and converge to
reasonable solutions.
In some deployments this inaccuracy may be acceptable if there
is a quick way of checking the quality of the amortized model,
\eg the residuals for fixed-point and convex problems.
If that is the case, then poorly-solved instances can be flagged
and re-solved with a standard solver for the problem that
may incur more computational time for that instance.
\citet{sambharya2022l2ws} presents generalization bounds for learned
warm-starts based on Rademacher complexity,
and \citet{sambharya2023l2ws,sucker2024generalization} investigate PAC-Bayes generalization bounds.
\citet{banert2021accelerated,premont2022simple} add provable
convergence guarantees to semi-amortized models by guarding
the update and ensuring the learned optimizer does not
does not deviate too much from a known convergent algorithm.
A practical takeaway is that some models are more likely
to result in convergent and stable semi-amortized models
than others.
For example, the semi-amortized model
parameterized with gradient descent (which
has some mild converge guarantees) in \citet{finn2017model}
is often more stable than the semi-amortized model parameterized
by a sequential model (without many convergence guarantees)
in \citep{ravi2016optimization}.
Other modeling and architecture tricks such as layer
normalization \citep{ba2016layernorm} help improve the
stability of amortized optimization models.
Additionally, \citet{fahy2024greedya} investigates learning preconditioners
and prove that their parameterization of the preconditioning space
always results in a convergent optimizer.

\section{Measuring performance}
Quantifying the performance of amortization models can be even more
challenging than the choice between using a regression- or
objective-based loss and is often tied to
problem-specific metrics that are important.
For example, even if a method is able to attain low objective values
in a few iterations, the computation may take \emph{longer} than
a specialized algorithm or another amortization model that can reach
the same level of accuracy, thus not making it useful for the
original goal of speeding up solves to \cref{eq:opt}.

\section{Successes and limitations of amortized optimization}
While amortized optimization has standout applications in
variational inference, reinforcement learning, and meta-learning,
it struggles to bring value in other settings.
Often, learning the amortized model is computationally more
expensive than solving the original optimization problems and
brings instabilities into a higher-level learning or optimization
process deployed on top of potentially inaccurate solutions
from the amortized model.
This section summarizes principles behind successful applications
of amortized optimization and characterize limitations that
may arise.

\subsection*{Characteristics of successful applications}
\begin{itemize}
\item \textbf{Objective $f(y; x)$ is smooth over the domain $\gY$
  and has unique solutions $y^\star$.}
  With objective-based learning, non-convex objectives with
  few poor local optima are ideal.
  This behavior can be encouraged with smoothing as is
  often done for meta-learning and policy learning (\cref{sec:smooth}).
\item \textbf{A higher-level process should tolerate sub-optimal
  solutions given by $\hat y$ in the beginning of training.}
  In variational encoders, the suboptimal bound on the likelihood
  is still acceptable to optimize the density model's parameters over
  in \cref{eq:vae-full}.
  And in reinforcement learning policies,
  a suboptimal solution to the maximum value problem is still
  acceptable to deploy on the system in early phases of training,
  and may even be desirable for the exploration induced by randomly
  initialized policies.
\item \textbf{The context distribution $p(x)$ is not too big and
  well-scoped and deployed on a specialized class of sub-problems.}
  For example, instead of trying to amortize the solution to
  \emph{every} possible $\ELBO$ maximization, VAEs
  amortize the problem only over the dataset the density
  model is being trained on.
  And in reinforcement learning, the policy $\pi_\theta$ doesn't
  try to amortize the solution to \emph{every} possible control
  problem, but instead focuses only on amortizing the solutions
  to the control problems on the replay buffer of the
  specific MDP.
\item \textbf{In semi-amortized models, parameterizing the initialization
    and specialized components for the updates.}
  While semi-amortized models are a thriving research topic,
  the most successful applications of them:
  \begin{enumerate}
  \item \textbf{Parameterize and learn the initial iterate.}
    MAML \citep{finn2017model} \emph{only} parameterizes the initial
    iterate and follows it with gradient descent steps.
    \citet{bai2022neural} parameterizes
    the initial iterate and follows it with accelerated
    fixed-point iterations.
  \item \textbf{Parameterize and learn specialized components of the
    updates.} In sparse coding, LISTA \citep{gregor2010learning}
    only parameterized $\{F,G,\beta\}$ instead of the
    entire update rule.
    \citet{bai2022neural} only parameterizes $\alpha,\beta$
    after the initial iterate, and
    RLQP \citep{ichnowski2021accelerating} only parameterizing $\rho$.
  \end{enumerate}
  While using a pure sequence model to update a sequence of
  iterations is possible and theoretically satisfying as it
  gives the model the power to arbitrarily update the sequence
  of iterates, in practice this can be unstable and severely
  overfit to the training instances.
  \citet{metz2021gradients} observes, for example, that semi-amortized
  recurrent sequence models induce chaotic behaviors
  and exploding gradients.
\end{itemize}

\subsection*{Limitations and failures}
\begin{itemize}
\item \textbf{Amortized optimization does \emph{not} magically solve otherwise
  intractable optimization problems!}
  At least not without significant insights.
  In most successful settings, the original optimization problem can be
  (semi-)tractably solved for a context $x$ with classical methods,
  such as using standard black-box variational inference
  or model-predictive control methods.
  Intractabilities indeed start arising when repeatedly solving
  the optimization problem, even if a single one can be reasonably solved,
  and amortization often thrive in these settings to rapidly solve
  problems with similar structure.
\item \textbf{The combination of $p(x)$ and $y^\star(x)$ are too hard
  for a model to learn.} This could come from $p(x)$ being too
  large, \eg contexts of every optimization problem in the universe,
  or the solution $y^\star(x)$ not being smooth or predictable.
  $y^\star(x)$ may also not be unique, but this is perhaps easier
  to handle if the loss is carefully set up, \eg objective-based
  losses handle this more nicely.
\item \textbf{The domain requires accurate solutions.}
  Even though metrics that measure the solution quality of $\hat y$
  can be defined on top of \cref{eq:opt}, amortized methods
  typically cannot rival the accuracy of standard algorithms
  used to solve the optimization problems.
  In these settings, amortized optimization still has the
  potential at uncovering new foundations and algorithms
  for solving problems, but is non-trivial to
  successfully demonstrate.
  From an amortization perspective, one difficulty of safety-critical
  model-free reinforcement learning comes from needing to
  ensure the amortized policy properly optimizes a
  value estimate that (hopefully) encodes safety-critical
  properties of the state-action space.
\end{itemize}

\section{Some open problems and under-explored directions}
In most domains, introducing or significantly improving amortized
optimization is extremely valuable and will likely be well-received.
Beyond this, there are many under-explored directions and
combinations of ideas covered in this tutorial that can
be shared between the existing fields using amortized optimization,
for example:

\begin{enumerate}
\item \textbf{Overcoming local minima with objective-based losses
    and connections to stochastic policies.}
  \Cref{sec:smooth} covered the objective smoothing by
  \citet{metz2019understanding,merchant2021learn2hop}
  to overcome suboptimal local minima in the objective.
  These have striking similarities to stochastic policies
  in reinforcement learning that also overcome local
  minima, \eg in \cref{eq:Q-opt-sto-exp}.
  The stochastic policies, such as in \citet{haarnoja2018soft},
  have the desirable property of starting with a high variance
  and then focusing in on a low-variance solution with a
  penalty constraining the entropy to a fixed value.
  A similar method is employed in GECO \citep{rezende2018taming}
  that adjusts a Lagrange multiplier in the ELBO objective
  to achieve a target conditional log-likelihood.
  These tricks seem useful to generalize and apply to
  other amortization settings to overcome poor minima.
\item \textbf{Widespread and usable amortized convex solvers.}
  When using off-the-shelf optimization packages such as
  \citet{diamond2016cvxpy,o2016conic,stellato2018osqp},
  users are likely solving many similar problem instances
  that amortization can help improve.
  \citet{venkataraman2021neural,ichnowski2021accelerating}
  are active research directions that study adding
  amortization to these solvers, but they do not scale
  to the general online setting that also doesn't
  add too much learning overhead for the user.
\item \textbf{Improving the wall-clock training time
    of implicit models and differentiable optimization.}
  Optimization problems and fixed-point problems
  are being integrated into machine learning models,
  such as with differentiable optimization
  \citep{domke2012generic,gould2016differentiating,amos2017optnet,amos2019differentiable,agrawal2019differentiable,lee2019meta}
  and deep equilibrium models
  \citep{bai2019deep,bai2020multiscale}.
  In these settings, the data distribution the model
  is being trained on naturally induces a distribution over
  contexts that seem amenable to amortization.
  \citet{venkataraman2021neural,bai2022neural}
  explore amortization in these settings, but often do not
  improve the wall-clock time it takes to train these models
  from scratch.
\item \textbf{Understanding the amortization gap.}
  \citet{cremer2018inference} study the \emph{amortization gap}
  in amortized variational inference, which measures how well the
  amortization model approximates the true solution.
  This crucial concept should be analyzed in most amortized
  optimization settings to understand the accuracy of
  the amortization model.
\item \textbf{Implicit differentiation and shrinkage.}
  \citet{chen2019modular,rajeswaran2019meta} show that penalizing
  the amortization objective can significantly improve the
  computational and memory requirements to train a semi-amortized
  model for meta-learning. Many of the ideas in these settings
  can be applied in other amortization settings,
  as also observed by \citet{huszar2019imaml}.
\item \textbf{Distribution shift of $p(x)$ and out-of-distribution generalization.}
  This tutorial has assumed that $p(x)$ is fixed and remains
  the same through the entire training process.
  However, in some settings $p(x)$ may shift over time, which
  could come from 1) the data generating process naturally
  changing, or 2) a \emph{higher-level} learning process
  also influencing $p(x)$.
  Furthermore, after training on some context distribution $p(x)$,
  a deploy model is likely not going to be evaluated on the
  same distribution and should ideally be resilient
  to out-of-distribution samples.
  The out-of-distribution performance can often be measured
  and quantified and reported alongside the model.
  Even if the amortization model fails at optimizing \cref{eq:opt},
  it's detectable because the optimality conditions of
  \cref{eq:opt} or other solution quality metrics can be checked.
  If the solution quality isn't high enough, then a slower
  optimizer could potentially be used as a fallback.
\item \textbf{Amortized and semi-amortized control and reinforcement learning.}
  Applications of semi-amortization in control and reinforcement learning
  covered in \cref{sec:apps:ctrl} are budding and
  learning sample-efficient optimal controllers is
  an active research area, especially in model-based settings
  where the dynamics model is known or approximated.
  \citet{amos2019dcem} shows how amortization can learn latent
  control spaces that are aware of the structure of the
  solutions to control problems.
  \citet{marino2020iterative} study semi-amortized methods
  based on gradient descent and show that they better-amortize
  the solutions than the standard fully-amortized models.
\end{enumerate}

\section{Related work}
\subsection{Other tutorials, reviews, and discussions
  on amortized optimization}
My goal in writing this tutorial was to provide a perspective
of existing amortized optimization methods for learning
to optimize with a categorization of the
modeling (fully-amortized and semi-amortized)
and learning (gradient-based, objective-based, or RL-based)
aspects that I have found useful and have not seen
emphasized as much in the literature.
The other tutorials and reviews on
amortized optimization, learning to optimize, and
meta-learning over continuous domains
that I am aware of are excellent resources:

\begin{itemize}
\item \citet{chen2021learning} captures many other emerging areas
  of learning to optimize and discuss many other modeling paradigms
  and optimization methods for learning to optimize, such as
  plug-and-play methods \citep{venkatakrishnan2013plug,meinhardt2017learning,rick2017one,zhang2017learning}.
  They emphasize the key aspects and questions to tackle as a community,
  including model capacity, trainability, generalization, and
  interpretability.
  They propose \emph{Open-L2O} as a new benchmark for
  learning to optimize and review many other applications,
  including sparse and low-rank regression, graphical models,
  differential equations, quadratic optimization, inverse problems,
  constrained optimization, image restoration and reconstruction,
  medical and biological imaging, wireless communications,
  seismic imaging.
\item \citet{shu2017amortized} is a blog post that discusses
  fully-amortized models with gradient-based learning
  and includes applications in variational inference,
  meta-learning, image style transfer,
  and survival-based classification.
\item \citet{weng2018metalearning} is a blog post
  with an introduction and review of meta-learning methods.
  After defining the problem setup, the review discusses
  metric-based, model-based, and optimization-based approaches,
  and discusses approximations to the second-order derivatives
  that come up with MAML.
\item \citet{hospedales2020meta} is a review focused on meta-learning,
  where they categorize meta-learning components into a
  meta-representation, meta-optimizer, and meta-objective.
  The most relevant connections to amortization here are that
  the meta-representation can instantiate an
  amortized optimization problem that is solved with the
  meta-optimizer.
\item \citet{kim2020deep} is a dissertation on deep
  latent variable models for natural language
  and contextualizes and studies the use of amortization and
  semi-amortization in this setting.
\item \citet{marino2021learned} is a dissertation on learned
  feedback and feedforward information for perception and control
  and contextualizes and studies the use of amortization and
  semi-amortization in these settings.
\item \citet{monga2021algorithm} is a review on
  algorithm unrolling that starts with the unrolling
  in LISTA \citep{gregor2010learning} for amortized
  sparse coding, and then connects to other methods
  of unrolling specialized algorithms.
  While some unrolling methods have applications in
  semi-amortized models, this review also considers
  applications and use-cases beyond just
  amortized optimization.
\item \citet{banert2020data} consider theoretical foundations
  for data-driven nonsmooth optimization and show applications
  in deblurring and solving inverse problems for
  computed tomography.
\item \citet{liu2022teaching} study fully-amortized
  models based on deep sets \citep{zaheer2017deep}
  and set transformers \citep{lee2019set}.
  They consider regression- and objective-based losses
  for regression, PCA, core-set creation, and
  supply management for cyber-physical systems.
\item \citet{vanhentenryck2025optimizationlearning} presents an overview
  of learned optimization methods arising in power systems,
  for real-time risk assessment and security-constrained optimal power flow.
\end{itemize}

\subsection{Amortized optimization over discrete domains}
A significant generalization of \cref{eq:opt} is to optimization
problems that have discrete domains,
which includes combinatorial optimization
and mixed discrete-continuous optimization.
I have chosen to not include these works in this tutorial
as many methods for discrete optimization are significantly
different from the methods considered here, as learning with
derivative information often becomes impossible.
Key works in discrete and combinatorial spaces include
\citet{khalil2016learning,dai2017learning,jeong2019learning,bertsimas2019online,shao2021learning,bertsimas2021voice,cappart2021combinatorial}
and the surveys
\citep{lodi2017learning,bengio2021machine,kotary2021end}
capture a much broader view of this space.
\citet{banerjee2015efficiently} consider repeated ILP solves
and show applications in aircraft carrier deck scheduling and vehicle routing.
For architecture search, \citet{luo2018neural} learn a continuous
latent space behind the discrete architecture space.
Many reinforcement learning and control methods over discrete
spaces can also be seen as amortizing or semi-amortizing the
discrete control problems, for example:
\citet{cauligi2020learning,cauligi2021coco} use regression-based
amortization to learn mixed-integer control policies.
\citet{fickinger2021scalable} fine-tune the policy
optimizer for every encountered state.
\citet{tennenholtz2019natural,chandak2019learning,van2020q}
learn latent action spaces for high-dimensional
discrete action spaces with shared structure.

\subsection{Learning-augmented and amortized algorithms beyond optimization}
While many algorithms can be interpreted as solving an
optimization problems or fixed-point computations and
can therefore be improved with amortized optimization,
it is also fruitful to use learning to improve
algorithms that have nothing to do with optimization.
Some key starting references in this space include
data-driven algorithm design \citep{balcan2020data},
algorithms with predictions
\citep{dinitz2021faster,sakaue2022discrete,chen2022faster,khodak2022learning},
learning to prune \citep{alabi2019learning},
learning solutions to differential equations
\citep{li2020fourier,poli2020hypersolvers,karniadakis2021physics,kovachki2021universal,chen2021solving,blechschmidt2021three,marwah2021parametric,berto2021neural}
learning simulators for physics \citep{grzeszczuk1998neuroanimator,ladicky2015data,he2019learning,sanchez2020learning,wiewel2019latent,usman2021machine,vinuesa2021potential},
and learning for symbolic math
\citep{lample2019deep,charton2021linear,charton2021deep,drori2021neural,dascoli2022deep}
\citet{salimans2022progressive} progressively amortizes a
sampling process for diffusion models.
\citet{schwarzschild2021can} learn recurrent neural networks
to solve algorithmic problems for prefix sum, mazes, and chess.

\subsection{Continuation and homotopy methods}
Amortized optimization settings share a similar motivation to
continuation and homotopy methods that have been studied for
over four decades
\citep{richter1983continuation,watson1989modern,allgower2012numerical}.
These methods usually set the context space to be the
interval $\gX=[0,1]$ and simultaneously solve (without learning)
problems along this line.
This similarity indicates that problem classes typically
studied by continuation and homotopy methods could also benefit
from the shared amortization models here.

\newpage
\section*{Acknowledgments}
I would like to thank
Nil-Jana Akpinar,
Alfredo Canziani,
Samuel Cohen,
Georgina Hall,
Misha Khodak,
Boris Knyazev,
Hane Lee,
Joe Marino,
Maximilian Nickel,
Paavo Parmas,
Rajiv Sambharya,
Jens Sj\"olund,
Bartolomeo Stellato,
Alex Terenin,
Eugene Vinitsky,
Atlas Wang,
and
Arman Zharmagambetov
for insightful discussions
and feedback on this tutorial.
I am also grateful to the anonymous FnT reviewers
who gave a significant amount of helpful and
detailed feedback.

{\footnotesize\bibliography{amor}}

\end{document}